\documentclass[a4paper,11pt]{article}
\usepackage{geometry}
\usepackage{amsfonts}
\usepackage{url,multirow}
\usepackage[dvipsnames]{xcolor}
\usepackage{subfigure}
\usepackage[ruled,vlined]{algorithm2e}
\usepackage{algorithmic}
\usepackage{enumitem}
\usepackage{amsmath,bm}
\usepackage{amssymb,graphicx,booktabs,mathtools}
\usepackage{float}
\usepackage{longtable}
\usepackage{amsthm}
\newtheorem{theorem}{Theorem}

\newtheorem{assumption}[theorem]{Assumption}

\usepackage[capitalize,noabbrev]{cleveref}

\usepackage[textsize=tiny]{todonotes}

\title{Preconditioned One-Step Generative Modeling\\ for Bayesian Inverse Problems in Function Spaces}
\author{
Zilan Cheng\thanks{Division of Mathematical Sciences, School of Physical and Mathematical Sciences, Nanyang Technological University, 637371, Singapore. Email: \texttt{zilan001@e.ntu.edu.sg}.}
\and Li-Lian Wang\thanks{Division of Mathematical Sciences, School of Physical and Mathematical Sciences, Nanyang Technological University, 637371, Singapore. Email: \texttt{lilian@ntu.edu.sg}.}
\and Zhongjian Wang\thanks{Corresponding author. Division of Mathematical Sciences, School of Physical and Mathematical Sciences, Nanyang Technological University, 637371, Singapore. Email: \texttt{zhongjian.wang@ntu.edu.sg}.}
}

\begin{document}
\maketitle
\begin{abstract}
We propose a machine-learning algorithm for Bayesian inverse problems in the function-space regime. Based on one-step generative transport, the method learns an amortized neural operator whose pushforward of a Gaussian source approximates the posterior distribution conditioned on each new observation. We show that white-noise sources are incompatible with the function-space limit, and therefore adopt a prior-aligned GRF as the source. We justify this choice through the Lipschitz regularity of the resulting one-step conditional posterior transport and numerical experiments on linear inverse and PDE-based inverse problems. The method is not distilled from MCMC: it is trained only with prior samples and simulated partial noisy observations. Once trained, it generates a \(64\times64\) posterior sample in \(\sim 10^{-3}\)s, avoiding repeated forward-model evaluations in MCMC and repeated network evaluations in multistep generative samplers while matching key posterior summaries.
\end{abstract}
\section{Introduction}
Bayesian inference provides a principled framework for uncertainty quantification in scientific computing \cite{stuart2010inverse,dashti2015bayesian}, characterizing uncertainty induced by noisy, incomplete, or indirect observations. This paradigm underpins uncertainty-aware prediction and decision-making in applications such as porous media \cite{iglesias2014well}, imaging \cite{zhou2020bayesian}, geophysics \cite{malinverno2002parsimonious}, structural model updating \cite{beck1998updating}, and model calibration \cite{kennedy2001bayesian}.
This perspective is particularly important for PDE inverse problems, where the unknowns are functions such as coefficients, source terms, or initial conditions; and observational and modeling bias are unavoidable. In such settings, posterior distributions are typically high- or infinite-dimensional, and efficient sampling is central to practical Bayesian inversion.

Classical sampling methods for Bayesian inverse problems, most notably Markov Chain Monte Carlo (MCMC), are theoretically well understood and asymptotically exact. Function-space theory shows that their efficiency depends crucially on the geometry of the prior and posterior measures \cite{stuart2010inverse,cotter2013mcmc}.
In particular, preconditioned samplers---such as the preconditioned Crank--Nicolson (pCN) algorithm \cite{cotter2013mcmc}---replace white-noise random-walk proposals by Gaussian perturbations aligned with the prior covariance, yielding dimension-robust behavior and making them a cornerstone of modern Bayesian inversion. Nevertheless, MCMC methods rely on long correlated chains and repeated forward solves, and can be prohibitively expensive in practice for large-scale PDE inverse problems (see Table~\ref{tab:sampling_time} for runtime comparisons).

Generative models offer a data-driven alternative for sampling complex distributions. Diffusion and score-based generative models \cite{NEURIPS2020_4c5bcfec,song2021scorebased,cardoso2024monte} have become standard paradigms for transforming simple source distributions into complex distributions. Closely related, flow matching learns the transport velocity field directly from an ODE viewpoint \cite{lipman2023flow,albergo2023building,liu2023flow}. These approaches are powerful because, after training, they can generate samples much faster than MCMC. However, they are typically multistep: sampling requires a time-discretized SDE or ODE, and hence many network evaluations, with possible error accumulation along the trajectory. 

For inverse problems, diffusion models are often adapted by first learning a reusable unconditional generative prior and then incorporating observations only at inference time, following guidance originally developed for diffusion-based image synthesis \cite{dhariwal2021diffusion}. DPS \cite{chung2023diffusion} and related posterior-sampling methods \cite{dou2024diffusion} follow this principle by combining a pretrained diffusion prior with measurement-dependent sampling corrections. This modularity is attractive, but the conditioning is typically an approximation, so the resulting samples are not guaranteed to follow the exact Bayesian posterior. Other diffusion- or flow-prior methods also reuse learned generative priors, but most of them are still closer to test-time optimization or MAP-style inversion than to posterior sampling \cite{wang2024dmplug,zhang2024flow}. While such methods are useful for vision-oriented reconstruction tasks, Bayesian inference problems require quantifying posterior uncertainty rather than merely producing a plausible reconstruction. 
 
Narrowing down to the PDE inverse problems, there is an additional function-space constraint: the unknown is a function whose posterior law should be stable under discretization refinement \cite{stuart2010inverse,cotter2013mcmc}. Many recent diffusion-based approaches for PDEs operate after discretization and use grid-based finite-dimensional parameterizations, such as U-Nets \cite{ronneberger2015u}, with iterative conditional sampling or guidance procedures \cite{huang2024diffusionpde,shysheya2024on}. Neural operators, including DeepONet~\cite{lu2021learning}, Fourier neural operators (FNO), and their variants \cite{li2021fourier,liu2024render,cheng2026podno}, provide a more natural function-to-function parameterization. Beyond parameterization of the score, the source measure in the generative models must also be compatible with the function-space geometry. It is not merely a finite-dimensional numerical choice, but it affects the well-posedness and discretization robustness of the sampling problem. Recent function-space flow-matching, diffusion, and score-based models make this point explicit through measure-theoretic constructions and geometry-compatible Gaussian source or noising measures \cite{pmlr-v238-kerrigan24a,pidstrigach2024infinite,lim2025score}. However, these methods typically require multistep ODE/SDE or Langevin sampling, leaving the discussion of efficient and accurate algorithms for function-space Bayesian inverse problems comparatively underexplored.

\medskip
To summarize, three obstacles remain when applying deep generative models to function-space Bayesian inverse problems: sampling efficiency, posterior consistency, and function-space regularity. In this work, we address these issues through a one-step conditional sampler for Bayesian inverse problems in function spaces.

First, to improve sampling efficiency, we employ an operator approach to model the conditional-averaged velocity field proposed in Mean Flows \cite{geng2025mean}. This yields a one-step sampler that avoids the time discretization of SDE- or ODE-based generative samplers and requires only a single model evaluation at inference time. Second, to ensure posterior consistency, we train the sampler on pairs of prior samples and their noisy observations generated by the simulator. We prove that, for each fixed observation, this conditional training corresponds exactly to learning from the associated posterior, without MCMC-generated posterior labels or inference-time guidance corrections. Third, for function-space regularity, we use a prior-aligned anisotropic Gaussian source rather than a discretized white-noise source. We analyze the role of the source measure in the function-space limit and establish Lipschitz regularity of the induced one-step transport under the Gaussian tail assumption \cite{xixian2026wasserstein,meng2025pathway}. Furthermore, in the numerical experiments, we validate the proposed source choice and demonstrate the accuracy and efficiency of the overall sampler across synthetic linear and practical PDE inverse problems.

The rest of the paper is organized as follows. Section~\ref{sec:preliminaries} reviews vanilla flow matching, Mean Flows, and the functional Bayesian inverse-problem formulation. Section~\ref{sec:methodology} derives the conditional one-step posterior transport and describes its neural operator parameterization. Section~\ref{sec:theory} analyzes the role of the source measure and establishes regularity properties of the prior-aligned transport. Section~\ref{sec:experiments} presents numerical experiments on a synthetic linear inverse problem and various practical PDE inverse problems. Section~\ref{sec:conclusion} concludes the paper.

\section{Preliminaries}
\label{sec:preliminaries}
This section provides the background for the proposed conditional one-step posterior sampler. We first recall flow matching as standard multistep generative samplers in Section~\ref{subsec:multistep_generative_samplers}. We then review the one-step Mean-Flow construction in Section~\ref{subsec:one_step_mean_flows} and finally formulate the functional Bayesian inverse-problem setting in Section~\ref{subsec:functional_bayesian_inverse_problems}. 

\subsection{Flow matching as a multistep generative sampler}
\label{subsec:multistep_generative_samplers}
Multistep generative samplers generate target samples from source samples through time-dependent dynamics. In this subsection, we focus on flow matching with a linear interpolation schedule \cite{liu2023flow}; a brief review of the score-based diffusion model with variance preserving properties \cite{song2021scorebased} is deferred to
Appendix~\ref{app:multistep}.

Let \(X\) denote a random variable under the target distribution and let \(\Xi\) denote a random variable from an easy-to-sample source distribution independent of $X$, usually Gaussian. In flow matching, we consider a probability-flow dynamics,
\begin{equation}
    \tfrac{d}{dt} Z_t = v(Z_t,t), \quad t\in[0,1],
    \label{eq:fm_probability_flow_ode}
\end{equation}
with some $v$ such that the marginal distribution of $Z_t$ follows a linear interpolation between the target and the source distribution, i.e.,
\begin{equation}
    Z_t\sim(1-t)X+t\Xi, \quad X\perp\Xi,
    \quad t\in[0,1].
    \label{eq:prelim_linear_path}
\end{equation}
This defines a family of marginal distributions connecting the target side \((t=0)\) to the source side \((t=1)\). Along each paired path, the velocity is \(\Xi-X\). However, the marginal sampling dynamics must depend only on the
current state \(z\), rather than on the unknown endpoint pair \((X,\Xi)\). 

To this end, we consider a Markovian projection at each
time \(t\), 
\begin{equation}\label{projected v}
    v(z,t)
    :=\frac{\int \frac{z-x}{t}p_X(x)p_\Xi(\frac{z-(1-t)x}{t})dx}{\int p_X(x)p_\Xi(\frac{z-(1-t)x}{t})dx}=
    \mathbb E[
        \Xi-X
        \mid
        Z_t=z
    ].
\end{equation}
 Then, the population velocity field \(v\) defines a flow that starts from a draw \(z_1\) of the source distribution, and the marginal distribution of the solution matches the marginal distribution induced by \eqref{eq:prelim_linear_path}. One can show that such  consistency of the marginal law holds when $\Xi$ is a centered Gaussian distribution, following a  similar argument as \cite{lu2025mathematical}.

From a practical learning perspective, the proposed \(v(\cdot,t)\), as a conditional expectation, is the population minimizer of the standard flow-matching \(L^2\) objective
\begin{equation}
    \mathcal L_{\mathrm{FM}}(\theta)
    :=
    \mathbb E[
        \|
        v_\theta(Z_t,t)-(\Xi-X)
        \|^2
    ].
    \label{eq:standard_fm_loss}
\end{equation}
In practice, \(v\) is replaced by the learned approximation \(v_\theta\), and target samples are generated by solving the resulting ODE backward from \(t=1\) to \(t=0\).

The construction of the dynamics \eqref{eq:fm_probability_flow_ode} with \eqref{projected v} can be viewed as a reversed heat-flow-type transport \cite{kim2012generalization}. Although \eqref{eq:prelim_linear_path} is linear along each paired path, the induced $v$ from the projection is generally nonlinear. Thus, standard flow-matching samplers still require multiple ODE discretization steps, motivating the one-step Mean-Flow \cite{geng2025mean} construction, which is reviewed next.
\subsection{One-step generative sampling based on Mean Flows}
\label{subsec:one_step_mean_flows}
Very recently, Mean Flows \cite{geng2025mean} modified flow-matching by learning an averaged velocity rather than the instantaneous velocity. Let \(v(\cdot,t)\) be the marginal velocity field associated with the interpolation path \eqref{eq:prelim_linear_path}. For \(0\le r<t\le 1\), define the trajectory-averaged velocity by
$w(z,r,t)
:=
\frac1{t-r}\int_r^t v(z_\tau,\tau)\,d\tau$ .
Formally, along a trajectory \(\{Z_\tau\}_{\tau\in[r,t]}\), one has
\begin{equation}
Z_t-Z_r = \int_r^t v(Z_\tau,\tau)\,d\tau = (t-r)w(Z_t,r,t).
\label{eq:mean_flow_integral_relation}    
\end{equation}
Differentiating \eqref{eq:mean_flow_integral_relation} with respect to \(t\) along the trajectory gives $w(Z_t,r,t) + (t-r)\frac{d}{dt}w(Z_t,r,t) = v(Z_t,t),$ where $\frac{d}{dt}w(Z_t,r,t) = \partial_t w(Z_t,r,t) + \nabla_z w(Z_t,r,t)\cdot v(Z_t,t)$.
Equivalently,
\begin{equation}   w(Z_t,r,t)   =    v(Z_t,t)    -    (t-r)\frac{d}{dt}w(Z_t,r,t).    \label{eq:mean_flow_self_consistency}
\end{equation}
During training, Mean Flows constructs the regression target from \eqref{eq:mean_flow_self_consistency} by replacing the marginal velocity term \(v(Z_t,t)\) with the path velocity \(\Xi-X\). This trains the model to learn the interval-averaged
velocity field \(w\) directly, instead of learning an instantaneous
velocity field and repeatedly evaluating it along a discretized trajectory,
as one may have done in other distilled models;
we provide a modified approach whose target distribution becomes a Bayesian posterior with greater details in Section \ref{subsec:algebraic_derivation}.

At inference time, a source draw \(\xi\) is mapped back to the target side
by evaluating the averaged field at \((r,t)=(0,1)\):
$    \mathcal T(\xi):=\xi-w(\xi,0,1).$
\subsection{Functional Bayesian inverse problems}
\label{subsec:functional_bayesian_inverse_problems}
We now recall the functional Bayesian formulation of inverse problems \cite{stuart2010inverse}. Let \(\mathcal H\) be a separable Hilbert space, and let \(X\) be an \(\mathcal H\)-valued unknown with realizations denoted by \(x\in\mathcal H\). For example, \(X\) may represent a PDE coefficient, a source term, or an initial condition. Let $\mathcal G:\mathcal H\to\mathcal Y$ be a forward operator, and let $\mathcal O:\mathcal Y\to\mathbb R^m$ be a finite-dimensional observation operator. We consider the observation model with an additive Gaussian observation noise,
\begin{equation}
    Y_{\mathrm{obs}}
    :=
    \mathcal G_{\mathrm{obs}}(X)+\eta
    =
    \mathcal O(\mathcal G(X))+\eta,
    \quad
    \eta\sim\mathcal N(0,\Gamma),
    \label{eq:obs_model}
\end{equation}
where $\mathcal G_{\mathrm{obs}}:=\mathcal O\,\circ\, \mathcal G$ is the observed forward map and \(\Gamma\) is the observation-noise covariance matrix. In the experiments reported in this work, we often set \(\Gamma = \sigma_\eta^2 I_m\) with a problem-specific \(\sigma\).

Following the standard Bayesian formulation of inverse problems  \cite{stuart2010inverse,dashti2015bayesian}, we place a Gaussian prior \(\gamma=\mathcal N(0,\Lambda)\) on \(X\). In this work, \(\Lambda\) is chosen as a Mat\'ern-type covariance operator, $\Lambda=(-\Delta+\kappa^2 I)^{-\alpha}$, where \(\kappa>0\) controls the correlation length and \(\alpha>0\) controls the prior smoothness.

For a fixed observation \(y_{\mathrm{obs}}\), the Bayesian posterior \(\pi(\cdot\mid y_{\mathrm{obs}})\) is the conditional distribution of \(X\) given \(Y_{\mathrm{obs}}=y_{\mathrm{obs}}\). In function space, it is written as
\begin{equation}
\pi(\mathrm dx\mid y_{\mathrm{obs}})
:=
\mathbb P(X\in \mathrm dx\mid Y_{\mathrm{obs}}=y_{\mathrm{obs}})
\propto
\exp(R_\Lambda(x;y_{\mathrm{obs}}))\gamma(\mathrm dx),
\label{eq:posterior_phi}
\end{equation}
where $    R_\Lambda(x;y_{\mathrm{obs}})
    =
    -\tfrac{1}{2\sigma_\eta^2}
\|
    y_{\mathrm{obs}}-\mathcal G_{\mathrm{obs}}(x)
\|_2^2  .$

\section{Conditional One-Step Posterior Transport}
\label{sec:methodology}
This section is devoted to the construction of an amortized one-step sampler for \(\pi(\cdot\mid y_{\mathrm{obs}})\) raised from the functional Bayesian formulation in Section~\ref{subsec:functional_bayesian_inverse_problems}. We first use the joint-law viewpoint to motivate the conditional posterior transport and derive the corresponding Mean-Flow objective and one-step sampling map in Section~\ref{subsec:algebraic_derivation}. We then describe its neural-operator realization and summarize the training and sampling algorithms in Section~\ref{subsec:operator_algorithms}.

\subsection{Derivation of the one-step sampler}
\label{subsec:algebraic_derivation}
In the functional Bayesian inverse-problem setting, our goal is to construct a conditional sampler that, for each observation \(y_{\mathrm{obs}}\), maps a draw $\xi$ from the source distribution $\rho$ to a posterior sample following $\pi(\cdot|y_{\mathrm{obs}})$. Concretely, this amounts to a conditional map with learnable variable $\theta\in\Theta$,
\begin{equation}\label{transportmap}
\mathcal T_\theta:\mathcal H\times\mathbb R^m\to\mathcal H,
~
(\xi,y_{\mathrm{obs}})
\mapsto
\mathcal T_\theta(\xi;y_{\mathrm{obs}}),  
\end{equation}
or equivalently, a family of maps
\(\{\mathcal T_\theta(\cdot;y_{\mathrm{obs}})\}_{y_{\mathrm{obs}}}:\mathcal H\to\mathcal H\), such that
\begin{equation}
    \bigl(\mathcal T_\theta(\cdot;y_{\mathrm{obs}})\bigr)_\#\rho
    \approx
    \pi(\cdot\mid y_{\mathrm{obs}}).
    \label{eq:target_transport}
\end{equation} 
Note that under the formulation of \eqref{transportmap} and \eqref{eq:target_transport}, the same parameters $\theta$ are shared across observations \(y_{\mathrm{obs}}\), and therefore, once trained, for a new observation, the posterior samples can be retrieved from one forward pass of $\mathcal{T}_\theta$.

We first describe the observation-conditioned flow-based model under posterior-to-source interpolation. Let $(X,Y_{\mathrm{obs}})\sim\mathbb P_{X,Y_{\mathrm{obs}}}$ denote a joint simulator draw. By the definition of the posterior, conditioning on
\(Y_{\mathrm{obs}}=y_{\mathrm{obs}}\) makes the corresponding \(X\)-component
distributed as \(\pi(\cdot\mid y_{\mathrm{obs}})\). 
We further sample \(\Xi\) from the prior-aligned source distribution \(\rho_C:=\mathcal N(0,C)\), and take the linear interpolation path
\eqref{eq:prelim_linear_path}, namely,
\begin{equation}\label{eq:linear_path}
    Z_t = (1-t)X|y_{\mathrm{obs}} + t\Xi.
\end{equation}

Then, let \(p_t(\cdot\mid y_{\mathrm{obs}})\) denote the law of \(Z_t\), and define
$S_t(z;y_{\mathrm{obs}})
:=
\nabla_z\log p_t(z\mid y_{\mathrm{obs}}).$ The marginal velocity field associated with this interpolation is 
$v(z,t;y_{\mathrm{obs}}):=
\mathbb E[\Xi-X\mid Z_t=z,y_{\mathrm{obs}}].$ 
Therefore, by $\Xi-X=\tfrac{Z_t-X}{t}$, we have $v(z,t;y_{\mathrm{obs}})=\tfrac1t
(z-\mathbb E[X\mid Z_t=z,y_{\mathrm{obs}}]
)$. 
Since, conditional on \(X=x\),
$Z_t\mid X=x \sim \mathcal N((1-t)x,t^2C)$, the Tweedie's formula yields 
$C S_t(z;y_{\mathrm{obs}})
=
-\frac1{t^2}
\left(
z-(1-t)\mathbb E[X\mid Z_t=z,y_{\mathrm{obs}}]
\right),$
equivalently,
\[
\mathbb E[X\mid Z_t=z,y_{\mathrm{obs}}]
=\frac{z+t^2 C S_t(z;y_{\mathrm{obs}})}{1-t}.
\]
Thus, we could write the corresponding marginal velocity as
\begin{equation}
\label{eq:velocity_score_identity}
    v(z,t;y_{\mathrm{obs}})
    =
    -\frac{1}{1-t}
    (
        z+tC S_t(z;y_{\mathrm{obs}})
    ).
\end{equation}
This identity relies on the consistency between the Gaussian source
\(\rho_C=\mathcal N(0, C)\) and the posterior $X$, and the importance of the source choice will be justified theoretically in
Section~\ref{sec:theory} and demonstrated experimentally in Section~\ref{sec:experiments}.

We can now construct the conditional sampler in \eqref{eq:target_transport} with $\rho_C$ by
adapting the Mean-Flow construction~\cite{geng2025mean} to the observation-conditioned setting. 
To be specific, we learn a conditional Mean-Flow predictor \cite{geng2025mean}
\(w_\theta:\mathcal H\times[0,1]\times[0,1]\times\mathbb R^m\to\mathcal H\),
\((z,r,t, y_{\mathrm{obs}})\mapsto
w_\theta(z,r,t; y_{\mathrm{obs}})\),
which approximates the observation-conditioned averaged velocity over the
interval \([r,t]\).

Following the Mean-Flow identity \eqref{eq:mean_flow_self_consistency}, we define the stop-gradient target from the path velocity \(V^{\mathrm{path}}:=\Xi-X\):
\begin{equation}
W_{\mathrm{tgt}}
:=
V^{\mathrm{path}}
-
(t-r)\,
\mathrm{JVP}_{(z,r,t)}
\bigl(w_\theta;(V^{\mathrm{path}},0,1)\bigr),
\label{eq:w_target_mf}
\end{equation}
where the JVP is evaluated at \((z,r,t)=(Z_t,r,t)\), and \(r\) is held fixed
in the derivative. Equivalently,
\begin{equation}
\begin{aligned}
\mathrm{JVP}_{(z,r,t)}
\bigl(w_\theta;(V^{\mathrm{path}},0,1)\bigr)
:=&
\left.
\frac{d}{d\epsilon}
w_\theta
\bigl(
Z_t+\epsilon V^{\mathrm{path}},
r,
t+\epsilon;
Y_{\mathrm{obs}}
\bigr)
\right|_{\epsilon=0}\\
=&
D_z w_\theta(Z_t,r,t; Y_{\mathrm{obs}})[V^{\mathrm{path}}]
+
\partial_t w_\theta(Z_t,r,t; Y_{\mathrm{obs}}).    
\end{aligned}
\label{eq:w_target_jvp_def}
\end{equation}
Here \(D_z w_\theta(\cdot)[V^{\mathrm{path}}]\) denotes the directional derivative with respect
to the state variable along the direction \(V^{\mathrm{path}}\).

Consequently, the training objective in our one-step generative modeling of Bayesian posterior reads,
\begin{equation}
\mathcal L(\theta)
:=
\mathbb E_{(X,Y_{\mathrm{obs}}),\,\Xi,\,r,\,t}
\big[
\|
w_\theta(Z_t,r,t;Y_{\mathrm{obs}})
-
\mathrm{sg}(W_{\mathrm{tgt}})
\|^2
\big],
\label{eq:mf_loss}
\end{equation}
where \(\mathrm{sg}(\cdot)\) denotes the stop-gradient operator. The expectation is defined over the joint simulator pairs \((X,Y_{\mathrm{obs}})\). Since these pairs are generated from the prior predictive law, it is not immediate that minimizing \eqref{eq:mf_loss} learns a posterior sampler for each fixed observation.  The following theorem provides the validity of \eqref{eq:mf_loss}.
\begin{theorem}[Exactness of the posterior sampling from the forward data]
\label{thm:posterior_conditioned_population_minimizer}
Let \((X,Y_{\mathrm{obs}})\) be sampled from the joint prior-predictive law, \(\Xi\sim\rho\) be independent of \((X,Y_{\mathrm{obs}})\), and 
\((r,t)\) from a joint distribution $\lambda$ on
\(\{0\le r\le t\le 1\}\). Assume that the loss integrand in
\eqref{eq:mf_loss} is integrable. Then the unique minimizer $w^{\star}$ of the joint-pair population objective
\(\mathcal L(\theta)\) constructs the pushforward map to the exact posterior.
\end{theorem}
\begin{proof}
By the definition of the posterior as the conditional law of \(X\) given
\(Y_{\mathrm{obs}}=y\), the joint law factorizes as
\[
    \mathbb P_{X,Y_{\mathrm{obs}}}(\mathrm dx,\mathrm dy)
    =
    \pi(\mathrm dx\mid y)\,
    \mathbb P_{Y_{\mathrm{obs}}}(\mathrm dy).
\]
Substituting this factorization into the population expectation
\eqref{eq:mf_loss} gives
\[
\begin{aligned}
\mathcal L(\theta)
&=
\int
\|
w_\theta(Z_t,r,t;\mathcal E_y(y))
-
\mathrm{sg}(w_{\mathrm{tgt}})
\|^2
\,
\pi(\mathrm dx\mid y)\,
\mathbb P_{Y_{\mathrm{obs}}}(\mathrm dy)\,
\rho(\mathrm d\xi)\,
\mathrm d\lambda(r,t) \\
&=
\int
[
\int
\|
w_\theta(Z_t,r,t;\mathcal E_y(y))
-
\mathrm{sg}(w_{\mathrm{tgt}})
\|^2
\,
\pi(\mathrm dx\mid y)\,
\rho(\mathrm d\xi)\,
\mathrm d\lambda(r,t)
]
\mathbb P_{Y_{\mathrm{obs}}}(\mathrm dy) \\
&=
\int
\mathcal L_y(\theta)\,
\mathbb P_{Y_{\mathrm{obs}}}(\mathrm dy)
=
\mathbb E_{Y_{\mathrm{obs}}}
[
    \mathcal L_{Y_{\mathrm{obs}}}(\theta)
],
\end{aligned}
\]
where, for each fixed observation \(y_{\mathrm{obs}}\),
\begin{equation}
    \label{Lyobs}\mathcal L_{y_{\mathrm{obs}}}(\theta)
:=
\mathbb E_{X\sim\pi(\cdot\mid y_{\mathrm{obs}}),\,\Xi,\,r,\,t}
[
\|
w_\theta(Z_t,r,t;y_{\mathrm{obs}})
-
\mathrm{sg}(W_{\mathrm{tgt}})
\|^2
]
\end{equation}
For each fixed $y_{\mathrm{obs}}$, the unique minimizer of \eqref{Lyobs}, \(w^\star(\cdot,r,t;\mathcal y_{\mathrm{obs}})\), constructs an exact posterior sampler $(\mathcal T^\star(\cdot;y_{\mathrm{obs}}))_\#\rho
=
\pi(\cdot\mid y_{\mathrm{obs}})$ by
$\mathcal T^\star(\xi;y_{\mathrm{obs}})
=
    \xi
    -
    w^\star(\xi,0,1;y_{\mathrm{obs}})$. Finally, note that taking the expectation of \eqref{Lyobs} over $y_{\mathrm{obs}}$ makes the overall minimizer of $\mathcal{L}$, $w^{\star}$ constructs the exact pushforward map for any possible $y_{\mathrm{obs}}$. 
\end{proof}
We use the learned predictor \(w_\theta\) as an approximation
of the exact averaged velocity \(w^\star\), giving the approximated one-step sampler in \eqref{transportmap} and \eqref{eq:target_transport}:
\begin{equation}
\label{eq:learned_transport}
\mathcal T_\theta(\xi;y_{\mathrm{obs}})
=
    \xi
    -
    w_\theta(\xi,0,1;y_{\mathrm{obs}}),
\quad
\xi\sim\rho .
\end{equation}

\begin{figure}[ht!]
  \centering
  \includegraphics[width=1\textwidth]{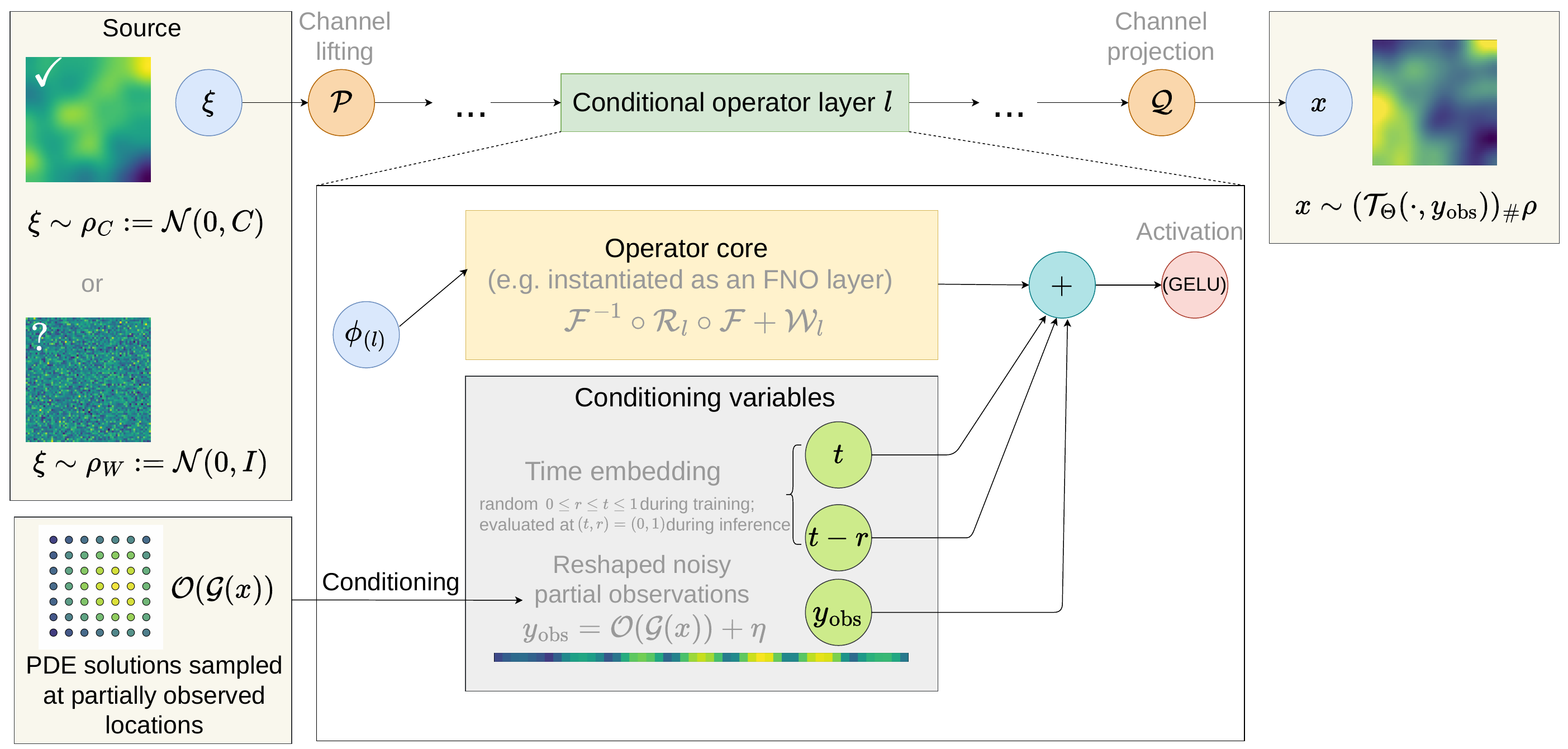}
  \caption{
  Conditional neural-operator parameterization of the mean-flow predictor.
A source draw \(\xi\sim\rho\) is lifted by \(\mathcal P\), passed through
\(L\) conditional operator layers, and projected by \(\mathcal Q\) to produce
the mean-flow correction. During training, time pairs \(0\le r\le t\le 1\)
are sampled randomly; at inference, the learned predictor is evaluated at
\((r,t)=(0,1)\) to generate one-step posterior samples.}
  \label{fig:conditional_operator_generator}
\end{figure}

\subsection{Neural-operator realization and algorithms}
\label{subsec:operator_algorithms}
\begin{algorithm}[!htb]
  \caption{Training the Conditional Mean-Flow Predictor}
  \label{alg:mf_training_encode_only}
  \begin{algorithmic}
    \STATE {\bfseries Input:} Training pairs
    \(\{(x^{(i)},y_{\mathrm{obs}}^{(i)})\}_{i=1}^{N_{\mathrm{train}}}\) sampled from the joint law of
    \((X,Y_{\mathrm{obs}})\);
    source measure \(\rho\).
    \STATE {\bfseries Preprocessing:}
    Estimate encoding statistics \((\mu_x,\sigma_x)\) and
    \((\mu_y,\sigma_y)\) from the training set, and define
    \(\tilde y_{\mathrm{obs}}=\mathcal E_y(y_{\mathrm{obs}})\) and \(\tilde x=\mathcal E_x(x)\) by
    \eqref{eq:normalization}.
    \STATE {\bfseries Initialize:} network parameters \(\theta\).
    \REPEAT
      \STATE Sample a mini-batch \((x,y_{\mathrm{obs}})\) from the training set.
      \STATE Encode
      \(\tilde x=\mathcal E_x(x)\) and
      \(\tilde y_{\mathrm{obs}}=\mathcal E_y(y_{\mathrm{obs}})\).
      \STATE Sample an independent source draw \(\xi\sim\rho\).
      \STATE Sample a time pair \(0\le r\le t\le 1\).
      \STATE Form
          \(z_t=(1-t)\tilde x+t\xi\),
          \(v^{\mathrm{path}}=\xi-\tilde x\).
      \STATE Predict the averaged velocity
      \(w_\theta(z_t,r,t;\tilde y_{\mathrm{obs}})\) as in
      \eqref{eq:conditional_operator_composition}.
      \STATE Construct the stop-gradient target \(w_{\mathrm{tgt}}\) using
      \eqref{eq:w_target_mf}.
      \STATE Update \(\theta\) by minimizing
        $\|
        w_\theta(z_t,r,t;\tilde y_{\mathrm{obs}})
        -
        \mathrm{sg}(w_{\mathrm{tgt}})
        \|^2 $.
    \UNTIL{convergence}
    \STATE {\bfseries Output:} trained parameters \(\hat\theta\).
  \end{algorithmic}
\end{algorithm}
\begin{algorithm}[!htb]
  \caption{One-Step Conditional Posterior Sampling}
  \label{alg:mf_sampling_encode_only}
  \begin{algorithmic}
    \STATE {\bfseries Input:} Observation \(y_{\mathrm{obs}}\); trained
    parameters \(\hat\theta\); source measure \(\rho\); encoders
    \(\mathcal E_x,\mathcal E_y\).
    \STATE Encode the observation:
        \(\tilde y_{\mathrm{obs}}
        =
        \mathcal E_y(y_{\mathrm{obs}})\).
    \STATE Sample \(\xi\sim\rho\).
    \STATE Generate an encoded posterior sample:
        \(\tilde x
        =
        \xi
        -
        w_{\hat\theta}(\xi,0,1;\tilde y_{\mathrm{obs}})\).
    \STATE Decode to physical space:
        \(x=\mathcal E_x^{-1}(\tilde x)\).
    \STATE {\bfseries Output:} Posterior sample \(x\).
  \end{algorithmic}
\end{algorithm}
Since both the input state \(Z_t\) and the output averaged velocity are functions on a spatial domain, we parameterize \(w_\theta\) as a FNO-style conditional neural operator. 

Compared with a standard FNO parameterization, our implementation uses source-dependent encoded coordinates for the function variables. Specifically,
we use the encoded coordinates
\begin{equation}
    \tilde Y_{\mathrm{obs}}
    =
    \mathcal E_y(Y_{\mathrm{obs}})
    =
    \frac{Y_{\mathrm{obs}}-\mu_y}{\sigma_y},
    \qquad
    \tilde X
    =
    \mathcal E_x(X)
    =
    \begin{cases}
    X-\mu_x, & \rho=\rho_C,\\[2pt]
    (X-\mu_x)/\sigma_x, & \rho=\rho_W,
    \end{cases}
    \label{eq:normalization}
\end{equation}
where the statistics are estimated from the training set. For the
prior-aligned source \(\rho_C\), which is used in the proposed sampler, we
only center the discretized field to preserve the prior covariance structure.
For the white source \(\rho_W=\mathcal N(0,I)\), included only as an implementation-level comparison baseline, we additionally standardize the discretized field componentwise, so as to prevent the white source from having to learn the physical scales of the discretized field in addition to the posterior-to-source transformation. 
This normalization is the most effective
practical setting for each source choice, but it does not alter the source covariance structure.



Then, given an encoded input state \(z\), viewed as a realization of \(Z_t\) in
\eqref{eq:linear_path}, we first lift \(z\) to hidden features using a pointwise
lifting map \(\mathcal P\), apply \(L\) conditional operator layers
\(\Psi_1,\ldots,\Psi_L\), and then project the hidden features back to the
state variable using \(\mathcal Q\). Equivalently,
\begin{equation}\label{eq:conditional_operator_composition}
w_\theta(z,r,t;\tilde y_{\mathrm{obs}})
=
(
\mathcal Q
\,\circ\,
\Psi_L^{r,t,\tilde y_{\mathrm{obs}}}
\,\circ\,
\cdots
\,\circ\,
\Psi_1^{r,t,\tilde y_{\mathrm{obs}}}
\,\circ\,
\mathcal P
)(z).    
\end{equation}
Concretely, if \(\psi_\ell\) denotes the input hidden feature to the
\(\ell\)-th layer, then the conditional FNO layer
\(\Psi_\ell^{r,t,\tilde y_{\mathrm{obs}}}\) \cite{li2021fourier,cheng2026podno} takes the form
\begin{equation}
\psi_{\ell+1}
=
\Psi_\ell^{r,t,\tilde y_{\mathrm{obs}}}(\psi_\ell)
=
\mathcal F^{-1}
(
    \mathcal R_\ell(\mathcal F\psi_\ell)
)
+
\mathcal W_\ell\psi_\ell
+
\mathcal I_\ell(t,t-r,\tilde y_{\mathrm{obs}}),
\label{eq:fno_layer}
\end{equation}
where \(\mathcal F\) denotes the Fourier transform, \(\mathcal R_\ell\) is a
learned spectral multiplier on truncated Fourier modes, and
\(\mathcal W_\ell\) is a pointwise linear map.
The injection term
\(\mathcal I_\ell(t,t-r,\tilde y_{\mathrm{obs}})\) embeds the time variables
and the encoded observation, maps them onto the spatial grid, and adds the
resulting features to the hidden representation. See Figure~\ref{fig:conditional_operator_generator} for details.

The complete training and sampling procedures are then summarized in
Algorithms~\ref{alg:mf_training_encode_only}--\ref{alg:mf_sampling_encode_only} by a combination of analysis in Section~\ref{subsec:algebraic_derivation} and the encoding \eqref{eq:normalization}.

\section{Theory: noise geometry and transport regularity}
\label{sec:theory}
This section explains why the source geometry is central to
resolution-robust one-step posterior transport. 

For a white source variable
$\Xi_W=\sum_{k\ge1}g_k e_k,
~
g_k\sim\mathcal N(0,1)~\text{i.i.d.},$
we have
$\mathbb E\|\Xi_W\|_{\mathcal H}^2
=
\sum_{k\ge1}\mathbb E|g_k|^2
=
\infty,$
so \(\Xi_W\notin\mathcal H\) almost surely. By contrast, when by \eqref{eq:posterior_phi},
\(\pi(\cdot\mid y_{\mathrm{obs}})\ll\gamma=\mathcal N(0,\Lambda)\) with
\(\Lambda\) trace-class, posterior samples live in $\mathcal{H}$, the same function-space geometry
as the prior. Thus, the source distribution should
respect this prior-induced geometry.

This requirement is reminiscent of pCN
\cite{cotter2013mcmc}, where the Gaussian reference is used to construct a
dimension-robust Markov chain whose invariant measure is the target posterior.
In contrast, our sampler is a finite-time amortized transport map. For such a
one-step sampler, measure-level absolute continuity only identifies a common
measure-theoretic setting; it does not control the smoothness of the finite-time
transport map. Additional regularity is therefore needed to ensure that the
transport remains stable under discretization refinement.

Building on the Gaussian-tail framework \cite{xixian2026wasserstein} and the corresponding velocity-regularity analysis \cite{meng2025pathway}, we therefore extend this mechanism to the conditional posterior transport in Bayesian inverse problems. Under Assumption~\ref{assu:gaussian_tail_function_space}, the endpoint singularity along the linear path cancels in the \(C\)-weighted geometry, yielding dimension-independent Lipschitz bounds for the averaged velocity and the one-step map.
\begin{assumption}[Gaussian-tail assumption in Bayesian inverse problems]\label{assu:gaussian_tail_function_space}
Fix $y_{\mathrm{obs}}$ and let $\pi(\cdot| y_{\mathrm{obs}})$ be the target posterior on $\mathcal{H}$. Let \(\rho_C=\mathcal N(0,C)\) be the Gaussian source measure, where \(C:\mathcal H\to\mathcal H\) is self-adjoint, nonnegative, and trace-class.
Assume $\pi(\cdot|y_{\mathrm{obs}})$ satisfies a Gaussian-tail condition under the $C$-geometry:
there exists a (possibly degenerate) self-adjoint, nonnegative, trace-class operator
$A:\mathcal{H}\to \mathcal{H}$ such that $\pi(\cdot|y_{\mathrm{obs}})\ll \mathcal N(0,A)$ and
\begin{equation}\label{eqn:gtail}
\frac{d\pi(\cdot| y_{\mathrm{obs}})}{d\mathcal N(0,A)}(x) \propto \exp\big(R_A(x;y_{\mathrm{obs}})\big).    
\end{equation}
Here \(R_A(\cdot;y_{\mathrm{obs}})\) is twice Fr\'echet differentiable and
satisfies
\begin{equation}
\label{eq:M1_M2}
    M_1(y_{\mathrm{obs}})
    :=
    \|
    C^{1/2}\nabla R_A(\cdot;y_{\mathrm{obs}})
    \|_\infty
    <\infty,
    \quad
    M_2(y_{\mathrm{obs}})
    :=
    \|
    C\nabla^2 R_A(\cdot;y_{\mathrm{obs}})
    \|_\infty
    <\infty .
\end{equation}
Moreover, \(A\) and \(C\) induce the same Cameron--Martin space, and \(AC^{-1}\)
and \(A^{-1}C\) extend to bounded operators on this common Cameron--Martin
space.
\end{assumption}

For common PDE inverse problems, the Gaussian-tail assumption holds when \(\mathcal G_{\mathrm{obs}}:=\mathcal O\,\circ\,\mathcal G\in C_b^2\) and the source covariance is consistent with the prior covariance. The following theorem gives this sufficient condition.
\begin{theorem}[A sufficient condition for the Gaussian-tail assumption]
\label{thm:gaussian_tail_sufficient_condition}
Consider the additive Gaussian noise model whose posterior can be written
relative to the prior \(\gamma=\mathcal N(0,\Lambda)\) as in
\eqref{eq:posterior_phi}. Suppose that the source covariance \(C\) is consistent
with the prior covariance \(\Lambda\), in the sense that \(\Lambda\) and \(C\)
induce the same Cameron--Martin space and the comparison operators
\(\Lambda C^{-1}\) and \(\Lambda^{-1}C\) extend to bounded operators on this
common Cameron--Martin space. Assume further that
$\mathcal G_{\mathrm{obs}}
:=
\mathcal O\,\circ\,\mathcal G
:
\mathcal H\to\mathbb R^m$
is bounded with bounded first and second Fr\'echet derivatives:
\[
\|\mathcal G_{\mathrm{obs}}\|_{\infty}
+
\|D\mathcal G_{\mathrm{obs}}\|_{\infty}
+
\|D^2\mathcal G_{\mathrm{obs}}\|_{\infty}
<\infty .
\]
Then Assumption~\ref{assu:gaussian_tail_function_space} holds with
Gaussian reference covariance \(A=\Lambda\).
\end{theorem}

\begin{proof}
Recall the posterior expression \eqref{eq:posterior_phi}. Taking
\(A=\Lambda\), the posterior density relative to \(\mathcal N(0,A)\) is given
by the likelihood correction \(R_A(\cdot;y_{\mathrm{obs}})\). For arbitrary
perturbations \(\delta x,\delta x'\in\mathcal H\), differentiating the
log-likelihood term gives
\[
\begin{dcases}
DR_A(x;y_{\mathrm{obs}})[\delta x]
=
\left\langle
\Gamma^{-1}
\bigl(y_{\mathrm{obs}}-\mathcal G_{\mathrm{obs}}(x)\bigr),
D\mathcal G_{\mathrm{obs}}(x)[\delta x]
\right\rangle_{\mathbb R^m},\\
\begin{aligned}
D^2R_A(x;y_{\mathrm{obs}})[\delta x,\delta x']
=&
-
\left\langle
D\mathcal G_{\mathrm{obs}}(x)[\delta x],
\Gamma^{-1}
D\mathcal G_{\mathrm{obs}}(x)[\delta x']
\right\rangle_{\mathbb R^m}\\
+&
\left\langle
\Gamma^{-1}
\bigl(y_{\mathrm{obs}}-\mathcal G_{\mathrm{obs}}(x)\bigr),
D^2\mathcal G_{\mathrm{obs}}(x)[\delta x,\delta x']
\right\rangle_{\mathbb R^m}.    
\end{aligned}
\end{dcases}
\]
The boundedness of
\(\mathcal G_{\mathrm{obs}}\), \(D\mathcal G_{\mathrm{obs}}\), and
\(D^2\mathcal G_{\mathrm{obs}}\) therefore implies that \(DR_A\) and
\(D^2R_A\) are uniformly bounded as linear and bilinear forms on
\(\mathcal H\). By the Riesz representation theorem, the corresponding
Hilbert-space gradient \(\nabla R_A\) and Hessian \(\nabla^2R_A\) are
uniformly bounded.
Since \(C\) is a bounded operator, the weighted bounds
$\|C^{1/2}\nabla R_A(\cdot;y_{\mathrm{obs}})\|_\infty<\infty,
~
\|C\nabla^2R_A(\cdot;y_{\mathrm{obs}})\|_\infty<\infty$
follow. 

Finally, the assumed Cameron--Martin consistency between \(\Lambda\) and \(C\)
implies the Cameron--Martin compatibility between \(A\) and \(C\), since
\(A=\Lambda\). Therefore Assumption~\ref{assu:gaussian_tail_function_space}
holds.
\end{proof}

For linear inverse problems, however, the sufficient condition in Theorem~\ref{thm:gaussian_tail_sufficient_condition} generally does not hold. To see this, consider the linear-Gaussian inverse problem 
\begin{equation}\label{eq:linear_obs}
X\sim \mathcal N(0,\Lambda), \quad y=GX+\eta, \quad
\eta\sim\mathcal N(0,\Gamma).    
\end{equation}
As shown in \eqref{eq:posterior_phi}, the posterior can be written relative to the prior as
$    \pi(dx\mid y)
    \propto
    \exp\big(
        -\tfrac12\|Gx-y\|_{\Gamma^{-1}}^2
    \big)
    \mathcal N(0,\Lambda)(dx).$ Thus, if we take $A=\Lambda$, and $R_A(x):=-\frac12\|Gx-y\|_{\Gamma^{-1}}^2$,
then $\nabla R_A(x)
    =
    -G^*\Gamma^{-1}Gx+G^*\Gamma^{-1}y$,
which grows linearly in \(x\) unless \(G^*\Gamma^{-1}G=0\). This shows that the linear case requires a separate condition ensuring
Gaussian tail assumption. We present such a condition in the following Theorem \ref{thm:linear_source}.
\begin{theorem}[Linear-Gaussian auxiliary source]
\label{thm:linear_source}
Consider the linear-Gaussian inverse problem \eqref{eq:linear_obs}. Suppose
$    \Lambda=(-\Delta+\kappa^2I)^{-\alpha},
    ~\alpha>d,$
and suppose the source covariance \(C\) is consistent with the prior covariance
\(\Lambda\). Then there exists a Gaussian reference
\(\mathcal N(0,A)\), with
$    A:=\bigl(\Lambda^{-1}+G^*\Gamma^{-1}G\bigr)^{-1},
$ such that the posterior satisfies
Assumption~\ref{assu:gaussian_tail_function_space}.
\end{theorem}

\begin{proof}
We first verify the Cameron--Martin compatibility between \(A\) and \(C\).
To this end, we compare \(A\) with the prior covariance \(\Lambda\). Let
$    L_\Lambda:=\Gamma^{-1/2}G\Lambda^{1/2}.$
For point observations
$    Gx=(x(q_1),\ldots,x(q_m)),$
Sobolev embedding \(H^s(\Omega)\hookrightarrow C^0(\bar\Omega)\) gives
\[
    \|Gx\|_{\mathbb R^m}
    \le \sqrt m\,\|x\|_{L^\infty}
    \lesssim \|x\|_{H^s},
    \quad s>d/2 .
\]
Thus \(G\) is bounded on \(H^s(\Omega)\) for any \(s>d/2\). Since
$    \Lambda=(-\Delta+\kappa^2I)^{-\alpha},
~\alpha>d$, prior samples satisfy $X\in H^s(\Omega)$ a.s. for every $s<\alpha-d/2$. so that the point observations are well defined.
Hence, choosing \(s\) with \(d/2<s<\alpha-d/2\), the composition
\(G\Lambda^{1/2}\) is bounded. Since \(\Gamma^{-1/2}\) is bounded on the
finite-dimensional observation space, \(L_\Lambda\) is bounded.
Then, in the \(\Lambda\)-Cameron--Martin coordinates,
$    \Lambda^{-1/2}A\Lambda^{-1/2}
    =
    (I+L_\Lambda^*L_\Lambda)^{-1}.$
Since \(L_\Lambda\) is bounded, \(I+L_\Lambda^*L_\Lambda\) and
\((I+L_\Lambda^*L_\Lambda)^{-1}\) are boundedly invertible. Therefore \(A\)
and \(\Lambda\) have the same Cameron--Martin space with equivalent norms.

By the source-consistency assumption, \(\Lambda\) and \(C\) also have the same
Cameron--Martin space with equivalent norms, and \(\Lambda C^{-1}\) and
\(\Lambda^{-1}C\) are bounded on this common space. Hence \(A\) and \(C\) have
the same Cameron--Martin space with equivalent norms. Furthermore,
$    AC^{-1}
    =
    A\Lambda^{-1}\Lambda C^{-1},
$ and $
    A^{-1}C
    =
    A^{-1}\Lambda\Lambda^{-1}C$
are bounded on the same common Cameron--Martin space. Thus the
Cameron--Martin compatibility required in
Assumption~\ref{assu:gaussian_tail_function_space} holds. 

It remains to verify the density form and the weighted derivative bounds. With \(A:=(\Lambda^{-1}+G^*\Gamma^{-1}G)^{-1}\), we have
\[
-\tfrac12\langle x,\Lambda^{-1}x\rangle
    -\tfrac12\|Gx-y\|_{\Gamma^{-1}}^2 
=
    -\tfrac12
    \left\langle
        x,
        A^{-1}x
    \right\rangle
    +
    \langle x,G^*\Gamma^{-1}y\rangle
    + \mathrm{const}.
\]
Hence the quadratic term is absorbed into the Gaussian reference
\(\mathcal N(0,A)\), and the remaining correction is
$    R_A(x;y)=\langle x,G^*\Gamma^{-1}y\rangle+\mathrm{const}.$
Consequently,
$    \nabla R_A(x;y)=G^*\Gamma^{-1}y,
~
    \nabla^2 R_A(x;y)=0.$
Thus the weighted derivative bounds in
Assumption~\ref{assu:gaussian_tail_function_space} hold:
$    \|C^{1/2}\nabla R_A(\cdot;y)\|_\infty<\infty$,
    $\|C\nabla^2 R_A(\cdot;y)\|_\infty=0.$
\end{proof}
We now state the regularity result under the Gaussian tail assumption \ref{assu:gaussian_tail_function_space}, which is valid either \(\mathcal G_{\mathrm{obs}}\in C_b^2\), as in
Theorem~\ref{thm:gaussian_tail_sufficient_condition}, or the linear-Gaussian
point-observation setting with
\(\Lambda=(-\Delta+\kappa^2I)^{-\alpha}\), \(\alpha>d\), as in
Theorem~\ref{thm:linear_source}. 
For readability, the derivation is written in
finite-dimensional Galerkin notation. The estimates are expressed in the
\(C\)-weighted geometry, with constants independent of the Galerkin dimension
under the compatibility assumptions above.

\begin{theorem}[Resolution-robust regularity of the one-step transport]
\label{thm:one_step_transport_lipschitz}
Under Assumption~\ref{assu:gaussian_tail_function_space}, there exist finite
constants \(\Theta_0,\Theta_1\), depending only on the comparability
constants of \(A\) and \(C\), such that,
the averaged velocity \(w(\cdot,r,t;y_{\mathrm{obs}}):=
    \frac{1}{t-r}
    \int_r^t v(z,\tau;y_{\mathrm{obs}})\,\mathrm d\tau,
~
    0\le r<t\le 1\) is globally Lipschitz \(C\)-geometry and satisfies
\[
    \operatorname{Lip}
    \left(
        w(\cdot,r,t;y_{\mathrm{obs}})
    \right)
    \le
    L_v(y_{\mathrm{obs}}),\quad\mathrm{where}~L_v(y_{\mathrm{obs}})
:=
\Theta_0
+
\Theta_1\cdot(M_2(y_{\mathrm{obs}})
+
M_1(y_{\mathrm{obs}})^2).
\]
In particular, the one-step transport
$\mathcal T(z;y_{\mathrm{obs}})
:=
z-w(z,0,1;y_{\mathrm{obs}})$
is globally Lipschitz \(C\)-geometry and satisfies
\[
    \operatorname{Lip}
    \left(
        \mathcal T(\cdot;y_{\mathrm{obs}})
    \right)
    \le
    1+L_v(y_{\mathrm{obs}}).
\]
\end{theorem}

\begin{proof}
We first derive a bound for the marginal instantaneous velocity
\(v(\cdot,t;y_{\mathrm{obs}})\). 
Define
$\widehat B_t:=(1-t)^2A+t^2C$,
$K_t:=(1-t)A\widehat B_t^{-1}$,
$\Sigma_t:=A-(1-t)^2A\widehat B_t^{-1}A$.    
By Gaussian conditioning, the density of \(Z_t\) can be written as $p_t(z\mid y_{\mathrm{obs}})
\propto
\varphi_{\widehat B_t}(z)Q_t(z;y_{\mathrm{obs}}),$
where \(\varphi_{\widehat B_t}\) denotes the density of the Gaussian \(\mathcal N(0,\widehat B_t)\) and 
$Q_t(z;y_{\mathrm{obs}})
:=
\mathbb E_{\zeta_t\sim\mathcal N(0,\Sigma_t)}
\left[
    \exp\bigl(
        R_A(K_tz+\zeta_t;y_{\mathrm{obs}})
    \bigr)
\right].$
Hence
$S_t(z;y_{\mathrm{obs}})
=
-\widehat B_t^{-1}z
+
\nabla_z\log Q_t(z;y_{\mathrm{obs}}).$
Define the modified \(C\)-score
\begin{equation}
\label{eq:modified_score}
    \widetilde S_t(z;y_{\mathrm{obs}})
    :=
    C S_t(z;y_{\mathrm{obs}})
    +
    C\widehat B_t^{-1}z
    =
    C\nabla_z\log Q_t(z;y_{\mathrm{obs}}).
\end{equation}
Substituting this decomposition into \eqref{eq:velocity_score_identity} gives
\begin{equation}
\label{eq:v_modified_score}
    v(z,t;y_{\mathrm{obs}})
    =
    -\frac1{1-t}
    \big[
        \big(I-tC\widehat B_t^{-1}\big)z
        +
        t\widetilde S_t(z;y_{\mathrm{obs}})
    \big].
\end{equation}

Define the tilted probability measure \(\nu_z\) by
$    d\nu_z(\zeta)
    =
    \tfrac{
        \exp\bigl(R_A(K_tz+\zeta;y_{\mathrm{obs}})\bigr)
    }{
        Q_t(z;y_{\mathrm{obs}})
    }
    d\mathcal N(0,\Sigma_t)(\zeta).$
Then by \eqref{eq:modified_score},
$\widetilde S_t(z;y_{\mathrm{obs}})
=
CK_t^*
\mathbb E_{\nu_z}
\left[
    \nabla R_A(K_tz+\zeta;y_{\mathrm{obs}})
\right].$
Differentiating once more gives
\[
\nabla \widetilde S_t(z;y_{\mathrm{obs}})
=
CK_t^*
\left[
    \mathbb E_{\nu_z}
    \nabla^2R_A(K_tz+\zeta;y_{\mathrm{obs}})
    +
    \operatorname{Cov}_{\nu_z}
    \left(
        \nabla R_A(K_tz+\zeta;y_{\mathrm{obs}})
    \right)
\right]
K_t .
\]
Using the definitions of \(M_1(y_{\mathrm{obs}})\) and
\(M_2(y_{\mathrm{obs}})\) from \eqref{eq:M1_M2}, we have
$    \|
    C
    \mathbb E_{\nu_z}
    \nabla^2R_A(K_tz+\zeta;y_{\mathrm{obs}})
    \|
    \le
    M_2(y_{\mathrm{obs}}),
$ and $
    \|
    C^{1/2}
    \operatorname{Cov}_{\nu_z}
    (
        \nabla R_A(K_tz+\zeta;y_{\mathrm{obs}})
    )
    C^{1/2}
    \|
    \le
    M_1(y_{\mathrm{obs}})^2 .$
Therefore, \eqref{eq:v_modified_score} gives 
\[
    \operatorname{Lip}
    (
        v(\cdot,t;y_{\mathrm{obs}})
    )
    \le
    \Big\|
    \frac{1}{1-t}
    (
        I-tC\widehat B_t^{-1}
    )
    \Big\|_C 
    +
    \frac{t}{1-t}\cdot\chi_t\cdot\big(
    M_2(y_{\mathrm{obs}})
    +
    M_1(y_{\mathrm{obs}})^2\big)
\]
with $\chi_t
:=\|CK_t^*C^{-1}\|_C\|K_t\|_C=
\|C^{1/2}K_t^*C^{-1/2}\|
\|C^{-1/2}K_tC^{1/2}\|$. 

Set
$    D:=C^{-1/2}AC^{-1/2}$. We can write $C^{-1/2}(C\widehat B_t^{-1})C^{1/2}
    =
    ((1-t)^2D+t^2I)^{-1}
$ and $
    C^{-1/2}K_tC^{1/2}
    =
    (1-t)D((1-t)^2D+t^2I)^{-1}.$
 Since by compatibility, \(D\) is
bounded, positive, and boundedly invertible, the functional calculus gives
finite constants \(\Theta_0,\Theta_1\), depending only on the
comparability constants of \(A\) and \(C\), such that, for all \(t\in(0,1)\),
$    \|
    \frac{1}{1-t}
    (
        I-tC\widehat B_t^{-1}
    )
    \|_C
    \le
    \Theta_0
$, $
    \frac{t}{1-t}\cdot\chi_t
    \le
    \Theta_1
$.
Consequently,
\[
    \operatorname{Lip}
    \left(
        v(\cdot,t;y_{\mathrm{obs}})
    \right)
    \le
    \Theta_0
    +
    \Theta_1\cdot(M_2(y_{\mathrm{obs}})
    +
    M_1(y_{\mathrm{obs}})^2)
    =
    L_v(y_{\mathrm{obs}}).
\]

The averaged-velocity bound follows directly:
\[
\begin{aligned}
    \|
    w(z_1;r,t;y_{\mathrm{obs}})
    -
    w(z_2;r,t;y_{\mathrm{obs}})
    \|_C
    \le&
    \frac{1}{t-r}
    \int_r^t
    \|
    v(z_1,\tau;y_{\mathrm{obs}})
    -
    v(z_2,\tau;y_{\mathrm{obs}})
    \|_C
    d\tau  \\
    \le&
    L_v(y_{\mathrm{obs}})
    \|z_1-z_2\|_C.    
\end{aligned}
\]

Finally, the definition $\mathcal T(z;y_{\mathrm{obs}})=z-w(z,0,1;y_{\mathrm{obs}})$
implies
\begin{equation*}
\operatorname{Lip}
\left(
    \mathcal T(\cdot;y_{\mathrm{obs}})
\right)
\le
1+
\operatorname{Lip}
\left(
    w(\cdot,0,1;y_{\mathrm{obs}})
\right)
\le
1+L_v(y_{\mathrm{obs}}).    
\end{equation*}
\end{proof}
This theorem implies that when the source is consistent with the Matérn-type prior under the required regularity conditions, the averaged velocity and the resulting one-step transport remain uniformly Lipschitz under grid refinement. This supports the use of prior-aligned sources in function-space Bayesian inverse problems.




\section{Experiments}\label{sec:experiments}
In this section, we present numerical experiments to evaluate the performance of the proposed approach in solving Bayesian inverse problems. We first validate the choice of kernels in our model on a synthetic linear inverse problem in Section \ref{subsec:identity}. Then we turn to Bayesian inverse problems arising from various PDE models, including the  Darcy flow, Advection, Reaction--Diffusion, and Navier--Stokes in Section \ref{subsec:pde}.
\subsection{Validation on the linear inverse problem on choice of source kernels}\label{subsec:identity}
We first consider a synthetic linear inverse problem as a validation of our claim on source consistency in Theorem \ref{thm:one_step_transport_lipschitz}. The unknown field is drawn from the Gaussian prior
\(x\sim\mathcal N(0,(-\Delta+3^2I)^{-2.5})\). We consider a linear observation model, so that \eqref{eq:obs_model} becomes
\begin{equation}\label{eq:identity_forward}
\begin{aligned}
y_\mathrm{obs}=\mathcal{O}_\mathrm{rand}(x)+\eta,\quad \mathcal{O}_\mathrm{rand}: x \mapsto (x(q))_{q\in \mathcal{M}},\quad \eta\sim\mathcal{N}(0,0.1^2I).
\end{aligned}    
\end{equation}
where \(\mathcal{M}\) denotes the set of sensor locations. Here, we use noisy observations at

\begin{figure}[htbp!]
    \centering

    \subfigure[$x$]{
        \includegraphics[width=0.18\linewidth]{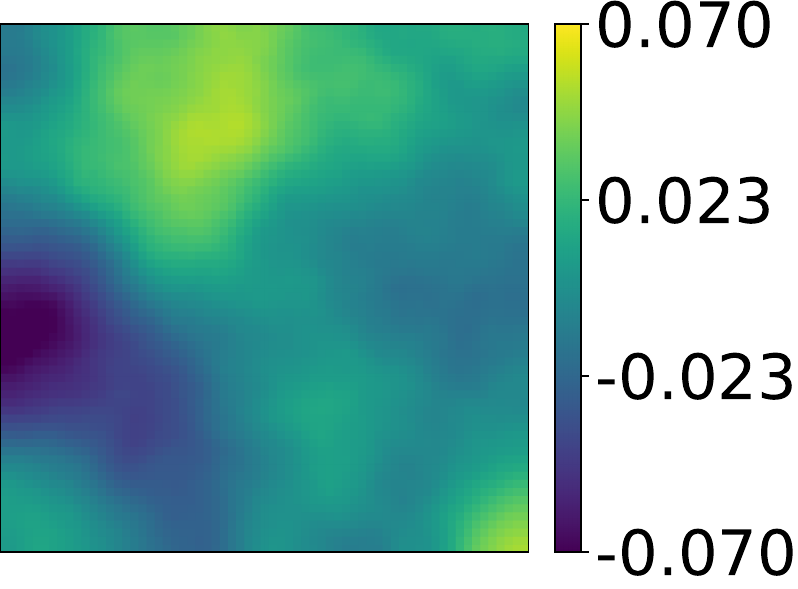}
        \label{fig:identity_x}
    }%
    \subfigure[sensors]{
        \includegraphics[width=0.18\linewidth]{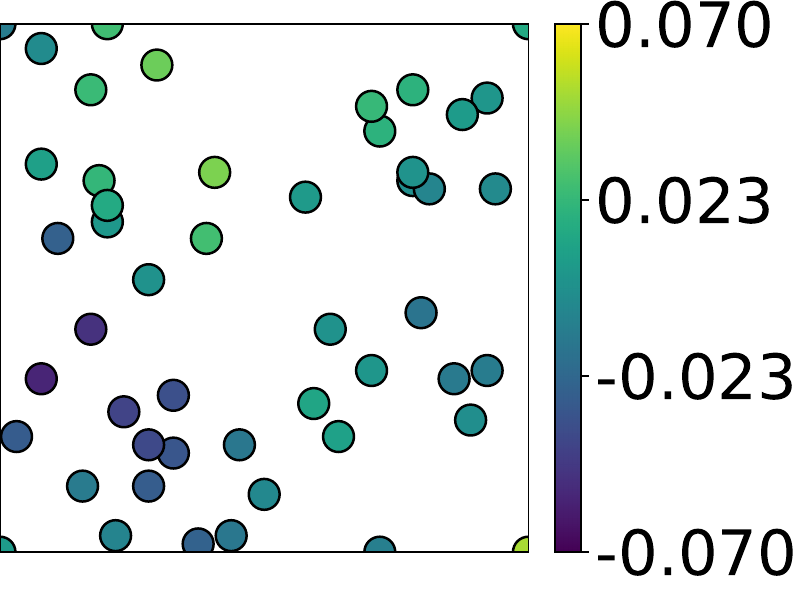}
        \label{fig:identity_Gobs_sensors}
    }%
    \subfigure[$\mu_{\mathrm{ref}}$]{
        \includegraphics[width=0.18\linewidth]{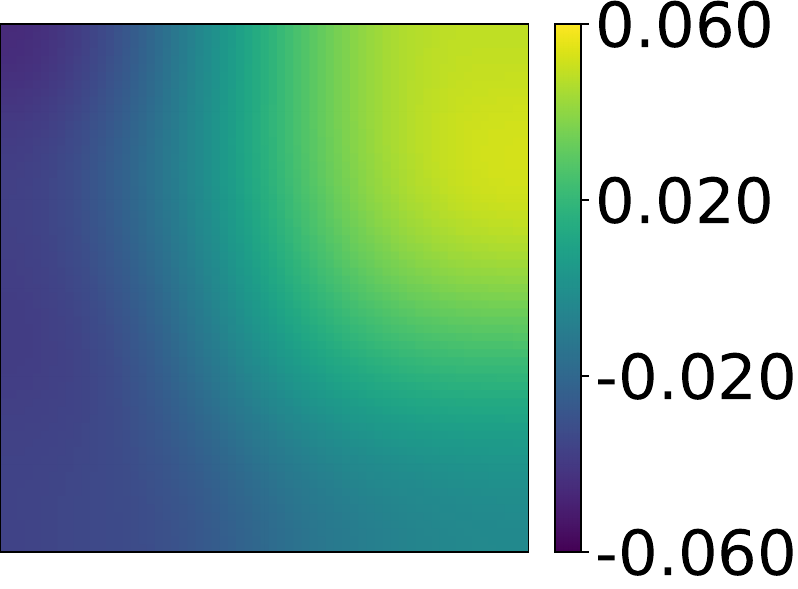}
        \label{fig:identity_mu_ref}
    }%
    \subfigure[$\mu_{\mathrm{pred}}$]{
        \includegraphics[width=0.18\linewidth]{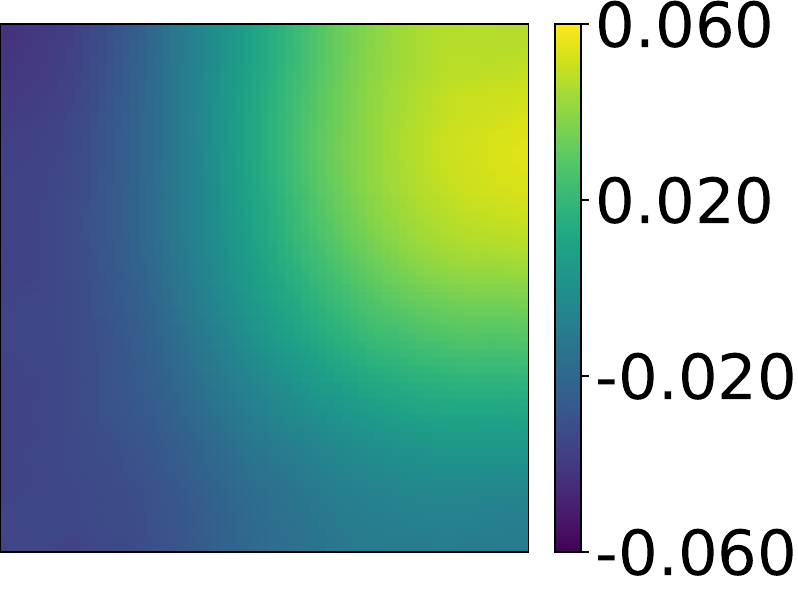}
        \label{fig:identity_mu_pred}
    }%
    \subfigure[$\mu_{\mathrm{pred}}-\mu_{\mathrm{ref}}$]{
        \includegraphics[width=0.18\linewidth]{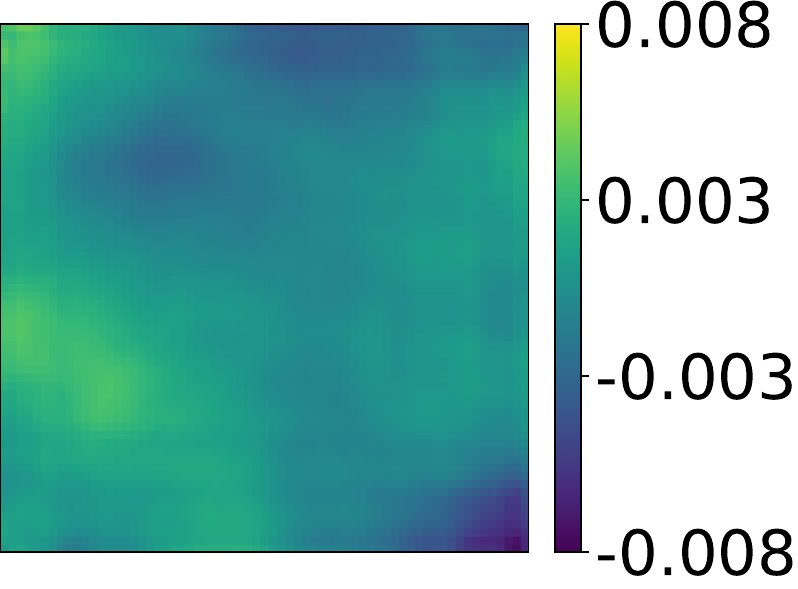}
        \label{fig:identity_mu_diff}
    }
    \subfigure[Posterior samples]{
            \centering
            \includegraphics[width=0.18\linewidth]{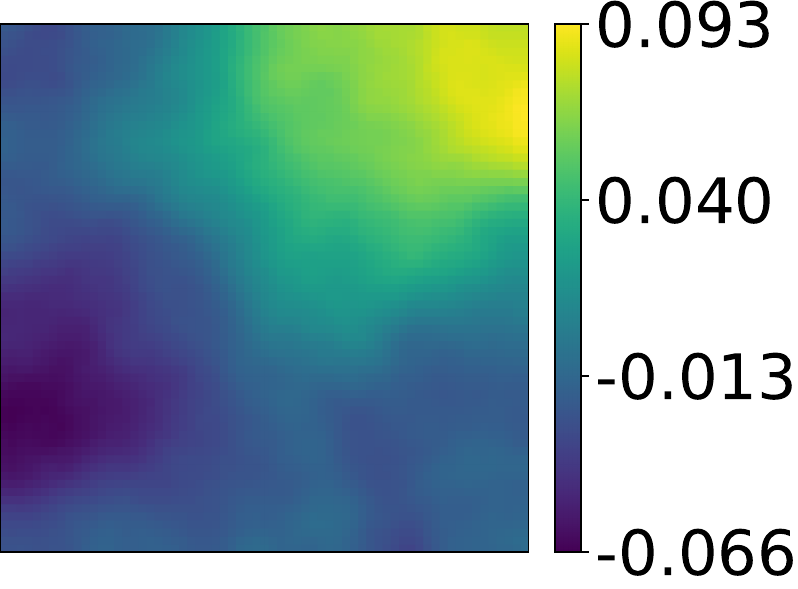}\hspace{2.5mm}
            \includegraphics[width=0.18\linewidth]{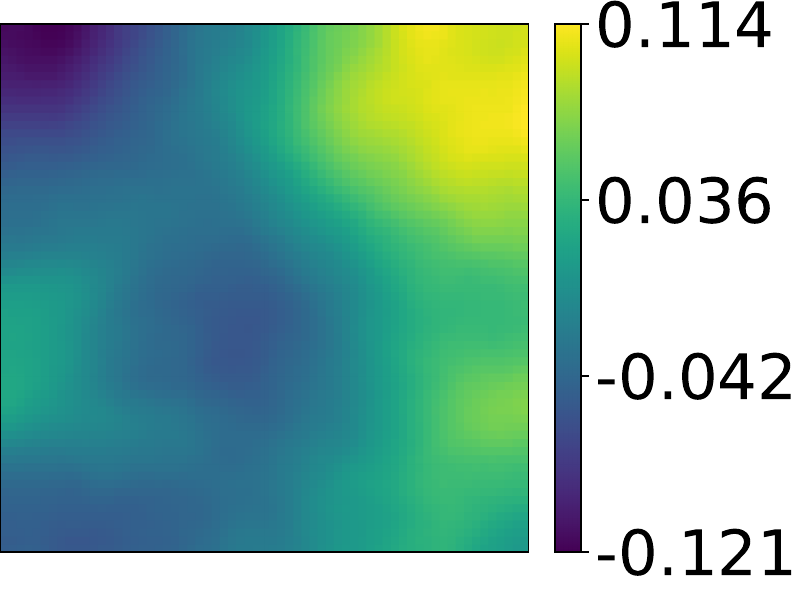}\hspace{2.5mm}
            \includegraphics[width=0.18\linewidth]{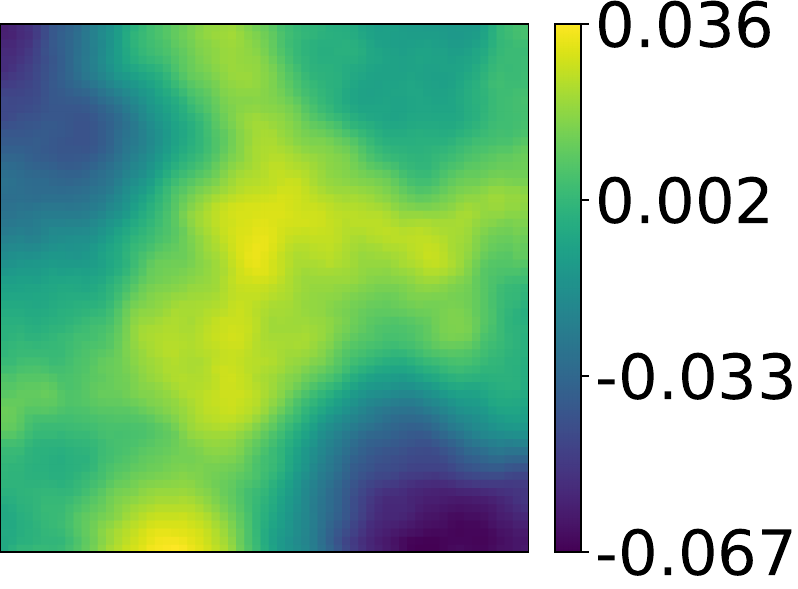}\hspace{2.5mm}
            \includegraphics[width=0.18\linewidth]{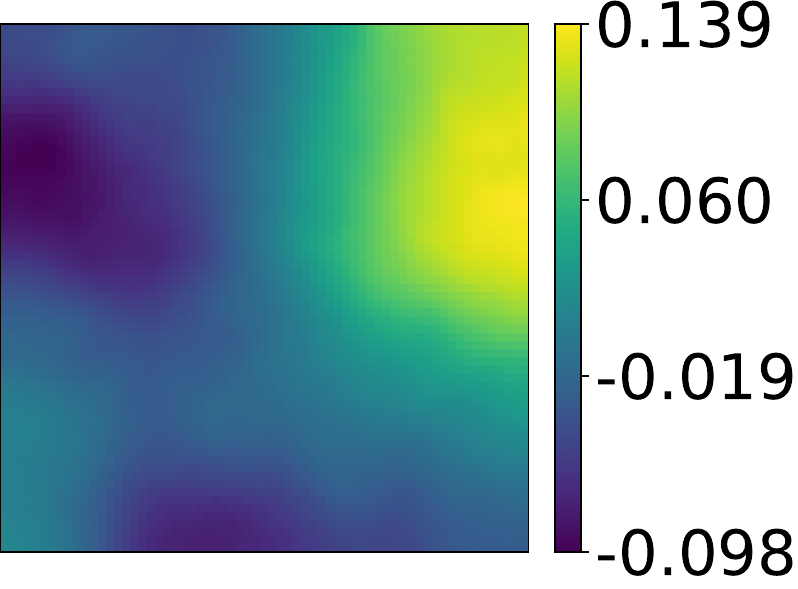}\hspace{2.5mm}
            \includegraphics[width=0.18\linewidth]{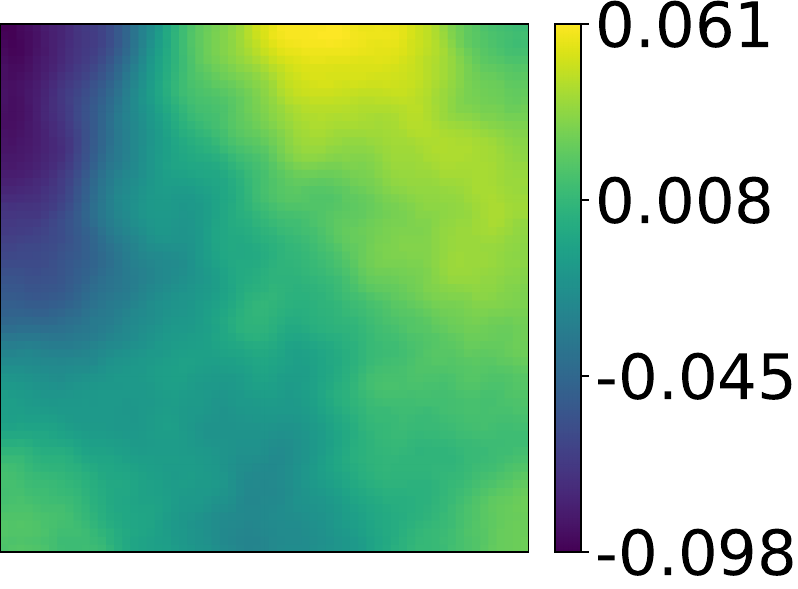}
        \label{fig:identity_samples}
    }

    \caption{Linear inverse problem. The first row shows the input field, sparse observations, reference posterior mean, predicted posterior mean, and mean error. The second row shows representative posterior samples generated by the proposed sampler.}
    \label{fig:identity_results}
\end{figure}
Figure~\ref{fig:choice_of_kernel} shows the energy spectra of the posterior for this linear inverse problem under different choices of source distribution kernels. More specifically, for any test observation $y_{\mathrm{obs}}$, the learned sampler induces an approximate posterior. Specifically, we sample $\xi^{(i)}\sim\rho$ i.i.d.\ and generate $N_\mathrm{eval}=10^3$ posterior samples
$x_{\mathrm{pred}}^{(i)}=\mathcal T_{\hat\theta}(\xi^{(i)};y_{\mathrm{obs}})$, $i=1,\ldots,N_{\mathrm{eval}}$, using Algorithm~\ref{alg:mf_sampling_encode_only}. We further let \(\hat x^{(i)}_{k_1,k_2}\) denotes its spectral coefficient at the two-dimensional frequency mode \((k_1,k_2)\), where \(k_1\) and \(k_2\) are the frequency indices in the two spatial directions. The radial frequency \(k\) groups all modes in the shell \(\mathcal S_k:=\{(k_1,k_2):\, k-\tfrac12\le
\sqrt{k_1^2+k_2^2}<k+\tfrac12\}\). We compute $E(k):=\frac{1}{|\mathcal S_k|}\sum_{(k_1,k_2)\in \mathcal S_k}\big(\frac{1} {N_\mathrm{eval}}\sum_{i=1}^{N_\mathrm{eval}}|\hat x^{(i)}_{k_1,k_2}|^2\big)$,
i.e., the average squared spectral energy over posterior samples and over all modes in the radial shell \(\mathcal S_k\). 

Here we compare four different source distributions, where \(\rho_C\) and \(\rho_1\) are consistent with the prior, while \(\rho_W\) and \(\rho_2\) are inconsistent:
\begin{equation}
\begin{dcases}
\rho_C=\mathcal N(0,(-\Delta+3^2I)^{-2.5})\\
\rho_W=\mathcal N(0,I) \\
\rho_1=\mathcal N(0,(-\Delta+I)^{-2.5})\\
\rho_2=\mathcal N\big(0,e^{-\frac{\|q-q'\|^2}{2\cdot 0.2^2}}\big).
\end{dcases}
\end{equation}
Specifically, $\rho_C$ is the natural choice, since the source distribution coincides with the Matérn-type anisotropic prior $\gamma$. $\alpha=\tfrac52$ is chosen from Theorem \ref{thm:linear_source}. The source $\rho_1$ is still prior-consistent in the sense that it has the same Mat\'ern smoothness order, and hence the same function-space regularity,
but with a mismatched length-scale parameter. The energy spectra show that the posteriors transported from the consistent sources, \(\rho_C\) and \(\rho_1\), closely match the closed-form reference posterior.

In contrast, $\rho_2$ is an RBF-type Gaussian source~\cite{seeger2004gaussian} with covariance
kernel $C_{\mathrm{RBF}}(q,q')=\exp\big(-\tfrac{\|q-q'\|^2}{2\cdot 0.2^2}\big).$ The corresponding spectral density decays exponentially,
$\widehat C_{\mathrm{RBF}}(k)
\propto
\exp\big(-\tfrac{0.2^2 |k|^2}{2}\big)$,
whereas the Matérn prior covariance $(-\Delta+3^2I)^{-2.5}$ has eigenvalues with polynomial decay,
$\widehat C_{\mathrm{Mat\acute{e}rn}}(k)
\propto
(|k|^2+3^2)^{-2.5}.$
Thus, although the RBF source defines a smoothing covariance operator, it imposes a substantially smoother geometry than the Mat\'ern prior. This over-smoothing makes \(\rho_2\) poorly aligned with the posterior geometry and, in our experiments, leads to the largest deviations from the exact closed-form posterior, especially in the high-frequency coefficients.

The white-noise source $\rho_W=\mathcal N(0,I)$ represents the opposite type of mismatch. Its covariance has no spectral decay and, therefore, places too much relative weight on high-frequency modes. Consequently, \(\rho_W\) is also not geometrically aligned with the posterior and thus also induces severe deviations.

\begin{figure}[htbp!]
    \centering
    \includegraphics[width=0.85\linewidth]{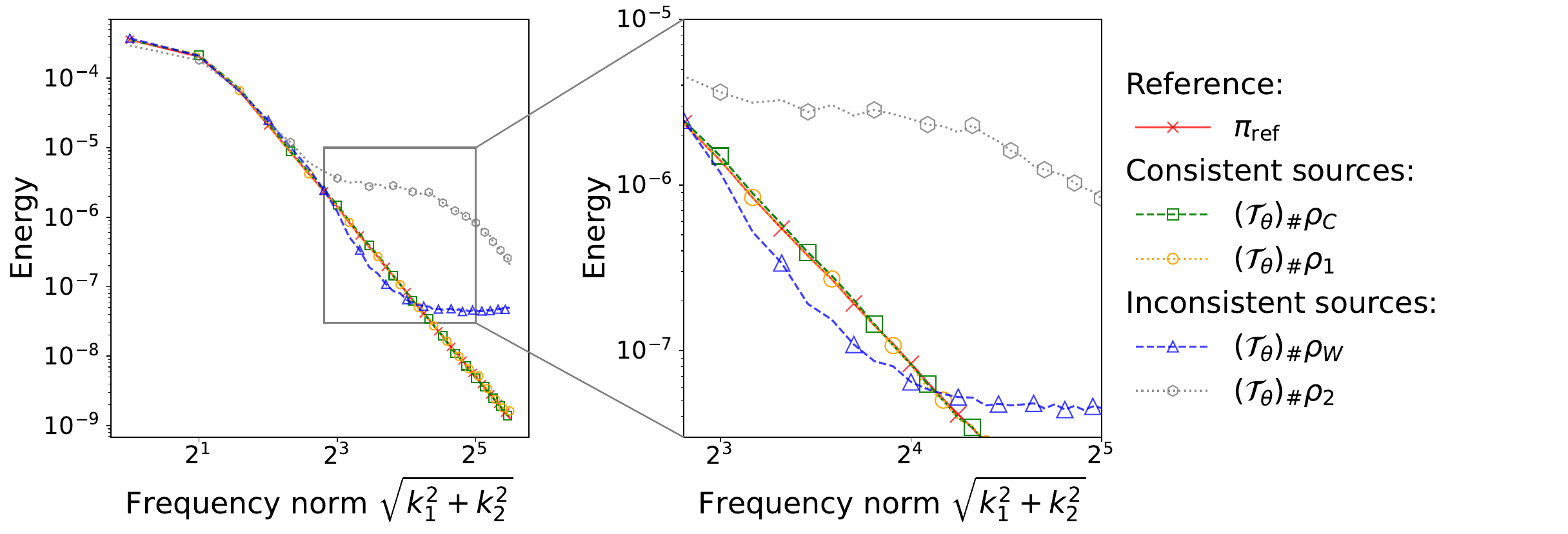}
    \caption{Energy spectra \(E(k)\) of the posterior distributions for the \(65\times65\) linear inverse problem under different choices of source kernels. }
    \label{fig:choice_of_kernel}
\end{figure}
After the qualitative spectral diagnostics, we now
define quantitative metrics for evaluating posterior quality. We use errors in two
standard posterior statistics: the posterior mean and the pointwise posterior
standard deviation. These metrics quantify the reconstructed field and its
uncertainty.

\begin{table}[!ht]
\centering
\caption{Effect of the source measure on posterior approximation for the linear inverse problem.}
\label{tab:benchmark_summary_identity}
\setlength{\tabcolsep}{6pt} 
\begin{tabular}{c c c }
\toprule
Source & $\epsilon_\mu^\mathrm{eval}$ & $\epsilon_\sigma^\mathrm{eval}$\\
\midrule
   $\rho_C$ &   $4.94\times10^{-2}$ & $2.72\times 10^{-2}$\\
   $\rho_W$   &  $8.24\times10^{-2}$ &$4.13\times10^{-2}$ \\
$\rho_{1}$ &   $6.28\times10^{-2}$&  $2.49\times10^{-2}$ \\
$\rho_{2}$ & $8.15\times10^{-2}$  &  $8.33\times10^{-2}$ \\
\bottomrule
\end{tabular}
\end{table}
For $N_\mathrm{eval}=10^3$ posterior samples
$x_{\mathrm{pred}}^{(i)}$ for the observation $y_\mathrm{obs}$, we estimate the posterior mean field $\mu_{\mathrm{pred}}$ and the (pointwise) standard deviation field $\sigma_{\mathrm{pred}}$ by:
 \vspace{-0.5em}
\begin{equation}
\mu_{\mathrm{pred}}
= \frac{1}{N_\mathrm{eval}}\sum_{i=1}^{N_\mathrm{eval}} x^{(i)}_{\mathrm{pred}},
\quad
\sigma_{\mathrm{pred}}
= \sqrt{\frac{1}{N_\mathrm{eval}-1}\sum_{i=1}^{N_\mathrm{eval}}(x^{(i)}_{\mathrm{pred}}-\mu_{\mathrm{pred}})^2},
\label{eq_std}
\end{equation}
where operations are taken pointwise over the spatial grid.

The reference posterior statistics \(\mu_{\mathrm{ref}}\) and
\(\sigma_{\mathrm{ref}}\) are computed in closed form. With a slight abuse of notation, we write the observation model after discretization as
$y_{\mathrm{obs}}
=
Gx+\eta,
~
\eta\sim\mathcal N(0,\sigma_\eta^2 I_m),$
where \(G\) is the matrix representation of the subsampling operator. Since the
prior is Gaussian, \(x\sim\mathcal N(0,C)\), the posterior is also Gaussian and
satisfies
\begin{equation}\label{eq:identity_closed_form_posterior}
\mu_{x\mid y_{\mathrm{obs}}}
=
CG^\top
(GCG^\top+\sigma_\eta^2 I_m)^{-1}
y_{\mathrm{obs}},
\quad
\Sigma_{x\mid y_{\mathrm{obs}}}
=
C
-
CG^\top
(GCG^\top+\sigma_\eta^2 I_m)^{-1}
GC .
\end{equation}
Thus, we take
$\mu_{\mathrm{ref}}
:=
\mu_{x\mid y_{\mathrm{obs}}}$ and
$\sigma_{\mathrm{ref}}
:=
\sqrt{\operatorname{diag}
(\Sigma_{x\mid y_{\mathrm{obs}}})}.$

Next, we quantify discrepancies using relative $L^2$ errors for both mean and uncertainty:
\begin{equation}\label{eq:posterior_statistics}
\epsilon_\mu^\mathrm{eval}
=
\frac{\| \mu_{\mathrm{pred}} - \mu_{\mathrm{ref}} \|_{L^2}}{\| \mu_{\mathrm{ref}} \|_{L^2}},\quad \epsilon_\sigma^\mathrm{eval}
=
\frac{\| \sigma_{\mathrm{pred}} - \sigma_{\mathrm{ref}} \|_{L^2}}{\| \sigma_{\mathrm{ref}} \|_{L^2}}.
\end{equation}
Table~\ref{tab:benchmark_summary_identity} reports these statistics and shows
that exact agreement between the source covariance and the prior covariance is
not necessary, whereas consistency with the prior geometry is crucial. Here,
the anisotropic source refers to the prior-aligned source
\(\rho_C=\mathcal N(0,C)\) with \(C=\Lambda\). In addition,
Figure~\ref{fig:identity_results}\subref{fig:identity_mu_ref}--\subref{fig:identity_mu_diff}
visualizes these statistics for the anisotropic source, while Figure~\ref{fig:identity_results}\subref{fig:identity_samples} shows some
posterior samples generated from the same source.

\subsection{PDE inverse problems}
\label{subsec:pde}
In this subsection, we evaluate the model on four PDE inverse problems. We first
describe the forward observation models.
\begin{enumerate}[leftmargin=*, label=(\arabic*),nosep]
    \item \underline{Darcy flow.}
We infer the positive permeability $a:\Omega\to\mathbb{R}^+$ on $\Omega=[0,1]^2$ from noisy observations of the hydraulic
head $u$ satisfying the steady Darcy equation
\vspace{-0.5em}
\begin{equation}
\begin{cases}
-\nabla\cdot\bigl(a\nabla u\bigr)=f  &\text{in }\Omega,\\
u=0  &\text{on }\partial\Omega.
\end{cases}
\label{eq:darcy_pde}
\end{equation}
with forward map $\mathcal G(\log a)=u$ and $f$ is fixed as a three-layer piecewise constant function taking values \(1000,2000,3000\) on \([0,1]\times[0,4/6]\), \([0,1]\times(4/6,5/6]\), and \([0,1]\times(5/6,1]\), respectively.

To enforce positivity, we place a Gaussian prior on $\log a$:
$\log a\sim\mathcal N_\mathrm{Neumann}(0,(-\Delta+3^2I)^{-2})$. See Figure~\ref{fig:darcy_results}\subref{fig:darcy_x}.
We then discretize \eqref{eq:darcy_pde} by a finite-volume TPFA scheme, yielding a sparse SPD linear system, in the $65\times65$ grid. Observations follow $y_{\mathrm{obs}}=\mathcal{O}_\mathrm{int}(u)+\eta$ where $\eta\sim\mathcal{N}(0,I_m)$. $\mathcal{O}_\mathrm{int}$ denotes an interior equidistant subsampling operator that avoids boundary points, producing $7\times7$ measurements
($m=49$) See Figure~\ref{fig:darcy_results}\subref{fig:darcy_Gobs_sensors}.
\end{enumerate}
\begin{figure}[htbp!]
    \centering

    \subfigure[$\log a$]{
        \includegraphics[width=0.18\linewidth]{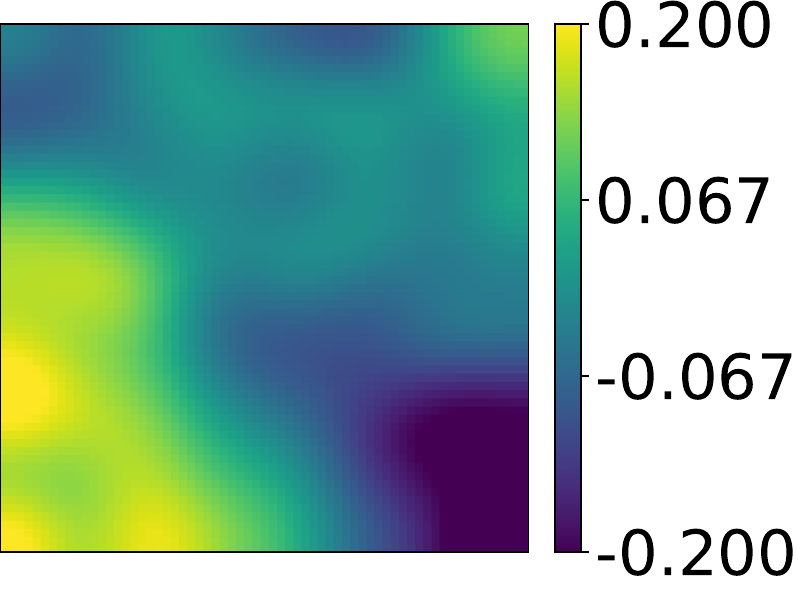}
        \label{fig:darcy_x}
    }%
    \subfigure[sensors]{
        \includegraphics[width=0.18\linewidth]{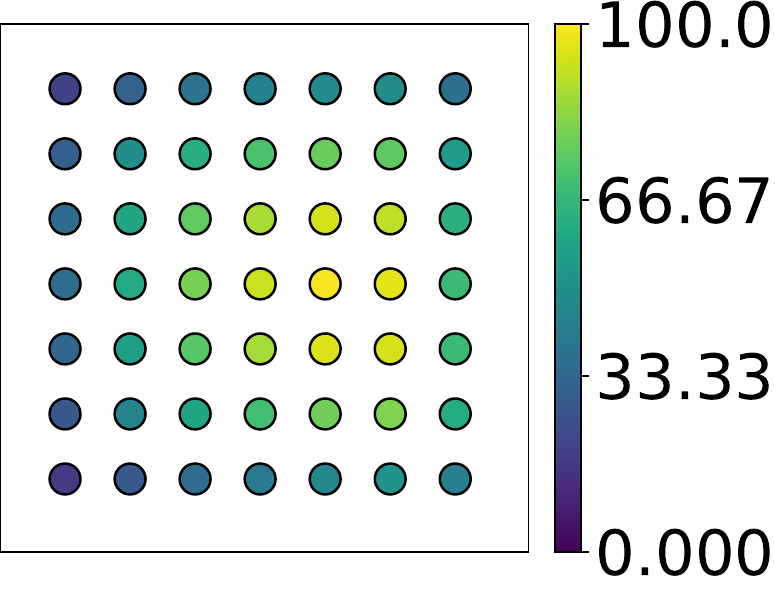}
        \label{fig:darcy_Gobs_sensors}
    }%
    \subfigure[$\mu_{\mathrm{ref}}$]{
        \includegraphics[width=0.18\linewidth]{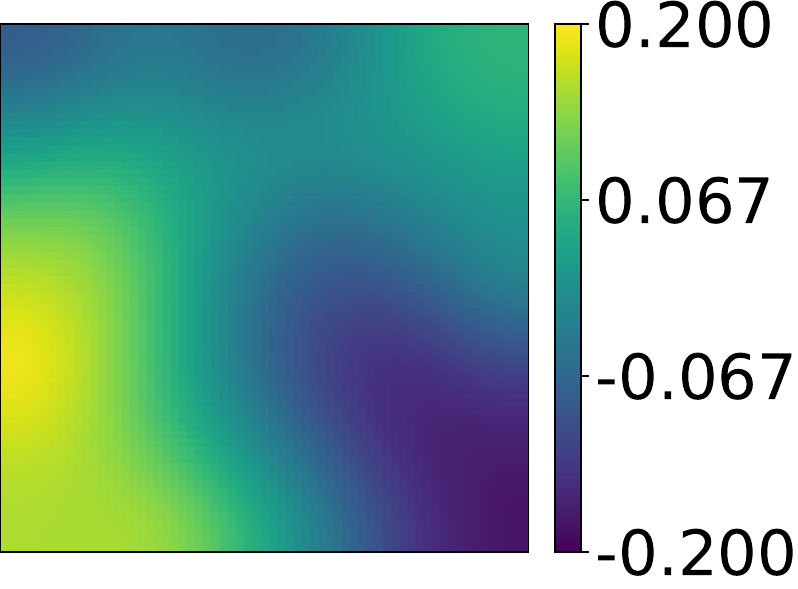}
        \label{fig:darcy_mu_ref}
    }%
    \subfigure[$\mu_{\mathrm{pred}}$]{
        \includegraphics[width=0.18\linewidth]{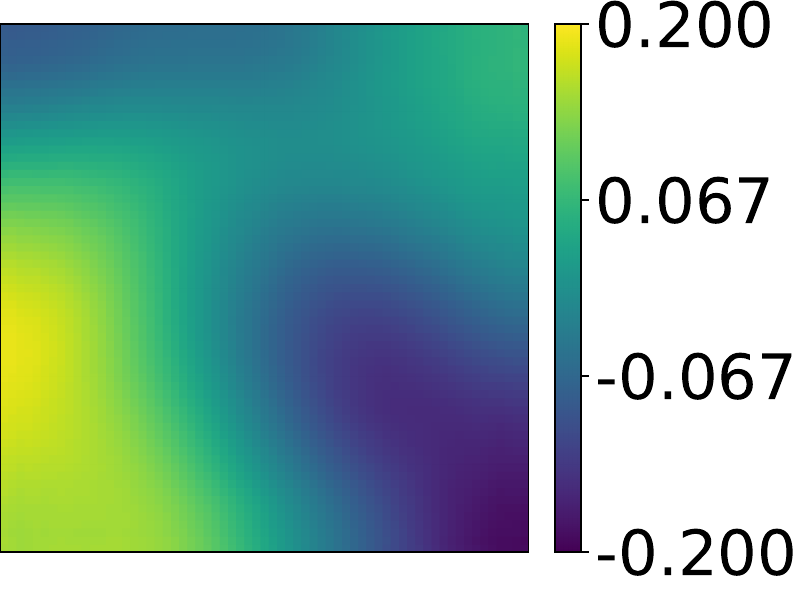}
        \label{fig:darcy_mu_pred}
    }%
    \subfigure[$\mu_{\mathrm{pred}}-\mu_{\mathrm{ref}}$]{
        \includegraphics[width=0.18\linewidth]{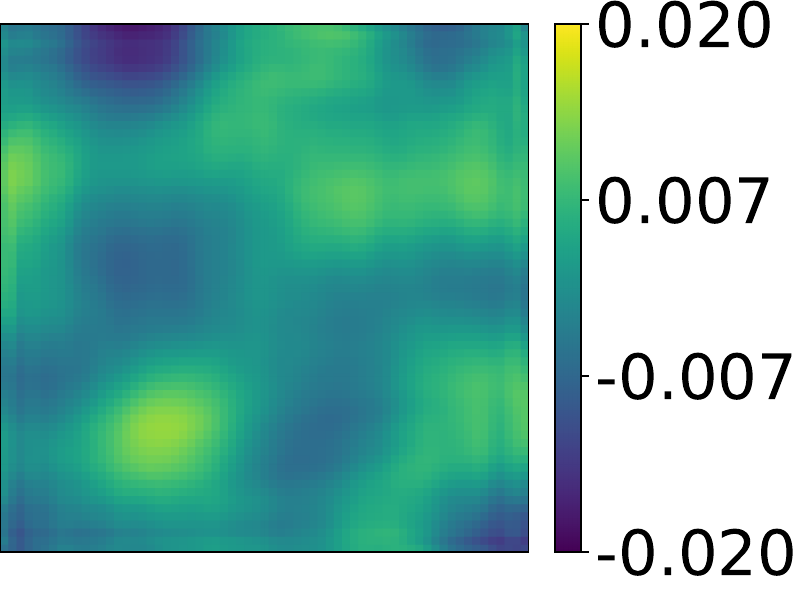}
        \label{fig:darcy_mu_diff}
    }
    \subfigure[Posterior samples]{
            \centering
            \includegraphics[width=0.18\linewidth]{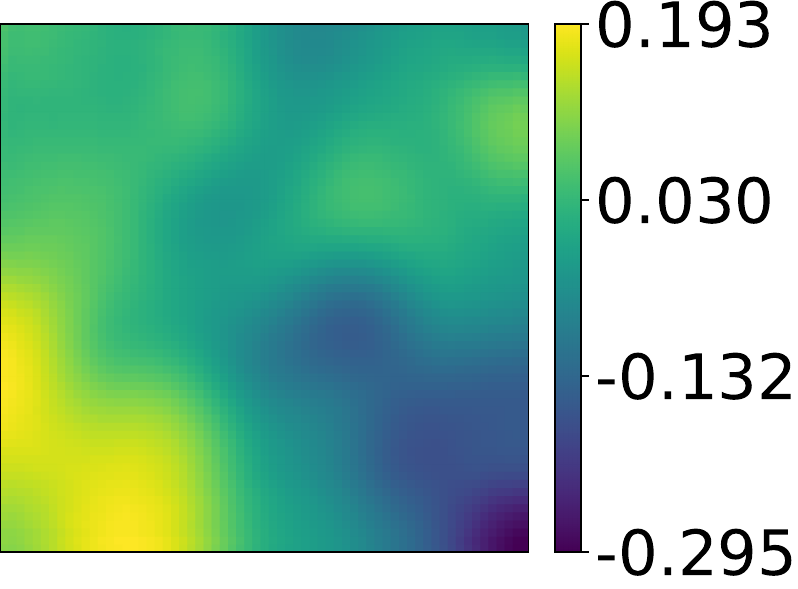}\hspace{2.5mm}
            \includegraphics[width=0.18\linewidth]{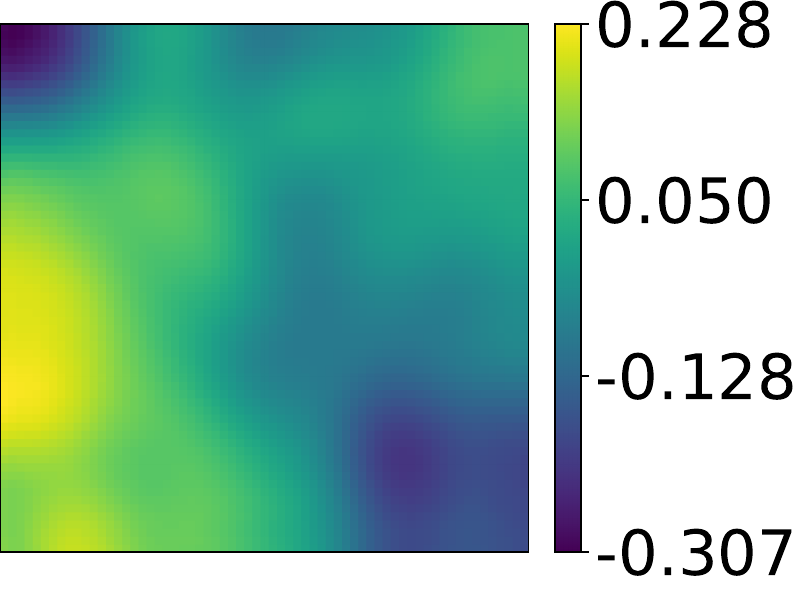}\hspace{2.5mm}
            \includegraphics[width=0.18\linewidth]{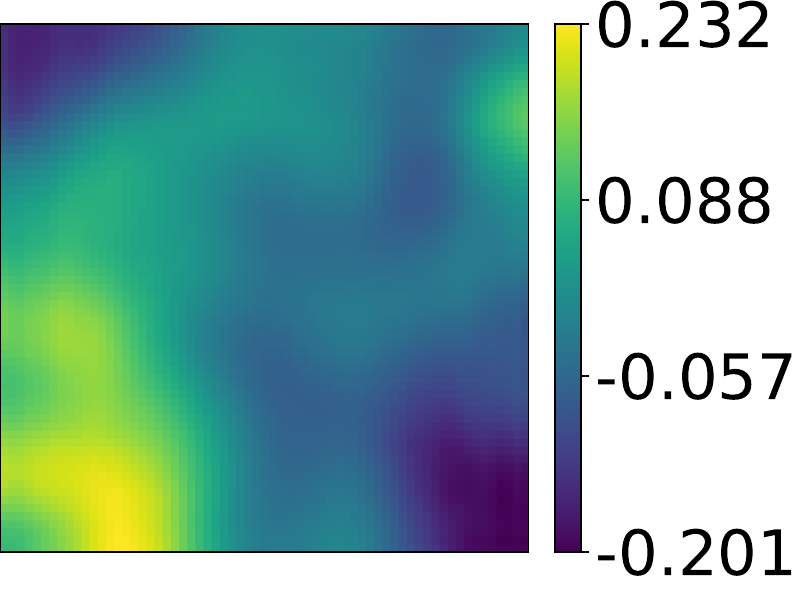}\hspace{2.5mm}
            \includegraphics[width=0.18\linewidth]{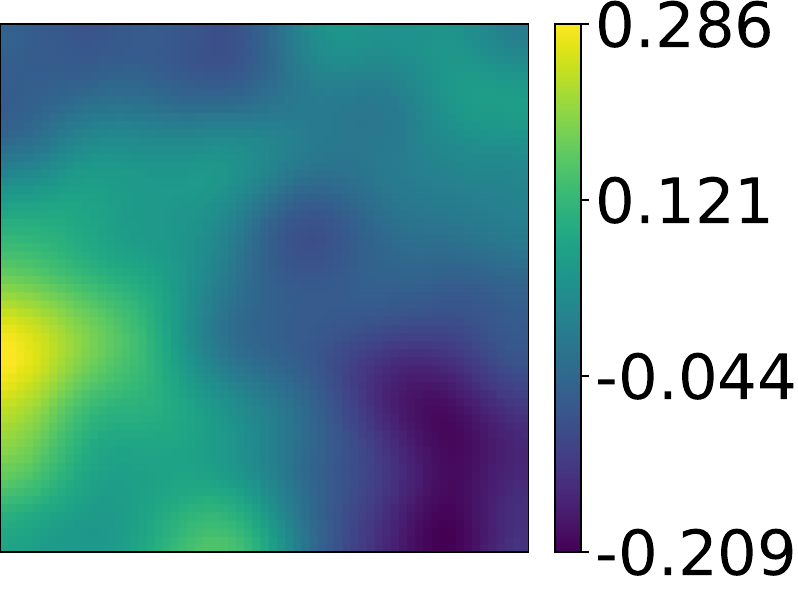}\hspace{2.5mm}
            \includegraphics[width=0.18\linewidth]{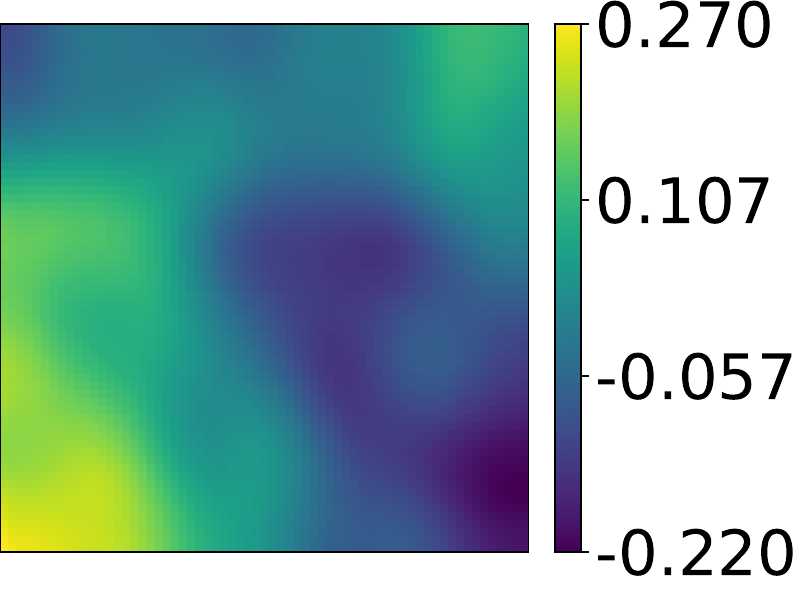}
        \label{fig:darcy_samples}
    }

    \caption{Darcy inverse problem. The first row shows the unknown field, sparse observations, reference posterior mean, predicted posterior mean, and mean error. The second row shows representative posterior samples generated by the proposed sampler.}
    \label{fig:darcy_results}
\end{figure}
 In the following fluid equations (Advection, RDE, NSE), we consider a periodic domain $\Omega=[0,1)^2$ and time variable $\tau>0$. The unknown field is drawn from the periodic Gaussian prior
$\mathcal N_{\mathrm{per}}(0,(-\Delta+3^2I)^{-2}).$ All experiments use a $64\times64$ grid. The observation operator $\mathcal O_{\mathrm{per}}$ uniformly subsamples the full periodic grid (without removing boundary points) at $8\times8$ locations ($m=64$). The standard deviation of the observation noise $\eta$ is chosen to be adapted to the output function value, specified individually.
 \begin{figure}[htbp!]
    \centering

    \subfigure[$\log u_0$]{
        \includegraphics[width=0.18\linewidth]{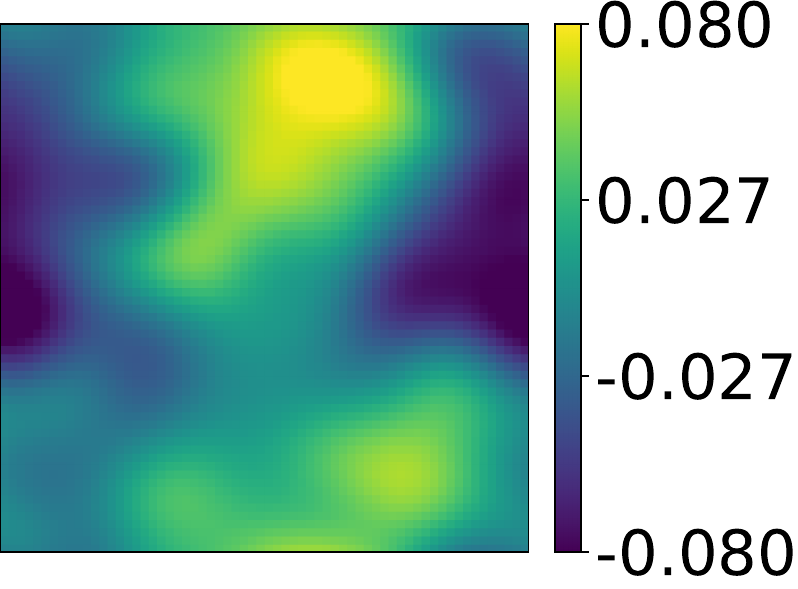}
        \label{fig:advection_x}
    }%
    \subfigure[sensors]{
        \includegraphics[width=0.18\linewidth]{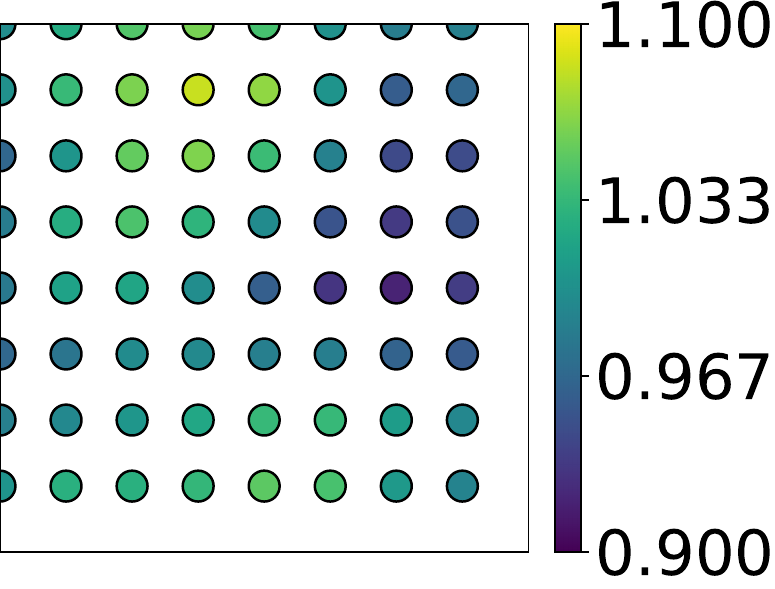}
        \label{fig:advection_Gobs_sensors}
    }%
    \subfigure[$\mu_{\mathrm{ref}}$]{
        \includegraphics[width=0.18\linewidth]{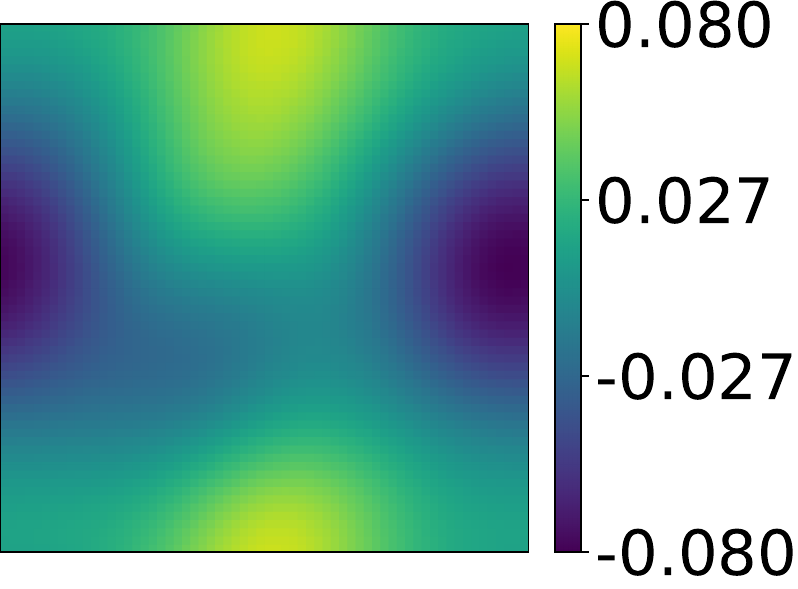}
        \label{fig:advection_mu_ref}
    }%
    \subfigure[$\mu_{\mathrm{pred}}$]{
        \includegraphics[width=0.18\linewidth]{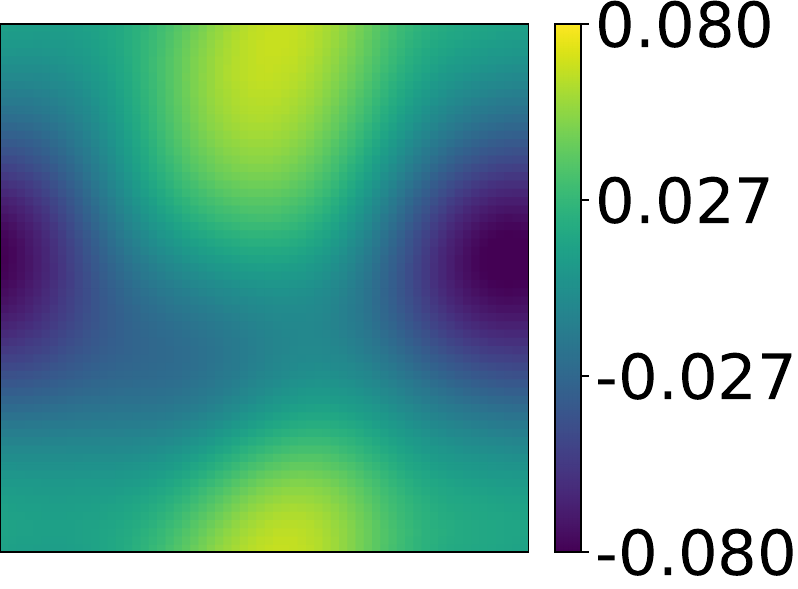}
        \label{fig:advection_mu_pred}
    }%
    \subfigure[$\mu_{\mathrm{pred}}-\mu_{\mathrm{ref}}$]{
        \includegraphics[width=0.18\linewidth]{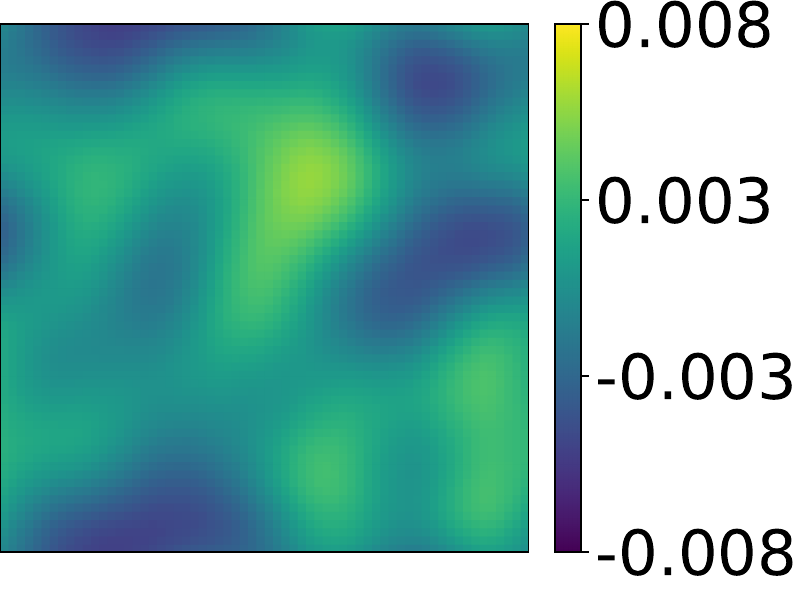}
        \label{fig:advection_mu_diff}
    }
    \subfigure[Posterior samples]{
            \centering
            \includegraphics[width=0.18\linewidth]{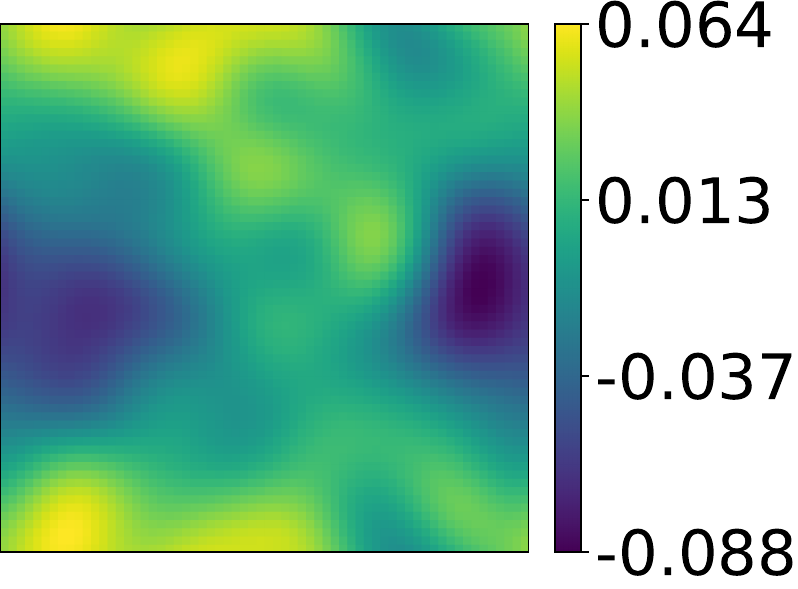}\hspace{2.5mm}
            \includegraphics[width=0.18\linewidth]{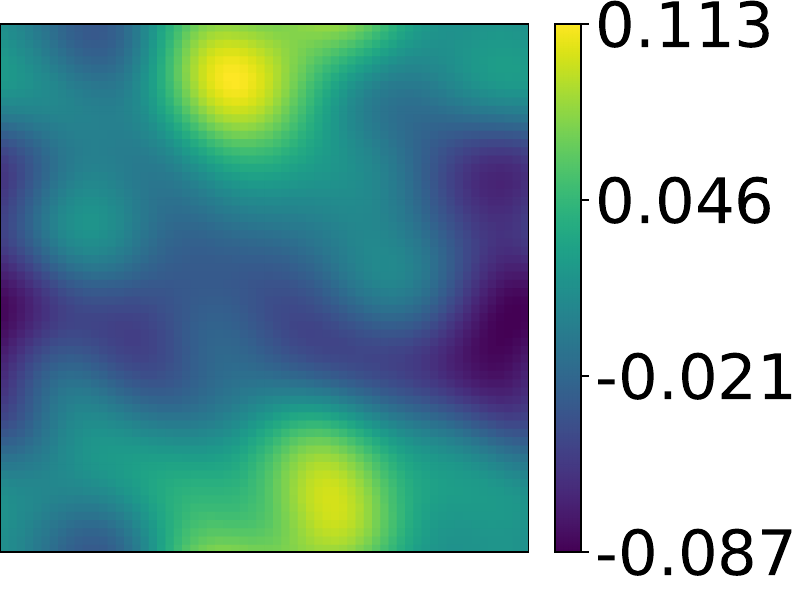}\hspace{2.5mm}
            \includegraphics[width=0.18\linewidth]{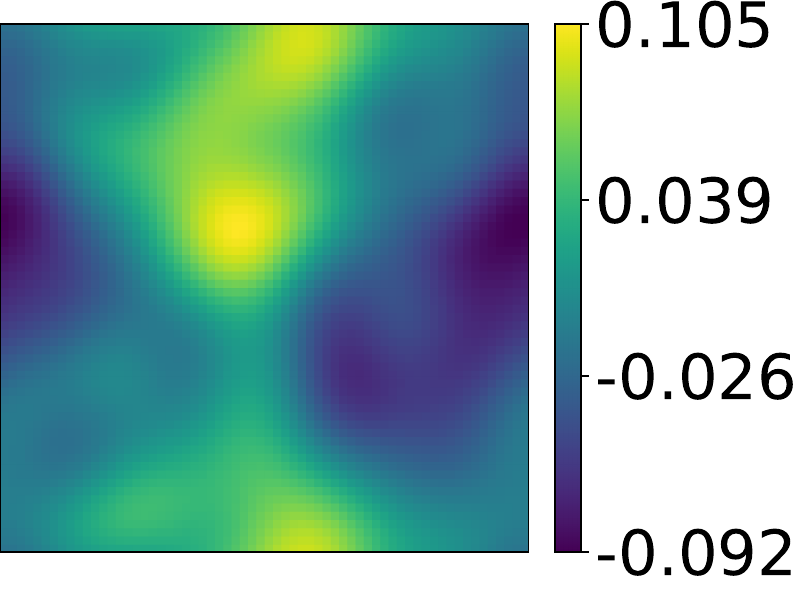}\hspace{2.5mm}
            \includegraphics[width=0.18\linewidth]{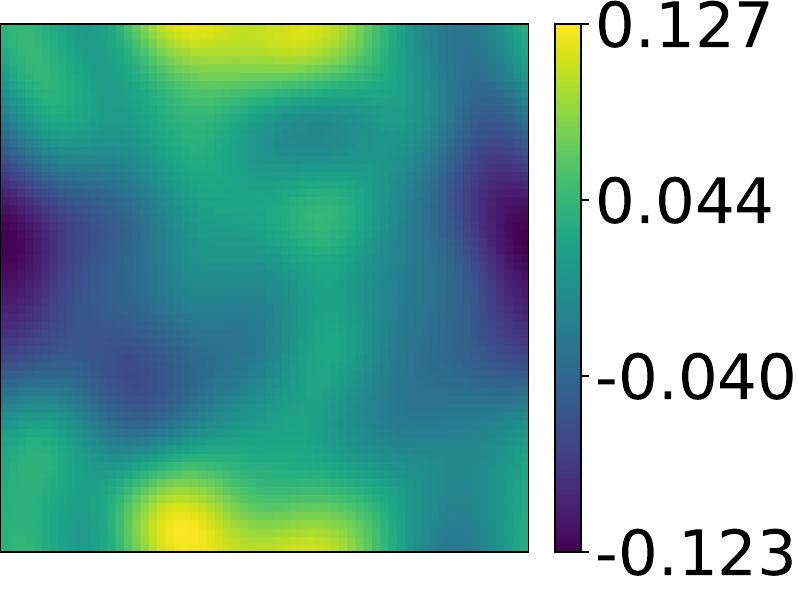}\hspace{2.5mm}
            \includegraphics[width=0.18\linewidth]{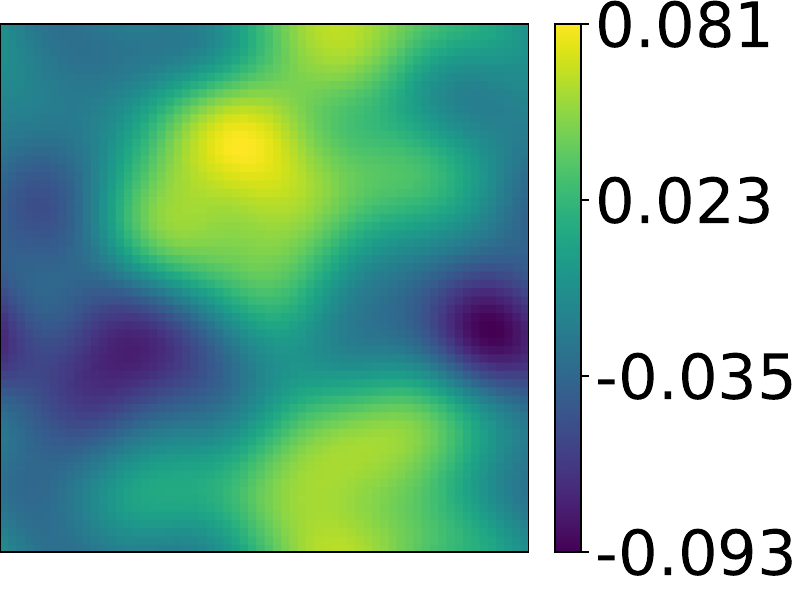}
        \label{fig:advection_samples}
    }

    \caption{Advection inverse problem. The first row shows the unknown field, sparse observations, reference posterior mean, predicted posterior mean, and mean error. The second row shows representative posterior samples generated by the proposed sampler.}
    \label{fig:advection_results}
\end{figure}
\begin{figure}[ht!]
    \centering

    \subfigure[$u_0$]{
        \includegraphics[width=0.18\linewidth]{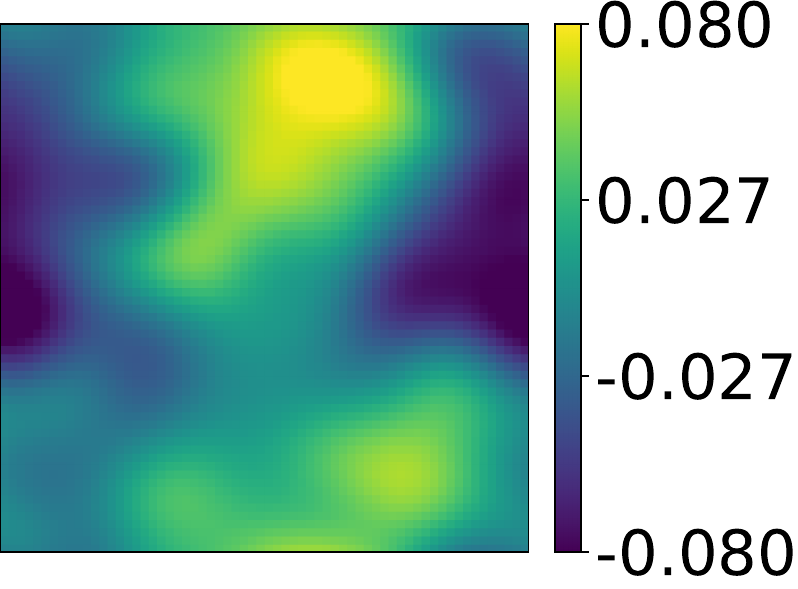}
        \label{fig:rd_x}
    }%
    \subfigure[sensors]{
        \includegraphics[width=0.18\linewidth]{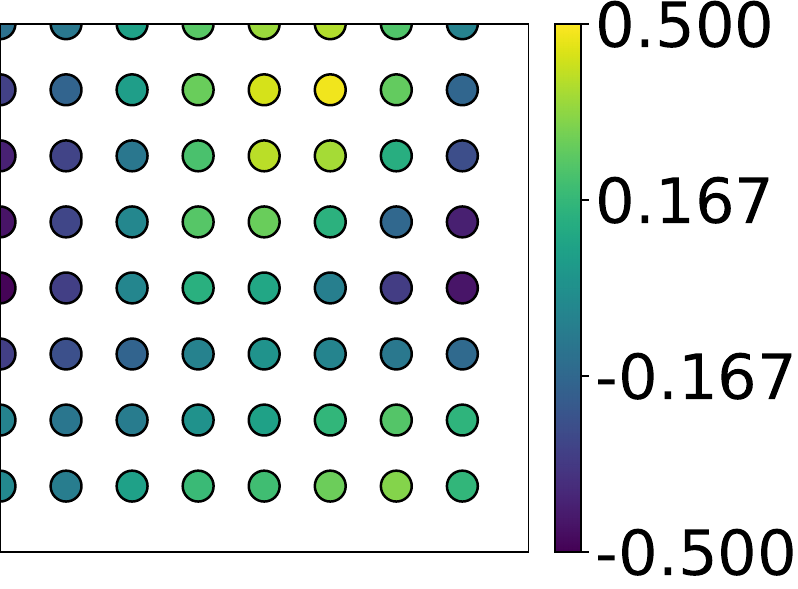}
        \label{fig:rd_Gobs_sensors}
    }%
    \subfigure[$\mu_{\mathrm{ref}}$]{
        \includegraphics[width=0.18\linewidth]{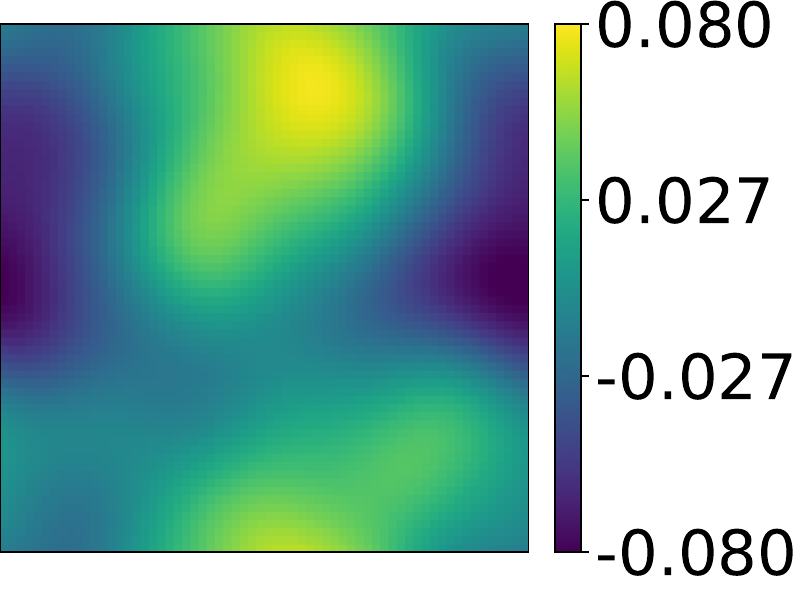}
        \label{fig:rd_mu_ref}
    }%
    \subfigure[$\mu_{\mathrm{pred}}$]{
        \includegraphics[width=0.18\linewidth]{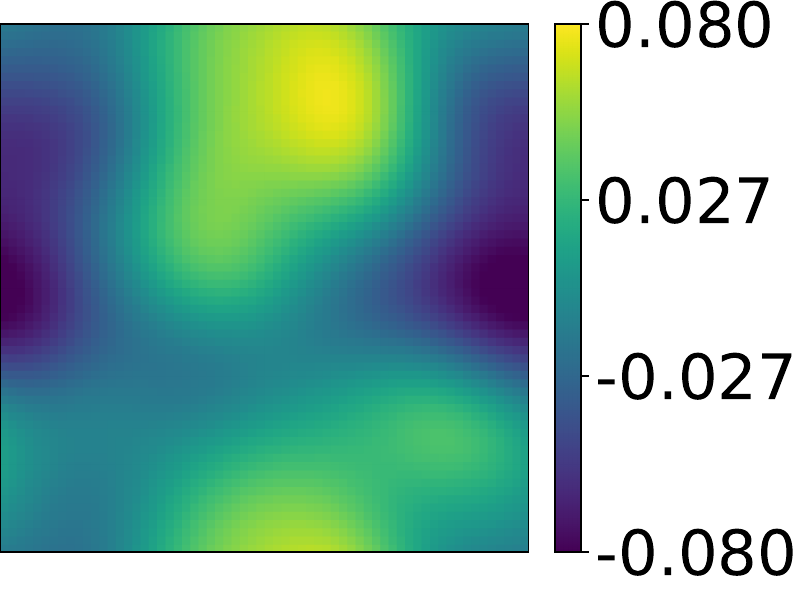}
        \label{fig:rd_mu_pred}
    }%
    \subfigure[$\mu_{\mathrm{pred}}-\mu_{\mathrm{ref}}$]{
        \includegraphics[width=0.18\linewidth]{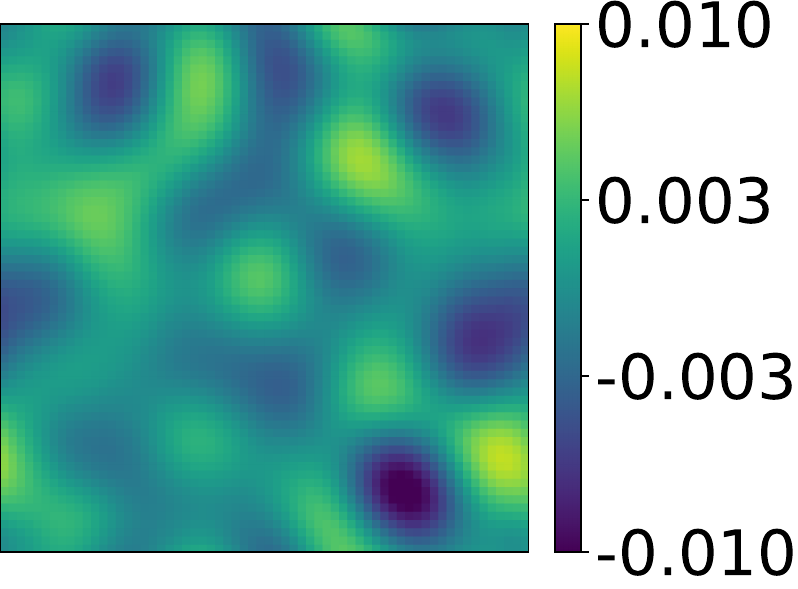}
        \label{fig:rd_mu_diff}
    }
    \subfigure[Posterior samples]{
            \centering
            \includegraphics[width=0.18\linewidth]{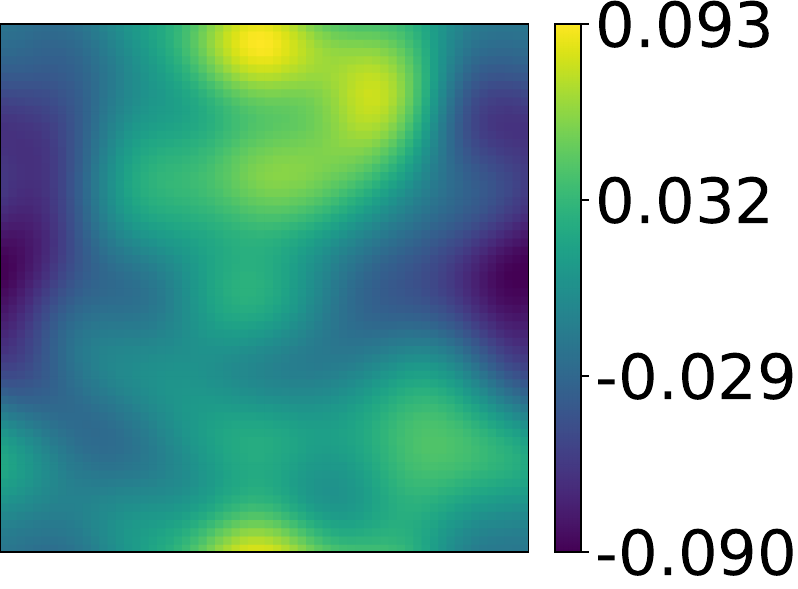}\hspace{2.5mm}
            \includegraphics[width=0.18\linewidth]{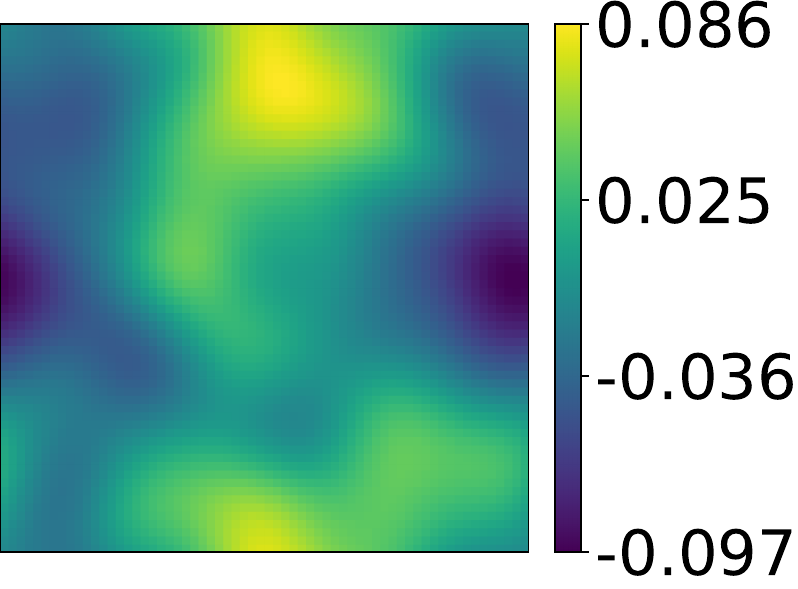}\hspace{2.5mm}
            \includegraphics[width=0.18\linewidth]{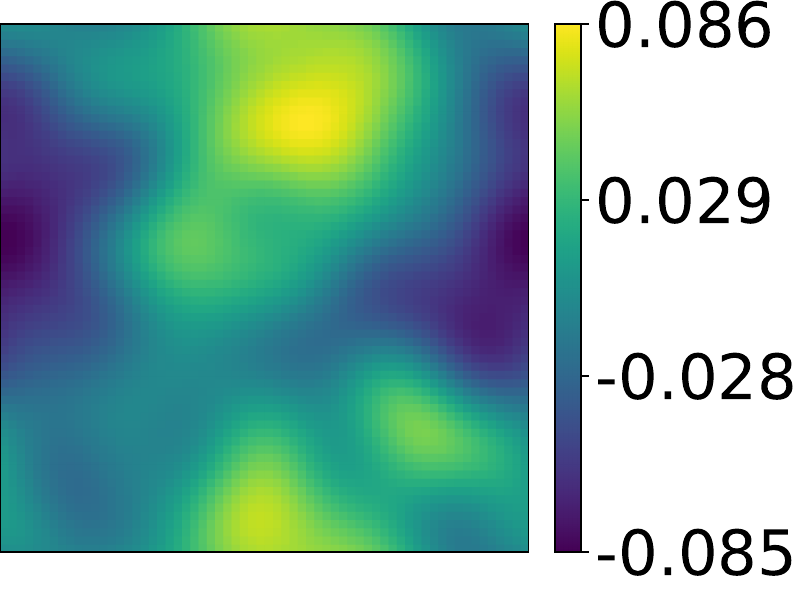}\hspace{2.5mm}
            \includegraphics[width=0.18\linewidth]{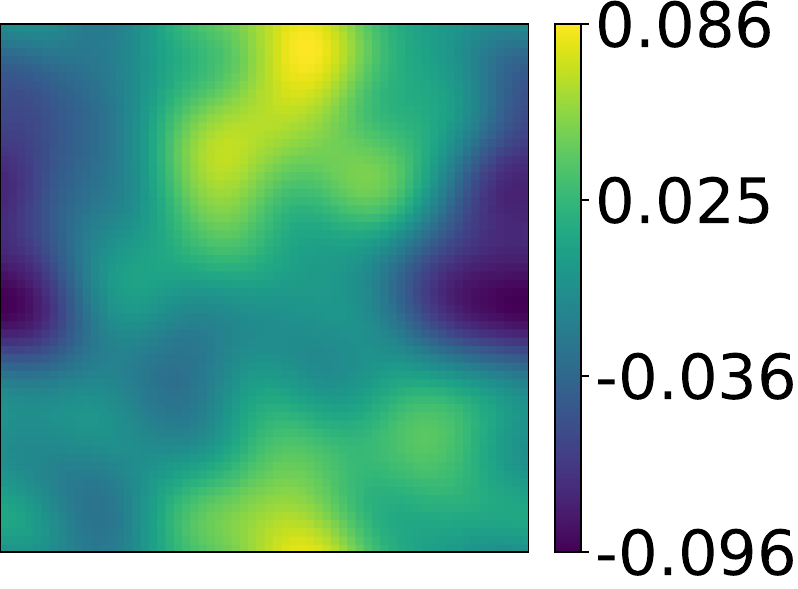}\hspace{2.5mm}
            \includegraphics[width=0.18\linewidth]{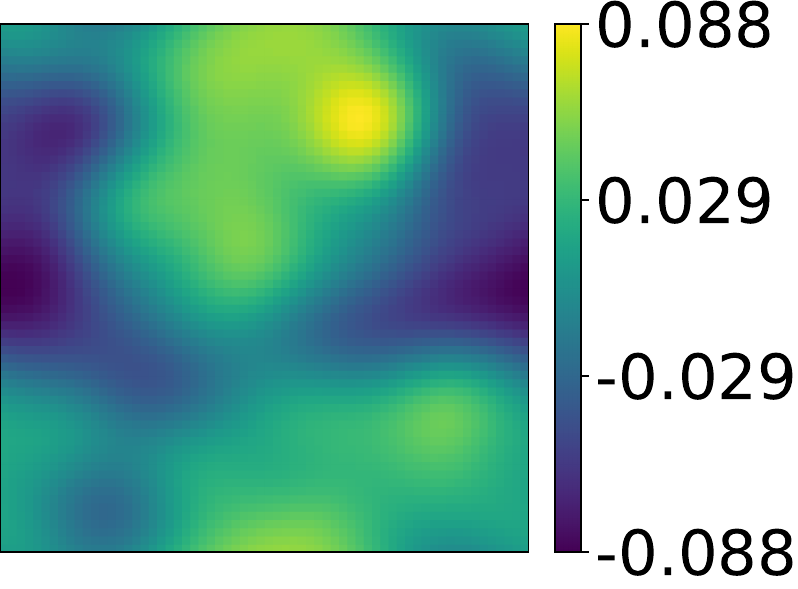}
        \label{fig:rd_samples}
    }

    \caption{Reaction-diffusion inverse problem. The first row shows the unknown field, sparse observations, reference posterior mean, predicted posterior mean, and mean error. The second row shows representative posterior samples generated by the proposed sampler.}
    \label{fig:reaction_diffusion_results}
\end{figure}
\begin{figure}[htbp!]
    \centering

    \subfigure[$u_0$]{
        \includegraphics[width=0.18\linewidth]{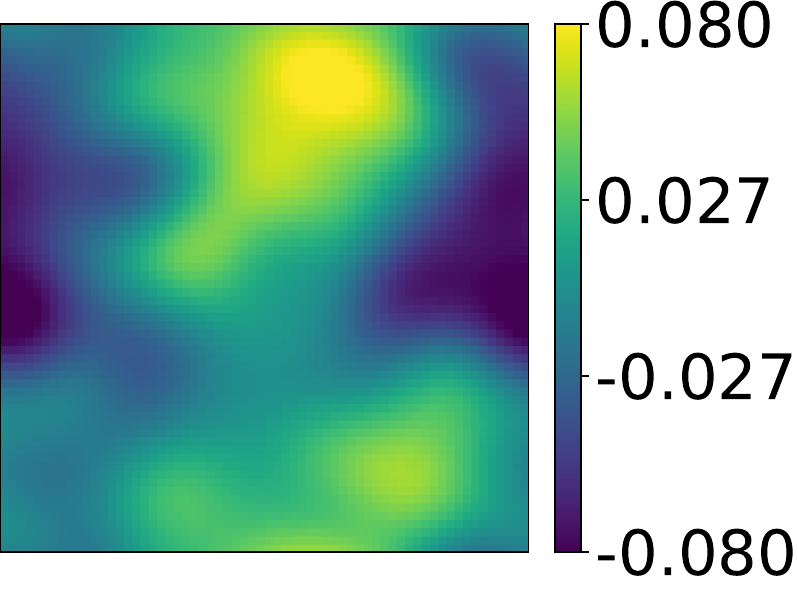}
        \label{fig:ns_x}
    }%
    \subfigure[sensors]{
        \includegraphics[width=0.18\linewidth]{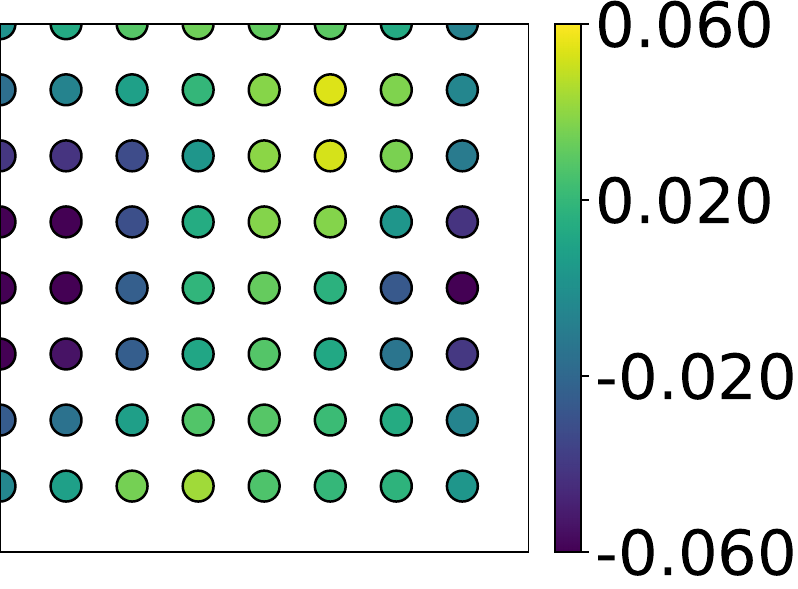}
        \label{fig:ns_Gobs_sensors}
    }%
    \subfigure[$\mu_{\mathrm{ref}}$]{
        \includegraphics[width=0.18\linewidth]{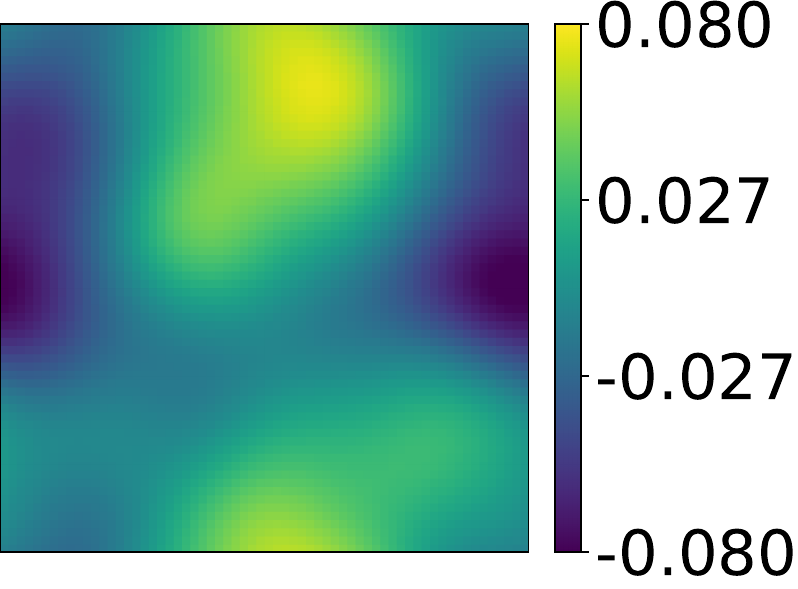}
        \label{fig:ns_mu_ref}
    }%
    \subfigure[$\mu_{\mathrm{pred}}$]{
        \includegraphics[width=0.18\linewidth]{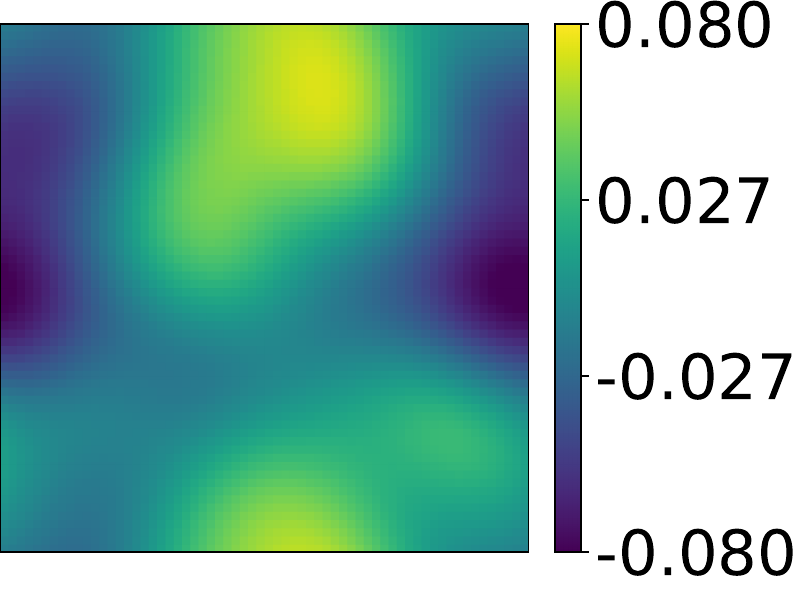}
        \label{fig:ns_mu_pred}
    }%
    \subfigure[$\mu_{\mathrm{pred}}-\mu_{\mathrm{ref}}$]{
        \includegraphics[width=0.18\linewidth]{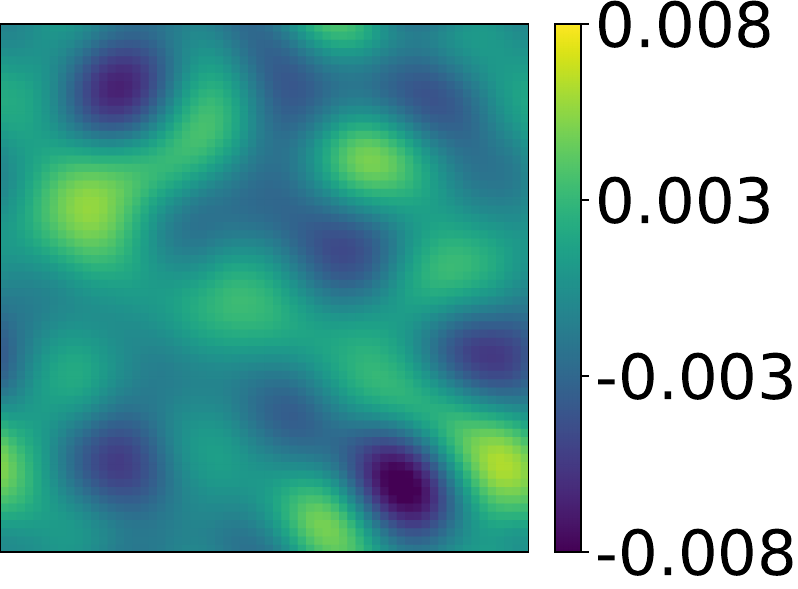}
        \label{fig:ns_mu_diff}
    }
    \subfigure[Posterior samples]{
            \centering
            \includegraphics[width=0.18\linewidth]{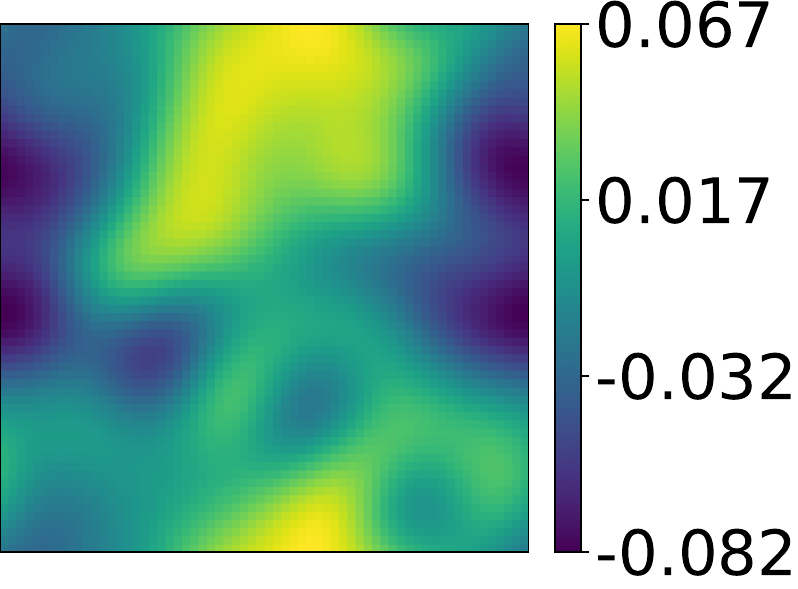}\hspace{2.5mm}
            \includegraphics[width=0.18\linewidth]{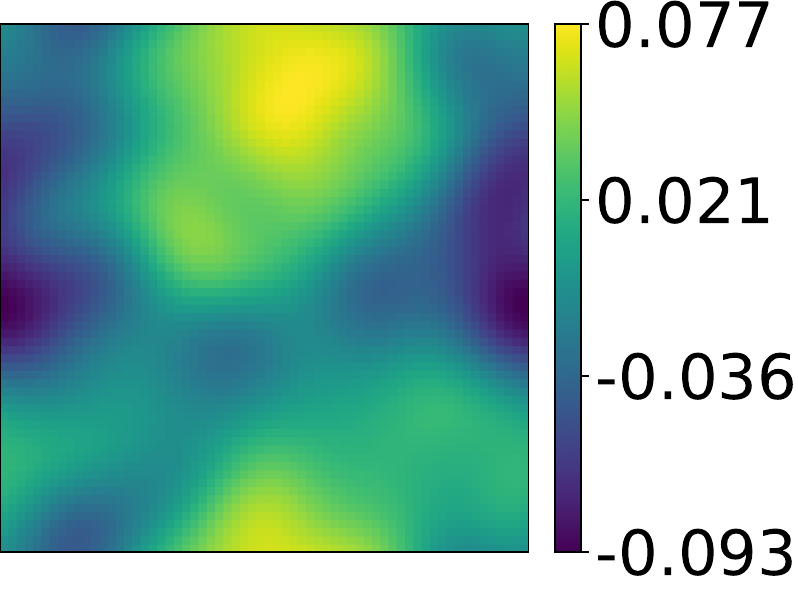}\hspace{2.5mm}
            \includegraphics[width=0.18\linewidth]{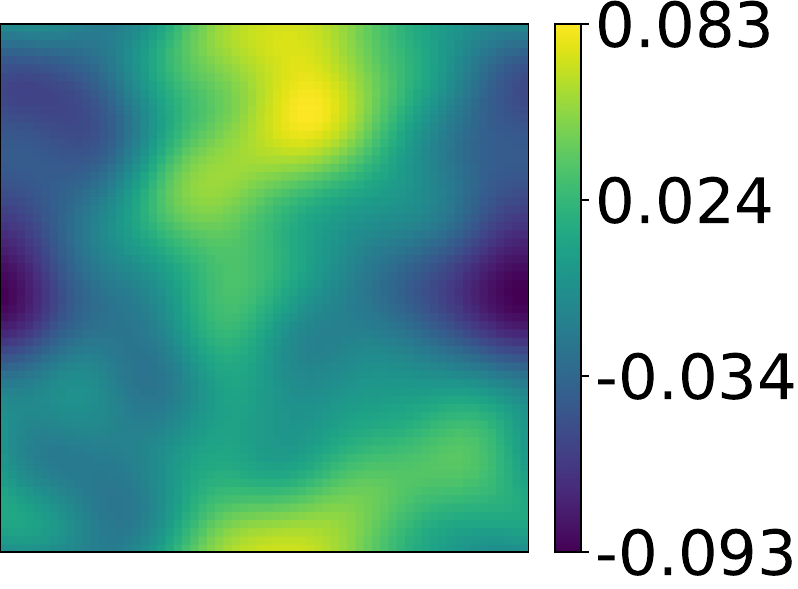}\hspace{2.5mm}
            \includegraphics[width=0.18\linewidth]{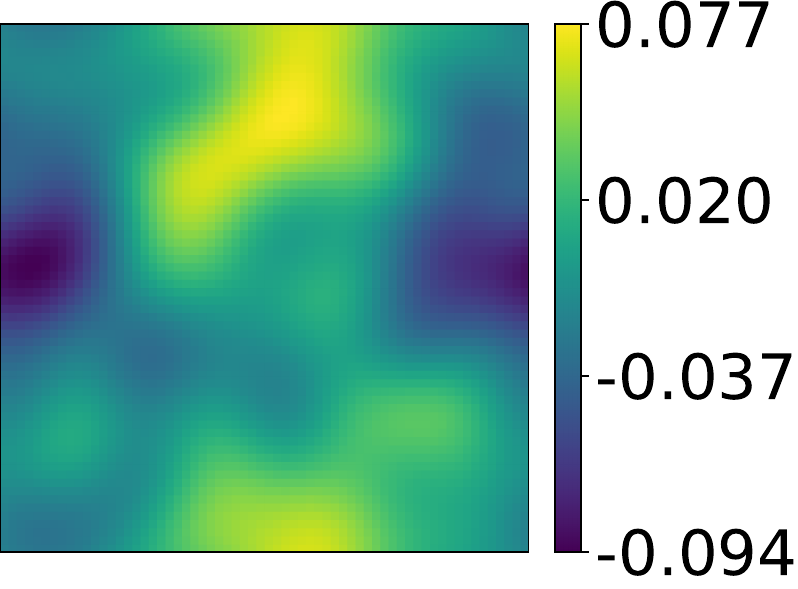}\hspace{2.5mm}
            \includegraphics[width=0.18\linewidth]{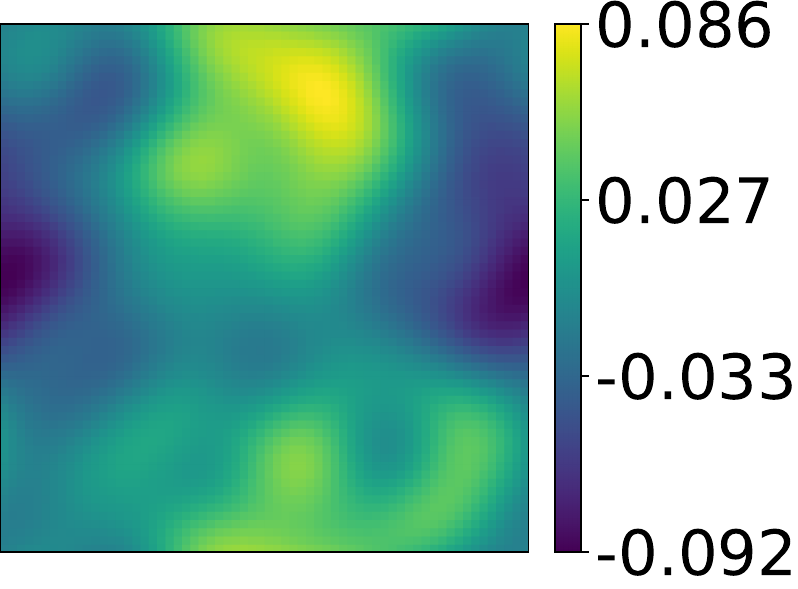}
        \label{fig:ns_samples}
    }

    \caption{Navier-Stokes inverse problem. The first row shows the unknown field, sparse observations, reference posterior mean, predicted posterior mean, and mean error. The second row shows representative posterior samples generated by the proposed sampler.}
    \label{fig:navier_stokes_results}
\end{figure}
 \begin{enumerate}[resume,leftmargin=*, label=(\arabic*),itemsep=0.2em]
\item \underline{Advection.}
In the 2D linear Advection equation,
\begin{equation}\label{eq:advection_pde}
\begin{cases}
\partial_\tau u + \mathbf b_\mathrm{adv}\cdot\nabla u = 0  &\text{in }\Omega\times(0,\infty),\\
u(\cdot,0)=u_0 &\text{in }\Omega,
\end{cases}
\end{equation}
with velocity $\mathbf b_\mathrm{adv}=(0.8,0.4)$, we consider the forward map $\mathcal G(u_0)=u(\cdot,T=1)$.
On the computational grid, we discretize \eqref{eq:advection_pde} using a first-order upwinding in space and forward Euler in time
scheme under the CFL constraint.
The observations are $y_{\mathrm{obs}}=\mathcal{O}_\mathrm{per}(u)+\eta$ with $\eta\sim\mathcal{N}(0,0.05^2I_m).$ See Figure~\ref{fig:advection_results}\subref{fig:advection_x}--\subref{fig:advection_Gobs_sensors}.
\item \underline{Reaction--Diffusion Equation (RDE).}
The 2D Allen-Cahn Reaction--Diffusion equation follows:
\begin{equation}
\begin{cases}
\partial_\tau u = D\Delta u + u - u^{3} & \text{in }\Omega\times(0,\infty),\\
u(\cdot,0)=u_0 &\text{in }\Omega,
\end{cases}
\label{eq:rd_pde}
\end{equation}
with diffusion coefficient $D=10^{-3}$. We consider the forward map $\mathcal G(u_0)=u(\cdot,T=2)$. We discretize \eqref{eq:rd_pde} on a uniform grid using the standard second-order central-difference Laplacian and forward Euler time stepping. Observations follow $y_{\mathrm{obs}}=\mathcal{O}_\mathrm{per}(u)+\eta$ where $\eta\sim\mathcal{N}(0,0.05^2I_m)$. See Figure~\ref{fig:reaction_diffusion_results}\subref{fig:rd_x}--\subref{fig:rd_Gobs_sensors}.
\item \underline{Navier--Stokes Equation(in vorticity form, NSE).}

In the 2D vorticity formulation
\begin{equation}
\begin{cases}
\partial_\tau u + \mathbf b_\mathrm{NS}\cdot\nabla u = \nu\,\Delta u + f,
& \text{in }\Omega\times(0,\infty),\\
-\Delta \psi = u,& \text{in }\Omega,\\
\mathbf b_\mathrm{NS}=\nabla^\perp\psi := (\partial_2\psi,\,-\partial_1\psi),
& \text{in }\Omega,\\
u(\cdot,0)=u_0,
& \text{in }\Omega,
\end{cases}
\label{eq:ns_vorticity}
\end{equation}
where $\nu=10^{-2}$ is the viscosity and $f(q_1,q_2)=
0.1[\sin(2\pi(q_1+q_2))+\cos(2\pi(q_1+q_2))]
$ is a fixed forcing. we consider the forward map $\mathcal{G}(u_0)=u(\cdot,T=0.2)$. We solve \eqref{eq:ns_vorticity} using a Fourier pseudo-spectral method with $2/3$ de-aliasing and an IMEX time integrator (Crank-Nicolson for diffusion and explicit advection).
Observations follow $y_{\mathrm{obs}}=\mathcal{O}_\mathrm{per}(u)+\eta$ where $\eta\sim\mathcal{N}(0,0.01^2I_m)$. See Figure~\ref{fig:navier_stokes_results}\subref{fig:ns_x}--\subref{fig:ns_Gobs_sensors}.
\end{enumerate}
\begin{table}[!htb]
\centering
\caption{Wall-clock generation time for producing posterior samples \emph{conditioned on a single fixed observation} $y_{\mathrm{obs}}$.
\emph{Hardware:} MCMC is run in parallel on $16$ cores of the CPU ($32$ thread at $3.6$GHz) cores; learned samplers are evaluated on a single NVIDIA RTX~A6000 GPU. The per-sample clock time is computed by averaging over the generation time of $1000$ posterior samples.}
\label{tab:sampling_time}
\setlength{\tabcolsep}{3pt}
\begin{tabular}{l c c c}
\toprule
Sampler & Steps   & Sampling time (s/sample) \\
\midrule
MCMC (Darcy)               & $3\times10^4$ &  $2.05\times 10^1$ \\
MCMC (Advection)           & $3\times10^4$ &  7.01 \\
MCMC (Reaction--Diffusion) & $3\times10^4$ &  $1.04\times 10^1$ \\
MCMC (Navier--Stokes)      & $3\times10^4$ & $1.97\times 10^1$ \\
\midrule
Multistep diffusion        & $320$            & $1.05\times10^{-1}$ \\
Multistep flow matching & $320$            & $1.67\times10^{-1}$\\
\midrule
One-step (MF) & $1$ &  $3.50\times10^{-4}$ \\
One-step (iMF)                         & $1$ & $3.51\times10^{-4}$ \\
One-step (PODNO)              & $1$ &  $3.17\times10^{-4}$ \\
One-step (U-Net)              & $1$           & $1.85\times10^{-3}$ \\
\bottomrule
\end{tabular}

\vspace{2pt}
\begin{minipage}{\linewidth}
{\textit{Note:} Generation of training data ($10000$) requires $432$s for Darcy, $75$s for Advection, $145$s for Reaction--Diffusion, $385$s for Navier--Stokes. The training and sampling costs of the machine
learning models are independent of the specific inverse problem and depend
only on the discretization resolution.}
\end{minipage}\vspace*{-0.5\baselineskip}
\end{table}

\begin{table}[!htb]
\centering
\caption{Posterior approximation errors of the conditional one-step sampler on PDE inverse problems.}
\label{tab:benchmark_summary}
\setlength{\tabcolsep}{6pt} 
\begin{tabular}{l l c c c }
\toprule
Dataset & Source & $\epsilon_\mu^\mathrm{eval}$ & $\epsilon_\sigma^\mathrm{eval}$\\
\midrule
\multirow{2}{*}{Darcy}
  & Anisotropic &   $5.55\times10^{-2}$ & $4.27\times 10^{-2}$\\
  & White   &  $8.26\times10^{-2}$ &$4.30\times10^{-2}$ \\
\midrule
\multirow{2}{*}{Advection}
  & Anisotropic &    $6.08\times10^{-2}$&  $5.09\times10^{-2}$ \\
  & White   &   $1.63\times10^{-1}$ & $5.14\times10^{-2}$ \\
\midrule
\multirow{2}{*}{\shortstack[l]{Reaction--\\Diffusion}}
  & Anisotropic &  $7.62\times10^{-2}$&$8.68\times10^{-2}$  \\
  & White  &  $7.89\times10^{-2}$&  $8.91\times10^{-2}$  \\
\midrule
\multirow{2}{*}{\shortstack[l]{Navier--\\Stokes}}
  & Anisotropic &  $6.41\times10^{-2}$ & $5.39\times10^{-2}$ \\
  & White   &  $8.16\times10^{-2}$ & $5.26\times10^{-2}$\\
\bottomrule
\end{tabular}
\begin{minipage}{\linewidth}
\end{minipage}

\end{table}

For the PDE inverse problems, the reference posterior is not available in closed form.
We therefore compute reference statistics as one in \eqref{eq:posterior_statistics} from MCMC samples in the truncated KL
coefficient space. The truncation is necessary since larger \(N_{\mathrm{KL}}\) substantially worsens MCMC mixing in coefficient space, making reliable reference posterior statistics practically infeasible. In contrast, our one-step sampler has no such KL-dimension restriction: for any chosen \(N_{\mathrm{KL}}\), each sample is generated by the same single network forward pass on the discretized field. However, for fair comparison, we still take the same \(N_{\mathrm{KL}}\) in our training and sampling.

Writing \(x=\Phi(c)\) with
\(c\in\mathbb R^{N_{\mathrm{KL}}}\), the prior becomes
\(c\sim\mathcal N(0,I_{N_{\mathrm{KL}}})\). We sample the corresponding
coefficient-space posterior using the affine-invariant ensemble sampler
implemented in \texttt{emcee}~\cite{goodman2010ensemble,foreman2013emcee}.
For all PDE problems, we use \(N_{\mathrm{KL}}=64\) and \(200\) walkers, discard
the first \(5{,}000\) steps as burn-in, and run approximately \(3\times10^4\)
steps in total, corresponding to about \(50\) integrated autocorrelation times
after burn-in. From the post-burn-in samples, we use
\(N_{\mathrm{eval}}=10^3\) field samples
\(x_{\mathrm{ref}}^{(i)}=\Phi(c^{(i)})\) and compute
\begin{equation}
\mu_{\mathrm{ref}}
=
\frac{1}{N_{\mathrm{eval}}}
\sum_{i=1}^{N_{\mathrm{eval}}}
x_{\mathrm{ref}}^{(i)},
\quad
\sigma_{\mathrm{ref}}
=
\sqrt{
\frac{1}{N_{\mathrm{eval}}-1}
\sum_{i=1}^{N_{\mathrm{eval}}}
(x_{\mathrm{ref}}^{(i)}-\mu_{\mathrm{ref}})^2
}.
\label{eq:mcmc_ref_stats}
\end{equation}
Then we still use \eqref{eq_std} and \eqref{eq:posterior_statistics} to quantify the posterior quality.
Figures~\ref{fig:darcy_results}\subref{fig:darcy_mu_ref}--\subref{fig:darcy_samples}, \ref{fig:advection_results}\subref{fig:advection_mu_ref}--\subref{fig:advection_samples}, \ref{fig:reaction_diffusion_results}\subref{fig:rd_mu_ref}--\subref{fig:rd_samples}, and \ref{fig:navier_stokes_results}\subref{fig:ns_mu_ref}--\subref{fig:ns_samples} summarize the posterior quality generated by our model with prior-aligned source. Table~\ref{tab:benchmark_summary} reports the results for both source choices. 

Overall, under our experimental setup with a prior-aligned anisotropic source, we achieve $\epsilon_\mu^{\mathrm{eval}}<10\%$ and $\epsilon_\sigma^{\mathrm{eval}}<10\%$ across all PDE inverse problems considered (Darcy, Advection, Reaction--Diffusion, and Navier--Stokes), indicating accurate recovery of both the posterior mean and posterior uncertainty relative to the reference posterior. 

With the white source, \(\epsilon_\mu^{\mathrm{eval}}\) and
\(\epsilon_\sigma^{\mathrm{eval}}\) are still worse than those of the
prior-aligned anisotropic source, although the gap is moderate at the tested resolutions. This is partly because the MCMC reference statistics are computed with the truncated KL expansion \(N_{\mathrm{KL}}=64\), which suppresses part of the high-frequency mismatch. 
In addition, the reported white-source results also use the source-dependent encoding in \eqref{eq:normalization}, where the white-source variables are standardized componentwise. This rescaling improves the conditioning of the
finite-dimensional coordinates, but it does not restore the prior covariance structure and therefore cannot remove the function-space
inconsistency. Indeed, when the same mean-only encoding as the anisotropic source is used, the performance of the white-source model degrades
substantially.

Moreover, under the improved Mean-Flow \cite{geng2025improved} objective, the white-source transport becomes severely unstable and can even diverge, whereas the prior-aligned anisotropic source transport remains stable; see Table~\ref{tab:improved_mean_flow}. Together with the discussion above, this suggests that the mild finite-resolution performance of the white source here is not intrinsic, but is partly masked by the
preprocessing encoding, and the finite KL truncation.

Additional ablations with alternative operator-learning parametrizations (PODNO; Table~\ref{tab:podno}) and a finite-dimensional U-Net parametrizations (Table~\ref{tab:unet}), together with comparisons to multistep diffusion models (Table~\ref{tab:diffusion}), and DPS with an unconditional prior model(Figures~\ref{fig:DPS_identity_results} and~\ref{fig:DPS_darcy_results}), are reported in Appendix~\ref{sec:alternative_models}--\ref{sec:prior_guidance_comparison}. These results provide further evidence that the proposed preconditioned one-step conditional neural operator transport is both computationally efficient and well-suited to
function-space inverse problems.

\section{Conclusion}
\label{sec:conclusion}

In this work, we presented a one-step amortized posterior sampler for Bayesian inverse problems in function space.  The method learns an observation-conditional transport map from joint simulation pairs \((x,y_{\mathrm{obs}})\), where \(x\) is drawn from the prior and \(y_{\mathrm{obs}}\) is obtained from the forward model with partial and noisy observations.  After training, the same neural operator sampler can be applied to unseen observations and generate
posterior samples in a single forward pass, without MCMC supervision, teacher models, or iterative SDE/ODE sampling.

A central message of this work is that the source measure is not merely a finite-dimensional numerical choice, but must be consistent with the function-space geometry of the posterior.  We therefore use a prior-aligned anisotropic Gaussian source and show theoretically, under suitable assumptions, that the induced conditional transport enjoys Lipschitz regularity.  The numerical results support this principle: prior-aligned sources yield more stable posterior statistics, whereas mismatched Gaussian sources, such as white noise, remain less robust.

Across the PDE inverse problems considered in this paper, the proposed one-step sampler accurately captures key posterior summaries while reducing
the sample-generation cost by orders of magnitude compared with MCMC. It also avoids the repeated numerical integration required by multistep diffusion or flow samplers, thereby removing a source of discretization error and error
accumulation. These results suggest that one-step conditional transport, combined with prior-aligned source measures and neural-operator parameterizations, provides a practical approach to fast amortized Bayesian inference for function-space PDE inverse problems.
\section*{Impact statement}
This paper studies machine-learning methods for PDE-based inverse problems in scientific computing.
While such methods may impact a broad range of downstream applications, we do not foresee
immediate societal risks that require specific discussion beyond standard considerations on
data, robustness, and responsible deployment.
\section*{Acknowledgements}
This work was supported by Singapore MOE AcRF Tier 1 Grant RG95/23 and RG17/24, NTU SUG-023162-00001, and Singapore MOE AcRF Tier 2 Grant: MOE-T2EP20224-0012.
\section*{Appendix}
\appendix
In this appendix, we provide additional details and comparisons supporting the
main text. Appendix~\ref{sec:alternative_models} reports additional model
variants and baselines. Appendix~\ref{sec:prior_guidance_comparison}
compares the proposed conditional sampler with an unconditional prior model combined with guidance. Appendix~\ref{app:notation_settings} summarizes the
notation and the network/training settings used.
\section{Additional model variants and baselines}\label{sec:alternative_models}
In this section, we report additional ablations for the proposed conditional
one-step sampler. Unless otherwise stated, all variants use the same FNO parameterization and setting as the main model. We first examine
alternative parameterizations, including PODNO and U-Net. We then
evaluate Improved Mean Flows~\cite{geng2025improved} as an additional variant. Finally, we include multistep diffusion baselines for further comparison.

\subsection{PODNO parameterization}
\label{app:fno_podno}
\begin{table}[htbp!]
\centering
\caption{Posterior approximation errors of the PODNO-parameterized Mean-Flow sampler.}
\label{tab:podno}
\begin{tabular}{l l c  c}
\toprule
Dataset & Source  & $\epsilon_\mu^\mathrm{eval}$ & $\epsilon_\sigma^\mathrm{eval}$ \\
\midrule
\multirow{2}{*}{Darcy}
  & Anisotropic & $5.55\times 10^{-2}$ & $5.45\times 10^{-2}$ \\
  & White    &$7.56\times 10^{-2}$&$4.37\times 10^{-2}$\\
\midrule
\multirow{2}{*}{Advection}
  & Anisotropic & $6.36\times 10^{-2}$ & $3.76\times 10^{-2}$ \\
  & White   & $1.43\times 10^{-1}$ & $5.08\times 10^{-2}$ \\
\midrule
\multirow{2}{*}{Reaction--Diffusion}
  & Anisotropic & $7.23\times 10^{-2}$ & $7.94\times 10^{-2}$ \\
  & White   & $8.08\times 10^{-2}$ & $7.85\times 10^{-2}$ \\
\midrule
\multirow{2}{*}{Navier--Stokes}
  & Anisotropic & $5.91\times 10^{-2}$ & $5.50\times 10^{-2}$ \\
  & White &  $8.08\times 10^{-2}$ & $5.52\times 10^{-2}$ \\
\bottomrule
\end{tabular}
\end{table}

The proposed one-step Mean Flows framework is agnostic to the choice of parameterizations.
Beyond FNO, we also evaluated Proper Orthogonal Decomposition Neural Operators (PODNO) \cite{cheng2026podno} under the same Mean Flows formulation and with matched parameter budgets: we use $288$ POD modes for PODNO and $12$ Fourier modes for FNO, yielding $1.2\times 10^{8}$ parameters in both cases.

Across all benchmarks in Table~\ref{tab:podno}, PODNO achieves posterior accuracy comparable to FNO while
offering significant computational advantages.
By projecting the dynamics onto a low-dimensional POD basis, PODNO reduces the cost
of online spectral operations and leads to approximately a speedup in training, without degrading posterior mean or uncertainty estimates.

These results indicate that the proposed framework is compatible with a range of
operator-learning architectures.
Designing neural operators that are both computationally efficient and
well-conditioned for probabilistic transport remains an important direction for
future work.
\subsection{UNet parameterization}
\label{app:unet_limitations}
\begin{table}[htbp!]
\centering
\caption{Posterior approximation errors of the U-Net-parameterized Mean-Flow sampler.}
\label{tab:unet}
\begin{tabular}{l l c c }
\toprule
Dataset & Source & $\epsilon_\mu^\mathrm{eval}$ & $\epsilon_\sigma^\mathrm{eval}$ \\
\midrule
\multirow{2}{*}{Darcy}
  & Anisotropic & $1.02\times 10^{-1}$ & $1.48\times 10^{-1}$ \\
  & White   &$1.41\times 10^{-1}$  &$1.02\times 10^{-1}$  \\
\midrule
\multirow{2}{*}{Advection}
  & Anisotropic &  $7.34\times 10^{-2}$&  $1.37\times 10^{-1}$\\
  & White   & $2.04\times 10^{-1}$ & $1.04\times 10^{-1}$ \\
\midrule
\multirow{2}{*}{Reaction--Diffusion}
  & Anisotropic &  $9.17\times 10^{-2}$&  $2.04\times 10^{-1}$\\
  & White   & $1.44\times 10^{-1}$ & $1.44\times 10^{-1}$\\
\midrule
\multirow{2}{*}{Navier--Stokes}
  & Anisotropic &  $1.06\times 10^{-1}$&  $1.31\times 10^{-1}$\\
  & White   &  $1.29\times 10^{-1}$&  $9.98\times 10^{-2}$\\
\bottomrule
\end{tabular}
\end{table}
We also tested UNet \cite{ronneberger2015u} parameterization within the same one-step Mean Flows framework.
While UNets are widely used in diffusion models for image generation, the primary limitation is structural.
It is a grid-dependent convolutional architecture, designed to exploit local
spatial correlations at a fixed resolution.
In contrast, the transport map in Bayesian PDE inverse problems is inherently an
operator acting on function spaces, with global and nonlocal dependence induced by
the forward PDE and the observation operator.

As a consequence, UNets require substantially more parameters (in our setting, $1.7\times 10^9$ for FNO compared to $1.2\times 10^8$ for UNet), yet still deliver inferior performance and struggle to represent resolution-independent transport maps as presented in Table~\ref{tab:unet}.
\subsection{Improved Mean Flows model}
\label{app:imf_white_noise}

\begin{algorithm}[!htb]
  \caption{Training the Improved Conditional Mean-Flow Predictor}
  \label{alg:imf_training_encode_only}
  \begin{algorithmic}
    \STATE {\bfseries Input:} Training pairs
    $\{(x^{(i)}, y_{\mathrm{obs}}^{(i)})\}_{i=1}^{N_{\mathrm{train}}}$sampled from the joint law of
    \((X,Y_{\mathrm{obs}})\);
    source measure $\rho$.
    \STATE {\bfseries Preprocessing:}
    \STATE \hspace{0.5em}
    Estimate encoding statistics from the training set and define
    $\tilde x=\mathcal E_x(x)$ and
    $\tilde y_{\mathrm{obs}}=\mathcal E_y(y_{\mathrm{obs}})$ as in
    \eqref{eq:normalization}.
    \STATE {\bfseries Initialize:} network parameters $\theta$.
    \REPEAT
      \STATE Sample a mini-batch $(x,y_{\mathrm{obs}})$ from the training set.
      \STATE Encode
      $\tilde x=\mathcal E_x(x)$ and
      $\tilde y_{\mathrm{obs}}=\mathcal E_y(y_{\mathrm{obs}})$.
      \STATE Sample an independent source variable $\xi\sim\rho$.
      \STATE Sample a time pair $(r,t)$ with $0\le r\le t\le 1$.
      \STATE Form the encoded interpolation path
          $z_t=(1-t)\tilde x+t\xi,
          ~
          v^{\mathrm{path}}=\xi-\tilde x$ .
      \STATE Compute the boundary-trick surrogate direction
          $v_\theta^{\mathrm{bd}}
          :=
          w_\theta(z_t,t,t;\tilde y_{\mathrm{obs}})$.
      \STATE Form the iMF composite predictor
          $\widehat v_\theta^{\mathrm{iMF}}
          :=
          w_\theta(z_t,r,t;\tilde y_{\mathrm{obs}})
          +(t-r)\,
          \mathrm{sg}\!\left(
          \mathrm{JVP}_{(z,r,t)}
          \bigl(w_\theta;(v_\theta^{\mathrm{bd}},0,1)\bigr)
          \right)$.
      \STATE Update $\theta$ by minimizing the encoded-space MSE loss
          $\|
          \widehat v_\theta^{\mathrm{iMF}}-v^{\mathrm{path}}
          \|^2$ .
    \UNTIL{convergence}
    \STATE {\bfseries Output:} trained parameters $\hat\theta$.
  \end{algorithmic}
\end{algorithm}

\begin{algorithm}[!htb]
  \caption{One-Step Conditional Posterior Sampling with iMF}
  \label{alg:imf_sampling_encode_only}
  \begin{algorithmic}
    \STATE {\bfseries Input:} Observation $y_{\mathrm{obs}}$;
    trained parameters $\hat\theta$;
    source measure $\rho$; encoders $\mathcal E_x,\mathcal E_y$.
    \STATE Encode the observation:
        $\tilde y_{\mathrm{obs}}
        =
        \mathcal E_y(y_{\mathrm{obs}})$.
    \STATE Sample $\xi\sim\rho$.
    \STATE Generate an encoded posterior sample by the one-step map:
        $\tilde x
        =
        \xi
        -
        w_{\hat\theta}(\xi,0,1;\tilde y_{\mathrm{obs}})$.
    \STATE Decode to physical space:
        $x
        =
        \mathcal E_x^{-1}(\tilde x)$.
    \STATE {\bfseries Output:} posterior sample $x$.
  \end{algorithmic}
\end{algorithm}

\begin{table}[!htb]
\centering
\caption{Posterior approximation errors of the improved Mean-Flow sampler.}
\label{tab:improved_mean_flow}
\begin{tabular}{l l c c  }
\toprule
Dataset & Source & $\epsilon_\mu^\mathrm{eval}$ & $\epsilon_\sigma^\mathrm{eval}$\\
\midrule
\multirow{2}{*}{Darcy}
  & Anisotropic  & $5.43\times10^{-2}$& $6.16\times10^{-2}$\\
  & White    & $2.80\times 10^9$& $3.14\times 10^9$\\
\midrule
\multirow{2}{*}{Advection}
  & Anisotropic & $7.19\times10^{-2}$& $5.05\times10^{-2}$\\
  & White   & $1.11\times 10^{10}$&$1.91\times 10^{5}$ \\
\midrule
\multirow{2}{*}{Reaction-Diffusion}
  & Anisotropic  & $8.23\times10^{-2}$& $9.64\times10^{-2}$\\
  & White   & $2.16\times 10^{10}$&$6.14\times 10^{5}$ \\
\midrule
\multirow{2}{*}{Navier--Stokes}
  & Anisotropic & $6.86\times 10^{-2}$& $6.52\times 10^{-2}$\\
  & White   &$1.71\times 10^{10}$ &$3.50\times 10^5$\\
\bottomrule
\end{tabular}
\end{table}

We additionally test Improved Mean Flows (iMF)~\cite{geng2025improved}, which augments Mean Flows training with an explicit total-derivative (JVP) correction in a training-time composite predictor. 

We use the same simulator pair,
encoded variables, source variable, and linear interpolation path as in
Algorithm~\ref{alg:mf_training_encode_only}--\ref{alg:mf_sampling_encode_only}.
However, we use the boundary-trick surrogate direction
$V_\theta^{\mathrm{bd}}
:=
w_\theta(Z_t,t,t;Y_{\mathrm{obs}}).$
The iMF composite predictor is then defined as the random variable
\begin{equation}
\label{eq:imf_predictor}
    \widehat V_\theta^{\mathrm{iMF}}
    :=
    w_\theta(Z_t,r,t;Y_{\mathrm{obs}})
    +(t-r)\,\mathrm{sg}\!\left(
    \mathrm{JVP}_{(z,r,t)}
    \bigl(w_\theta;(V_\theta^{\mathrm{bd}},0,1)\bigr)
    \right),
\end{equation}
where the JVP is evaluated at
\((Z_t,r,t)\). Equivalently,
\begin{equation}
\mathrm{JVP}_{(z,r,t)}
\bigl(w_\theta;(V_\theta^{\mathrm{bd}},0,1)\bigr)
:=
D_z w_\theta(Z_t,r,t; Y_{\mathrm{obs}})
    [V_\theta^{\mathrm{bd}}]
+
\partial_t w_\theta(Z_t,r,t; Y_{\mathrm{obs}}).
\label{eq:imf_jvp_def}
\end{equation}
Here \(D_z w_\theta(\cdot)[V_\theta^{\mathrm{bd}}]\) denotes the
Jacobian--vector product with respect to the state variable. The composite
quantity \(\widehat V_\theta^{\mathrm{iMF}}\) is not an independent network; it
is constructed from \(w_\theta\) only during training.

Unlike the original Mean-Flow objective, which regresses
\(w_\theta(Z_t,r,t; Y_{\mathrm{obs}})\) to the stop-gradient target
\(W_{\mathrm{tgt}}\), iMF regresses the corrected predictor
\(\widehat V_\theta^{\mathrm{iMF}}\) directly to the path velocity
\(V^{\mathrm{path}}\):
$\mathcal L_{\mathrm{iMF}}(\theta)
:=
\mathbb E
[
\|
\widehat V_\theta^{\mathrm{iMF}}
-
V^{\mathrm{path}}
\|^2
].
\label{eq:imf_loss}$
The expectation is taken over simulator pairs
\((X,Y_{\mathrm{obs}})\), independent source draws \(\Xi\sim\rho\), and time
pairs \((r,t)\). Algorithm \ref{alg:imf_training_encode_only}-\ref{alg:imf_sampling_encode_only} summarizes this variant with encoding \eqref{eq:normalization}.

Table~\ref{tab:improved_mean_flow} shows a consistent pattern across the PDE inverse problems.  With a prior-aligned anisotropic source, iMF trains stably and yields accurate posterior means and standard deviations.  With a
white-noise source, however, training can become numerically unstable, as indicated by loss blow-up or divergent values, and the resulting posterior statistics deteriorate substantially.

This difference can be understood from the source dependence of the iMF
composite predictor in~\eqref{eq:imf_predictor}. The additional correction is a JVP term evaluated
along the model-dependent boundary direction \(V_\theta^{\mathrm{bd}}\).  By \eqref{eq:imf_jvp_def}, this JVP contains the state-directional derivative $D_z w_\theta(Z_t,r,$ $t; Y_{\mathrm{obs}})[V_\theta^{\mathrm{bd}}]$. Under a white-noise source, resolution-sensitive high-frequency components can be amplified through this derivative, leading to an unstable correction in
the loss.  A prior-aligned anisotropic source suppresses these directions and makes the JVP correction better conditioned.

\subsection{Multistep diffusion model}\label{app:multistep}
We also test the PDE inverse problems using a vanilla multistep diffusion
sampler augmented with the same observation-conditioning mechanism as our proposed one-step model. For each fixed observation \(y_{\mathrm{obs}}\), the target
distribution is the posterior \(\pi(\cdot\mid y_{\mathrm{obs}})\). A score-based diffusion model first defines a forward noising process
that gradually transforms \(X|y_\mathrm{obs}\) into an approximately
Gaussian random variable. For example, the Ornstein--Uhlenbeck noising process
can be written as
\begin{equation}
    dZ_t
    =
    -\tfrac12 Z_t\,dt
    +\sqrt{C}dW_t,
    \quad
    Z_0=X|y_\mathrm{obs}.
    \label{eq:ou_forward_sde}
\end{equation}
Its marginal distribution admits the explicit representation
$    Z_t
    =
    e^{-t/2}X|y_\mathrm{obs}
    +
    \sqrt{1-e^{-t}}\,\Xi,
    ~
    \Xi\sim\mathcal N(0,C).$
Thus, as \(t\to\infty\), the law of \(Z_t\) approaches the Gaussian
distribution. In practice, the process is truncated at a large finite terminal
time.

After normalizing the diffusion time to \(t\in[0,1]\), we use the finite-time
Ornstein--Uhlenbeck schedule
\begin{equation}
    Z_t=\alpha_t X|y_\mathrm{obs
    }+\beta_t \Xi,
    \quad
    \alpha_t=e^{-T_dt/2},
    \quad
    \beta_t=\sqrt{1-e^{-T_dt}},
    \quad t\in[0,1],
    \label{eq:alpha_beta_schedule}
\end{equation}
where \(T_d>0\) is the terminal diffusion time used in the experiments. This
schedule satisfies \(\alpha_0=1\), \(\beta_0=0\), and, for sufficiently large
\(T_d\), \(\alpha_1\approx0\), \(\beta_1\approx1\).

Let \(p_t(\cdot\mid y_{\mathrm{obs}})\) denote the conditional law of \(Z_t\)
given \(Y_{\mathrm{obs}}=y_{\mathrm{obs}}\).  The reverse-time dynamics of the
Ornstein--Uhlenbeck process is driven by the conditional score
\(C\nabla_z\log p_t(z\mid y_{\mathrm{obs}})\). In the original diffusion time, it takes the form
\begin{equation}
    d\widetilde Z_t
    =
    \big(
        \tfrac12 \widetilde Z_t
        +
        C\nabla \log p_{T_d-t}
        (\widetilde Z_t|y_\mathrm{obs})
    \big)dt
    +
    \sqrt{C}d\widetilde W_t,
    \quad
    \widetilde Z_0\sim \pi(\cdot|y_\mathrm{obs}).
    \label{eq:reverse_sde_ou}
\end{equation}
For the Gaussian noising schedule \eqref{eq:alpha_beta_schedule}, Tweedie's
formula gives
\begin{equation}
    C\nabla\log p_t(z)
    =
    \mathbb E\big[
        \tfrac{\alpha_t X-z}{\beta_t^2}
        \,\big|\,
        Z_t=z,Y_\mathrm{obs}=y_\mathrm{obs}
    \big].
    \label{eq:tweedie_score_ou}
\end{equation}
Accordingly, a conditional score network \(s_\theta\) can be trained through the denoising
score-matching objective \cite{lu2025mathematical}
\begin{equation}
    \mathcal L_{\mathrm{score}}(\theta)
    :=
    \mathbb E\big[
        \big\|
        s_\theta(Z_t,t;Y_\mathrm{obs})
        -
        \tfrac{\alpha_t X-Z_t}
        {\beta_t^2}
        \big\|^2
    \big],
    \label{eq:score_matching_objective}
\end{equation}
whose population minimizer is the conditional score in \eqref{eq:tweedie_score_ou}. New
samples are then generated by solving the reverse dynamics
\eqref{eq:reverse_sde_ou}, which requires many time steps and repeated
evaluations of the learned score.
\begin{table}[htbp!]
\centering
\caption{Posterior approximation errors of the multistep diffusion baseline.}
\label{tab:diffusion}
\begin{tabular}{l l  c c}
\toprule
Dataset & Source &  $\epsilon_\mu^\mathrm{eval}$ & $\epsilon_\sigma^\mathrm{eval}$ \\
\midrule
\multirow{2}{*}{Darcy}
  & Anisotropic & $1.08\times10^{-1}$& $7.00\times10^{-2}$ \\
  & White    & $1.19\times10^{-1}$& $7.71\times10^{-2}$  \\
\midrule
\multirow{2}{*}{Advection}
  & Anisotropic  &$9.90\times10^{-2}$ & $4.30\times10^{-2}$\\
  & White    &$2.99\times10^{-1}$ & $8.24\times10^{-2}$\\
\midrule
\multirow{2}{*}{Reaction-Diffusion}
  & Anisotropic  & $1.08\times10^{-1}$& $1.02\times10^{-1}$\\
  & White    & $1.16\times10^{-1}$& $9.21\times10^{-2}$\\
\midrule
\multirow{2}{*}{Navier--Stokes}
  & Anisotropic  & $9.97\times10^{-2}$& $8.08\times10^{-2}$\\
  & White    & $1.09\times10^{-1}$& $7.63\times10^{-2}$\\
\bottomrule
\end{tabular}
\end{table}

\begin{figure}[h]
    \centering
    \subfigure[$\epsilon_\mu^\mathrm{eval}$]{
        \includegraphics[width=0.3\linewidth]{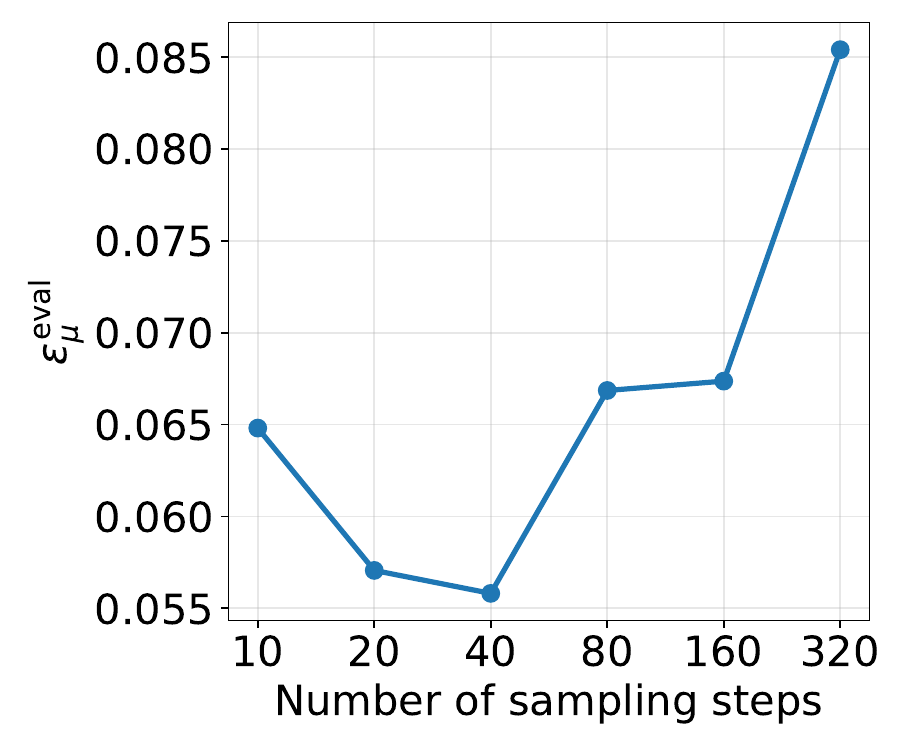}
        \label{fig:error_diffusion_mean}
    }
    \subfigure[$\epsilon_\sigma^\mathrm{eval}$]{
        \includegraphics[width=0.3\linewidth]{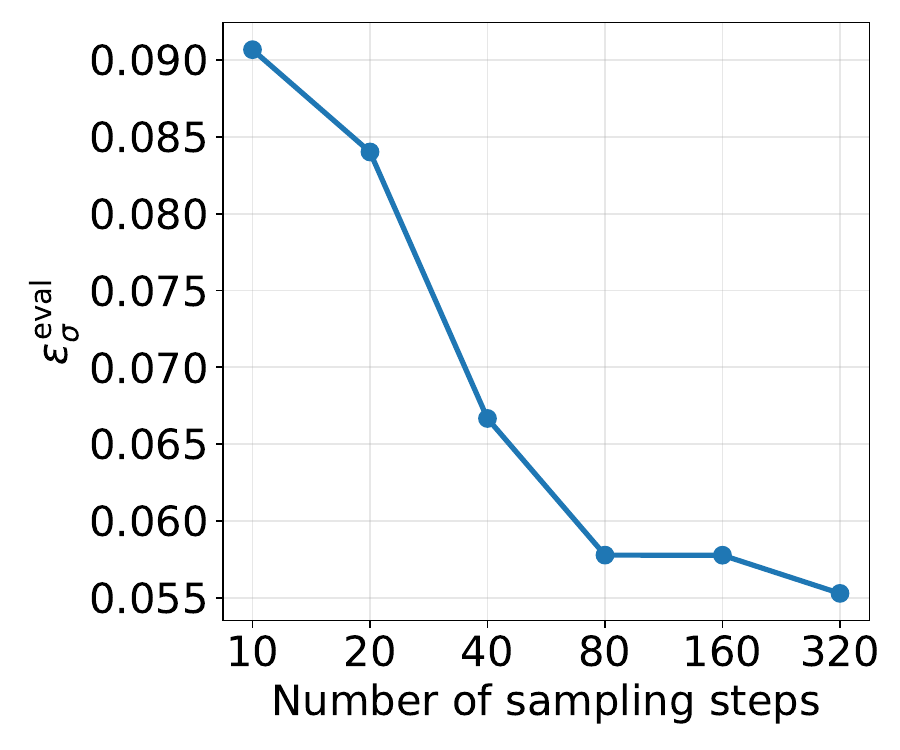}
        \label{fig:error_diffusion_std}
    }%
    \caption{Effect of the number of test-time sampling steps on posterior accuracy for the linear inverse problem for the multistep diffusion model. 
    We use the same model trained with 320 sampling steps and vary only the number of 
    sampling steps at inference time.}
    \label{fig:multistep_posterior_quality}
\end{figure}

As shown in Table~\ref{tab:diffusion}, with \(T_d=10\) and \(320\) sampling steps (best setting), the multistep sampler does not consistently improve the posterior approximation compared to our one-step sampler; it often gives larger mean and standard deviation errors than the one-step model. This is likely due to multistep SDE samplers introducing an additional trajectory-wise error accumulation mechanism: each sampling step relies on an imperfectly learned score and a finite time discretization, and the resulting local errors can accumulate along the numerical trajectory. 

This interpretation is supported by Figure~\ref{fig:multistep_posterior_quality},
where we keep the trained multistep diffusion model fixed and vary only the
number of sampling steps. As the number of sampling steps increases, \(\epsilon_\mu^\mathrm{eval}\) even becomes larger.


\section{Comparison with the prior-only model with guidance}\label{sec:prior_guidance_comparison}
\begin{figure}[!htb]
    \centering

    \subfigure[$\mu_\mathrm{ref}$]{
        \includegraphics[width=0.18\linewidth]{identity65_ref_mean.pdf}
        \label{fig:identity65_ref_mean}
    }%
    \subfigure[$\mu_\mathrm{ours}$]{
        \includegraphics[width=0.18\linewidth]{identity65_sol_mean.pdf}
        \label{fig:identity65_sol_mean}
    }%
    \subfigure[$\mu_\mathrm{DPS}$]{
        \includegraphics[width=0.18\linewidth]{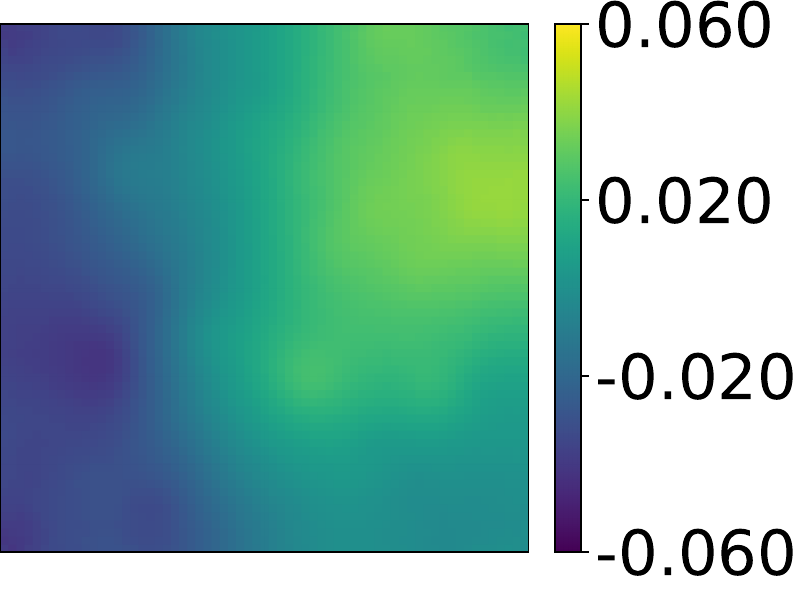}
        \label{fig:identity65_dps_mean_zeta1e-1}
    }%
    \subfigure[$\mu_{\mathrm{ours}}-\mu_{\mathrm{ref}}$]{
        \includegraphics[width=0.18\linewidth]{identity65_diff_mean.pdf}
        \label{fig:identity65_diff_mean}
    }%
    \subfigure[$\mu_{\mathrm{DPS}}-\mu_{\mathrm{ref}}$]{
        \includegraphics[width=0.18\linewidth]{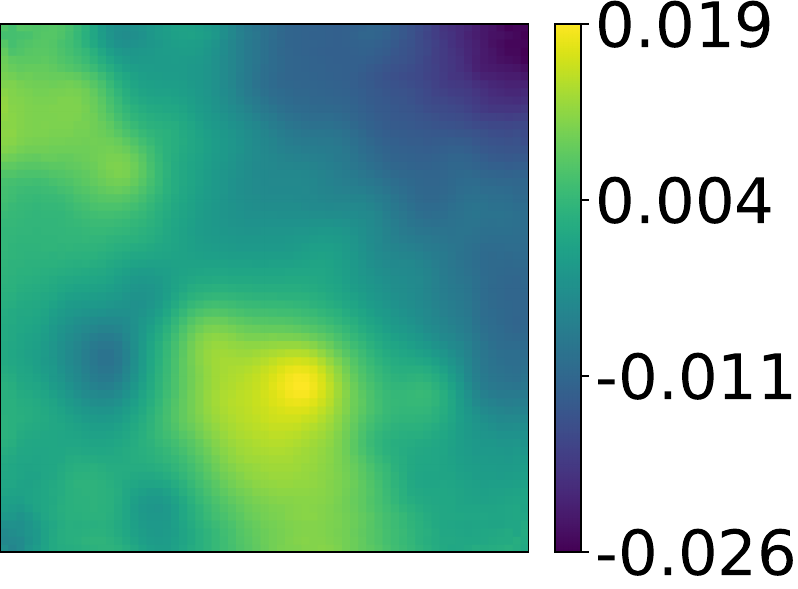}
        \label{fig:identity65_diff_mean_dps_zeta1e-1}
    }
    \subfigure[$\sigma_\mathrm{ref}$]{
        \includegraphics[width=0.18\linewidth]{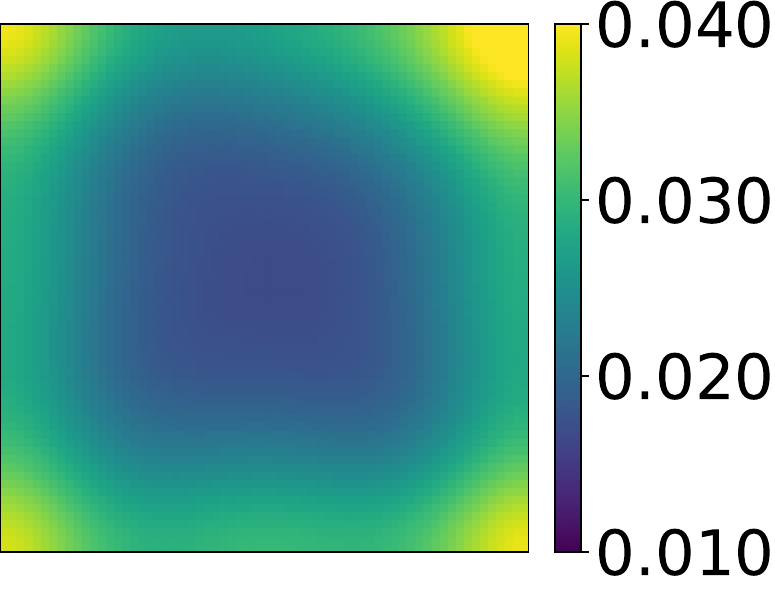}
        \label{fig:identity65_ref_std}
    }%
    \subfigure[$\sigma_\mathrm{ours}$]{
        \includegraphics[width=0.18\linewidth]{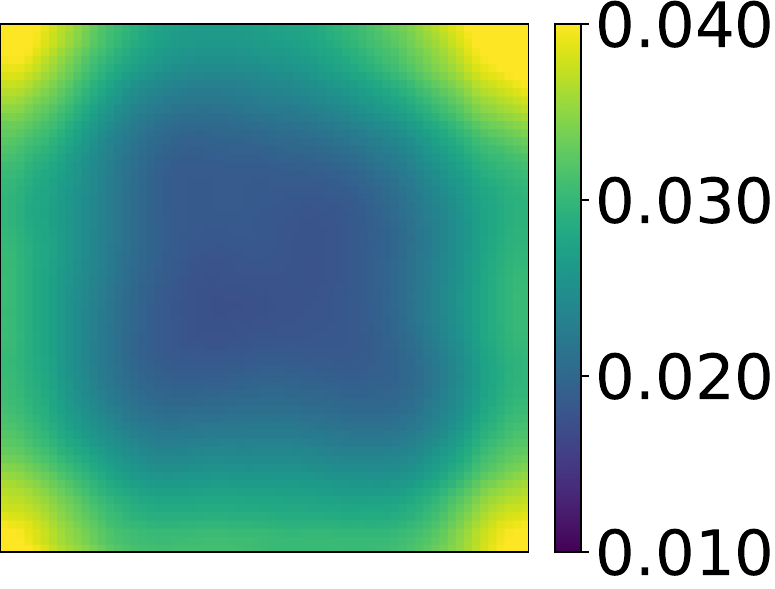}
        \label{fig:identity65_sol_std}
    }%
    \subfigure[$\sigma_{\mathrm{DPS}}$]{
        \includegraphics[width=0.18\linewidth]{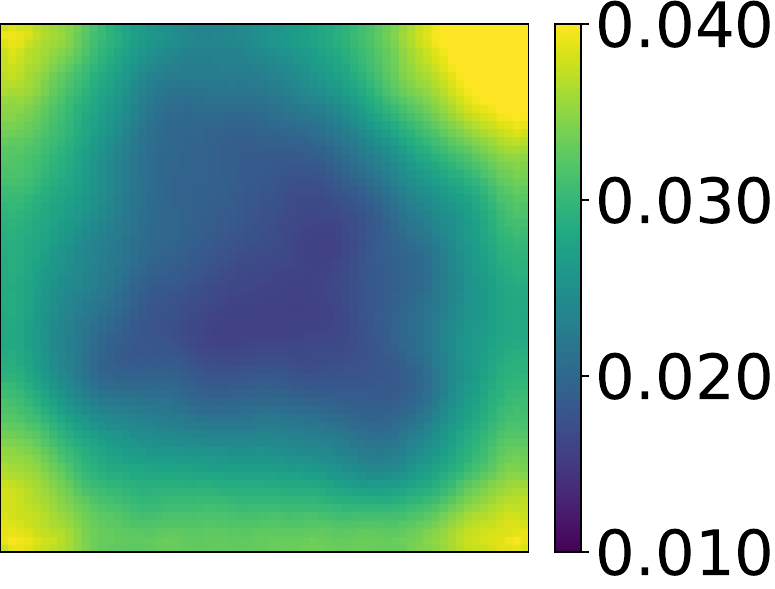}
        \label{fig:identity65_dps_std_zeta1e-1}
    }%
    \subfigure[$\sigma_{\mathrm{ours}}-\sigma_{\mathrm{ref}}$]{
        \includegraphics[width=0.18\linewidth]{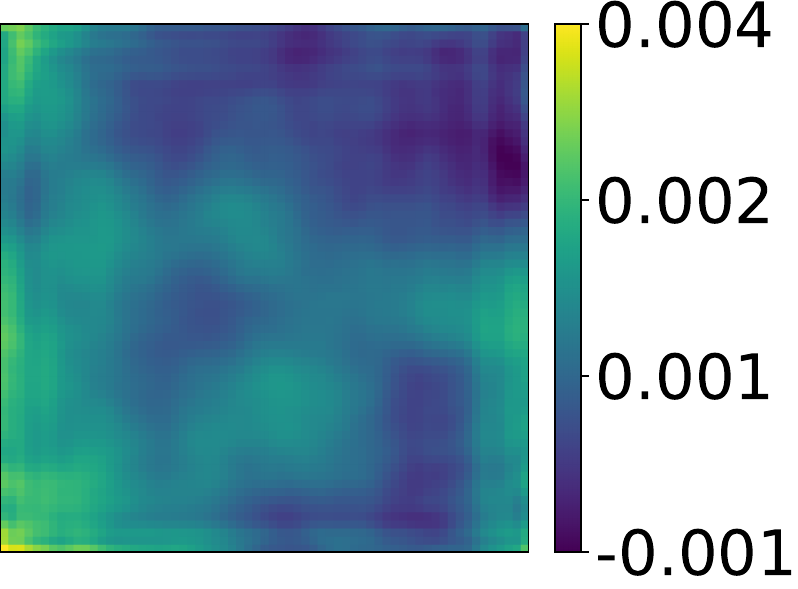}
        \label{fig:identity65_diff_std}
    }%
    \subfigure[$\sigma_{\mathrm{DPS}}-\sigma_{\mathrm{ref}}$]{
        \includegraphics[width=0.18\linewidth]{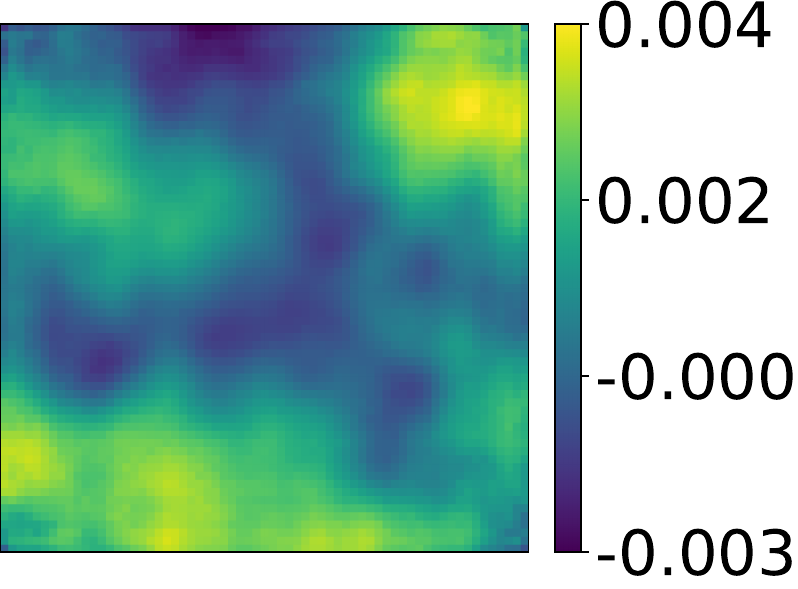}
        \label{fig:identity65_diff_std_dps_zeta1e-1}
    }

\caption{Comparison of posterior statistics from our model and DPS inference on the same $65\times65$ linear test instance as in Figure~\ref{fig:identity_results}. 
DPS uses an unconditionally trained diffusion model with a guidance weight, and we report the best-performing result obtained by tuning this weight over $[0,1]$, attained at $0.1$.}
    \label{fig:DPS_identity_results}
\end{figure}

\begin{figure}[!htb]
    \centering

    \subfigure[$\mu_\mathrm{ref}$]{
        \includegraphics[width=0.18\linewidth]{darcy65_ref_mean.pdf}
        \label{fig:darcy65_ref_mean}
    }%
    \subfigure[$\mu_\mathrm{ours}$]{
        \includegraphics[width=0.18\linewidth]{darcy65_sol_mean.pdf}
        \label{fig:darcy65_sol_mean}
    }%
    \subfigure[$\mu_\mathrm{DPS}$]{
        \includegraphics[width=0.18\linewidth]{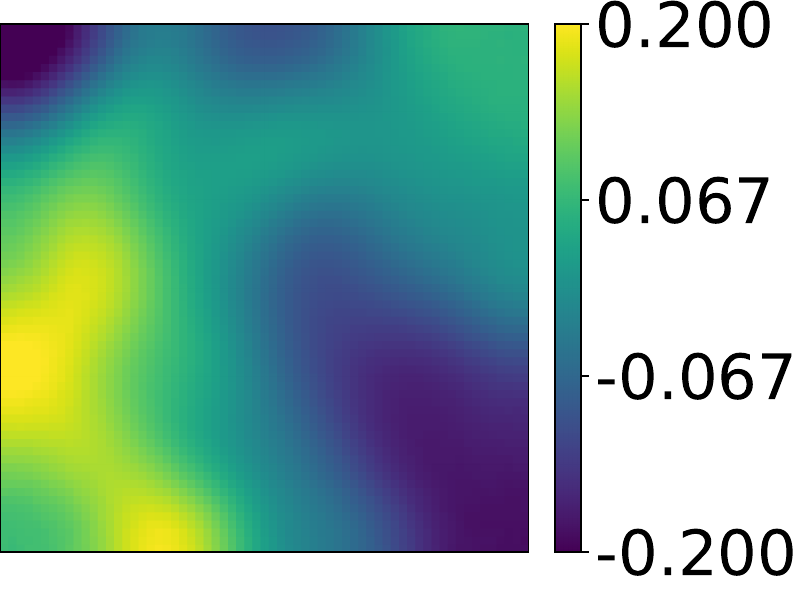}
        \label{fig:darcy65_dps_mean_zeta1e-1}
    }%
    \subfigure[$\mu_{\mathrm{ours}}-\mu_{\mathrm{ref}}$]{
        \includegraphics[width=0.18\linewidth]{darcy65_diff_mean.pdf}
        \label{fig:darcy65_diff_mean}
    }%
    \subfigure[$\mu_{\mathrm{DPS}}-\mu_{\mathrm{ref}}$]{
        \includegraphics[width=0.18\linewidth]{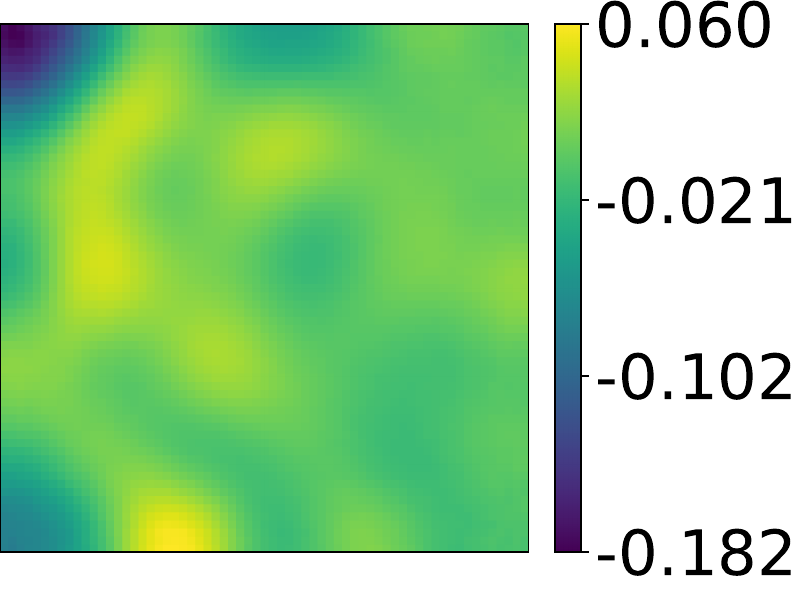}
        \label{fig:darcy65_diff_mean_dps_zeta1e-1}
    }
    \subfigure[$\sigma_\mathrm{ref}$]{
        \includegraphics[width=0.18\linewidth]{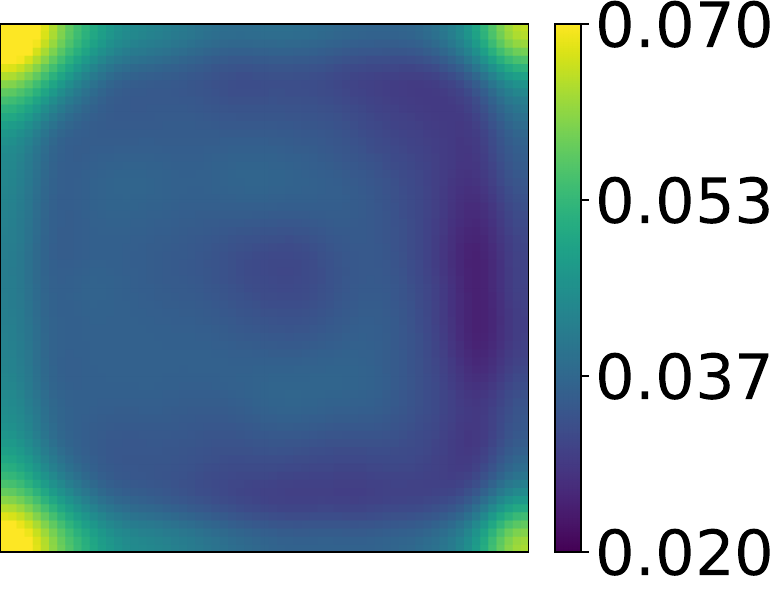}
        \label{fig:darcy65_ref_std}
    }%
    \subfigure[$\sigma_\mathrm{ours}$]{
        \includegraphics[width=0.18\linewidth]{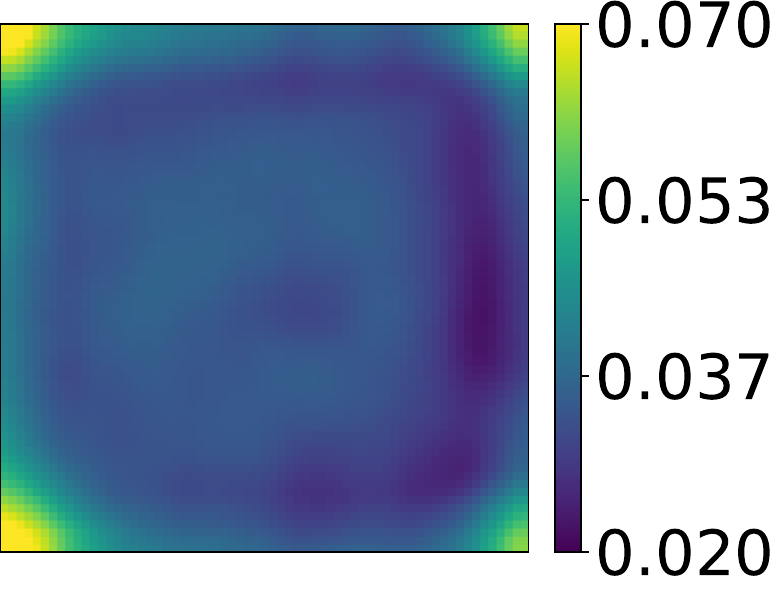}
        \label{fig:darcy65_sol_std}
    }%
    \subfigure[$\sigma_{\mathrm{DPS}}$]{
        \includegraphics[width=0.18\linewidth]{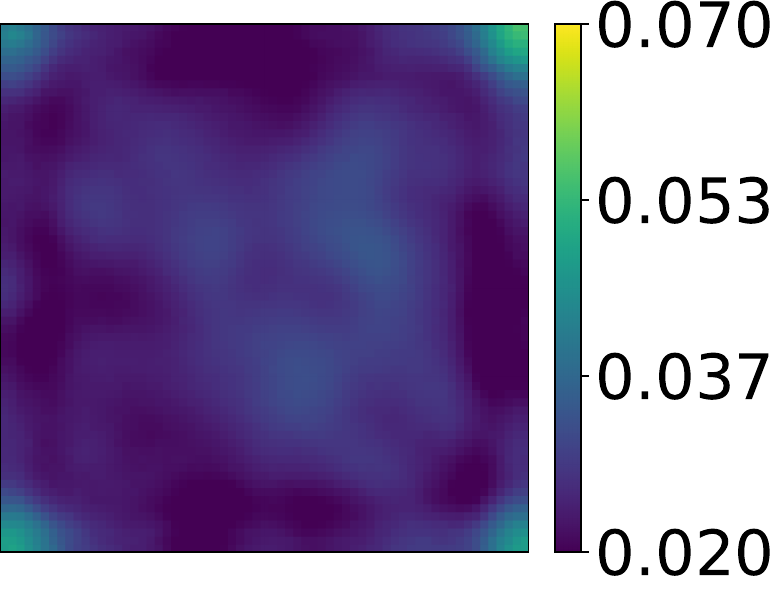}
        \label{fig:darcy65_dps_std_zeta1e-1}
    }%
    \subfigure[$\sigma_{\mathrm{ours}}-\sigma_{\mathrm{ref}}$]{
        \includegraphics[width=0.18\linewidth]{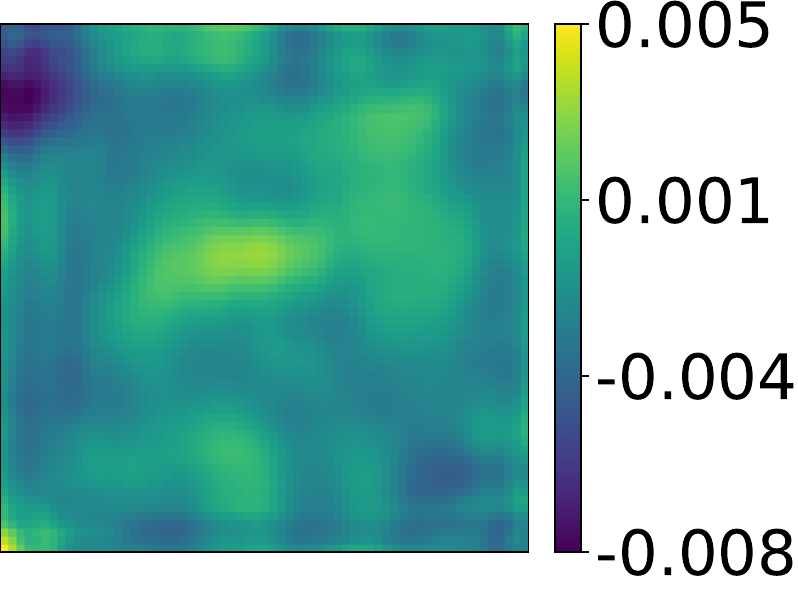}
        \label{fig:darcy65_diff_std}
    }%
    \subfigure[$\sigma_{\mathrm{DPS}}-\sigma_{\mathrm{ref}}$]{
        \includegraphics[width=0.18\linewidth]{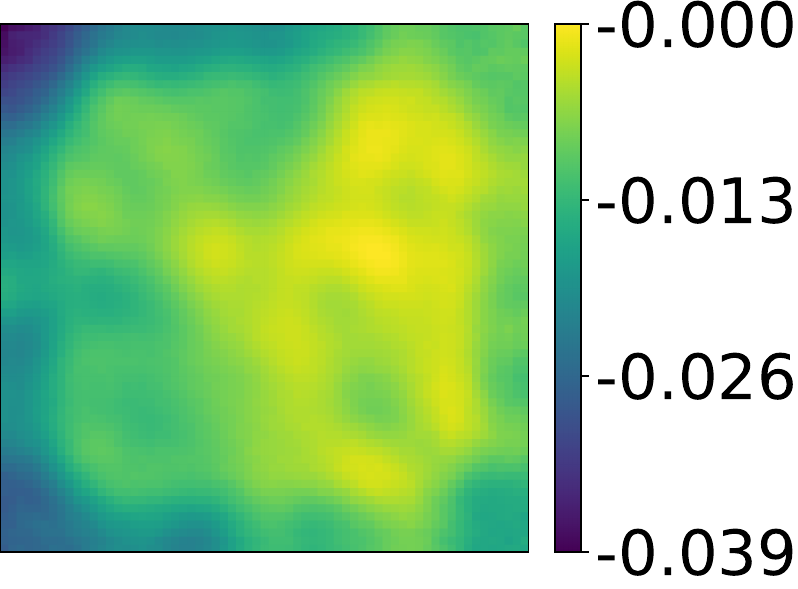}
        \label{fig:darcy65_diff_std_dps_zeta1e-1}
    }

    \caption{Comparison of posterior statistics from our model and DPS inference on the same $65\times65$ darcy test instance as in Figure~\ref{fig:darcy_results}. DPS uses an unconditionally trained diffusion model with a guidance weight, and we report the best-performing result obtained by tuning this weight over $[0,1]$, attained at $0.1$.}
    \label{fig:DPS_darcy_results}
\end{figure}
We also compare with DPS~\cite{chung2023diffusion}, which combines a prior-only diffusion model with guidance-based sampling. This type of method is modular and reusable, since the unconditional prior model can be trained once and then used
for different problems (e.g., different levels of observation noise $\eta$) by changing the guidance term. However, the posterior is not learned directly: the observation is incorporated only during sampling through an approximate likelihood guidance term. Thus, DPS is an approximate posterior sampler rather than an exact amortized conditional sampler for
\(\pi(\cdot\mid y_{\mathrm{obs}})\) as our model which has been proved in Theorem~\ref{thm:posterior_conditioned_population_minimizer}. In addition, for PDE inverse problems, this guidance step requires repeated evaluations of the forward solver and its gradient, making sampling problem-specifically expensive. 

These differences have been shown in numerical experiments: DPS takes \(0.290\)s/sample for the linear inverse problem and \(43.048\)s/sample for Darcy, while our one-step sampler takes only \(3.50\times10^{-4}\)s/sample.
Moreover, even after tuning the guidance weight, DPS yields inaccurate posterior errors in the function-space inverse-problem setting. For the linear inverse problem, DPS gives
\(\epsilon_\mu^{\mathrm{eval}}=0.2809\) and
\(\epsilon_\sigma^{\mathrm{eval}}=0.0700\), while our one-step sampler achieves
\(\epsilon_\mu^{\mathrm{eval}}=0.0494\) and
\(\epsilon_\sigma^{\mathrm{eval}}=0.0272\), as reported in
Table~\ref{tab:benchmark_summary_identity}. For the Darcy problem, DPS gives
\(\epsilon_\mu^{\mathrm{eval}}=0.3117\) and
\(\epsilon_\sigma^{\mathrm{eval}}=0.3394\), whereas our method achieves
\(\epsilon_\mu^{\mathrm{eval}}=0.0555\) and
\(\epsilon_\sigma^{\mathrm{eval}}=0.0427\) as reported in Table~\ref{tab:benchmark_summary}. Visual comparisons are shown in
Figures~\ref{fig:DPS_identity_results}--\ref{fig:DPS_darcy_results}.

\section{Notation and Network/Training Settings}
\label{app:notation_settings}
We summarize the notation in Table~\ref{tab:notation_summary} and the network architecture and training hyperparameters in Table~\ref{tab:network_setting}. 

\begin{table}[htbp!]
\centering
\caption{Notation.}
\label{tab:notation_summary}
\begin{small}
\begin{tabular}{l l}
\toprule
Symbol & Description \\
\midrule
$\Omega$ & Spatial domain \\
$d$ & Spatial dimension \\
$s$ & Sobolev regularity exponent defining the distribution space $\mathcal{H}^{-s}$ \\

\midrule
$X,x,\widetilde X$ & Unknown random function, its realization, and encoded variable \\
$Y_{\mathrm{obs}},y_{\mathrm{obs}},\widetilde Y_{\mathrm{obs}}$
& Observation random variable, its realization, and encoded observation \\
$\Xi,\xi$ & Source random variable and its realization, sampled from $\rho$ \\
$Z_t,z_t$ & Interpolated state and its realization at time $t$ \\
$V^{\mathrm{path}},v^{\mathrm{path}}$ & Path velocities, $V^{\mathrm{path}}:=\Xi-X$ and $v^{\mathrm{path}}:=\xi-x$\\
$v,v_\theta$ & Exact and learned instantaneous marginal velocity fields \\
$W_{\mathrm{tgt}}$ & Stop-gradient Mean-Flow target \\
$w,w_\theta$ & Exact and learned Mean-Flow averaged velocity fields \\
$r$, $t$ & Start and end time of the Mean-Flow interpolation \\

\midrule
$\mathcal{T}$ & Transport map from source to posterior \\
$\mathcal E_x,\mathcal E_y$ & State and observation encoders \\
$\Lambda$ & Prior covariance operator \\
$C$ & Anisotropic source covariance operator \\
$\gamma$ & Prior distribution, $\gamma\sim\mathcal N(0,\Lambda)$\\
$\rho$ & Source distribution; $\rho_C\sim\mathcal N(0,C)$ and $\rho_W\sim\mathcal N(0,I)$\\
$\pi$ & Posterior distribution \\
$\Gamma$ & Observation-noise covariance matrix \\
$\eta$ & Measurement noise \\
$\mathcal{G}$ & Forward operator \\
$\mathcal O$ & Observation operator \\
$\mathcal G_{\mathrm{obs}}$ & Observed forward map, $\mathcal G_{\mathrm{obs}}=\mathcal O\circ\mathcal G$\\

\midrule
$\theta$ & Network weights (trainable parameters)\\
$\hat\theta$ & Trained parameters used at inference \\

\midrule
$\mu_{\mathrm{ref}}$ & Reference posterior mean \\
$\sigma_{\mathrm{ref}}$ & Reference posterior pointwise standard deviation \\
$\mu_{\mathrm{pred}}$ & Predicted posterior mean from the learned model \\
$\sigma_{\mathrm{pred}}$ & Predicted posterior pointwise standard deviation from the learned model \\
$\epsilon_{\mu}^{\mathrm{eval}}$ & Relative $L^2(\Omega)$ error of posterior mean \\
$\epsilon_{\sigma}^{\mathrm{eval}}$ & Relative $L^2(\Omega)$ error of posterior standard deviation \\
\bottomrule
\end{tabular}
\end{small}
\end{table}

\begin{table}[htbp!]
\centering
\caption{Network architecture and training hyperparameters.}
\label{tab:network_setting}
\begin{small}
\begin{tabular}{l c c}
\toprule
Setting & Symbol & Value \\
\midrule
FNO Fourier modes & $K$ & 12 \\
FNO layers & $L$ & 4 \\
Network width & $N_{\mathrm{width}}$ & 96 \\
Learning rate & $\mathrm{lr}$ & $10^{-4}$ \\
Flow ratio & $\nu_{\mathrm{flow}}$ & 0.25 \\
Batch size & $N_{\mathrm{batch}}$ & 20 \\
Epochs & $N_{\mathrm{epoch}}$ & 1000 \\
Training size & $N_{\mathrm{train}}$ & 10000 \\
Posterior samples per test observation & $N_{\mathrm{eval}}$ & 1000 \\
Time-pair sampling & $\lambda$ & Uniform on \(0\le r\le t\le1\) \\
\midrule
Grid resolution (Linear, Darcy) & $N\times N$ & $65\times 65$ \\
Grid resolution (Advection, Reaction-Diffusion, Navier--Stokes)
& $N\times N$ & $64\times 64$ \\
Observation count (Linear, Darcy) & $m$ & 49 \\
Observation count (Advection, Reaction-Diffusion, Navier--Stokes) & $m$ & 64 \\
KL modes (anisotropic prior/source, Linear) & $N_{\mathrm{KL}}$ & $65^2-1$ \\
KL modes (anisotropic prior/source, PDE datasets) & $N_{\mathrm{KL}}$ & 64 \\
\bottomrule
\end{tabular}
\end{small}
\end{table}

\bibliography{cas-refs}
\bibliographystyle{plain}
\end{document}